\documentclass[smallextended]{svjour3}       
\smartqed  
\usepackage{booktabs}    
\usepackage{multirow}    
\usepackage{fullpage}
\usepackage{graphicx} 
\usepackage{lmodern}
\usepackage{amsmath}
\usepackage{amssymb}
\usepackage{wrapfig}
\usepackage{enumerate}
\usepackage{graphicx}
\usepackage{bbm}
\usepackage{comment}
\usepackage{bm}
\usepackage{cite}
\usepackage{algorithmic}
\usepackage{algorithm}
\usepackage{booktabs}
\usepackage{color}
\usepackage{url}

\newcommand{\vertiii}[1]{{\left\vert\kern-0.25ex\left\vert\kern-0.25ex #1 
    \right\vert\kern-0.25ex\right\vert\kern-0.25ex}}

\newcommand{\mbf}[1]{\boldsymbol{#1}}

\newcommand{\br}{\mbf{r}}

\newcommand{\bx}{\mbf{x}}
\newcommand{\bX}{\mbf{X}}

\newcommand{\bgamma}{\mbf{\gamma}}

\renewcommand{\top}{T}

\newcommand{\balpha}{\mbf{\alpha}}
\renewcommand{\algorithmicrequire}{\textbf{Input:}}    
\renewcommand{\algorithmicensure}{\textbf{Output:}}

\newcommand{\intkernel}{\phi}

\newcommand{\argmin}[1]{\underset{#1}{\operatorname{arg}\operatorname{min}}\;}

\usepackage{amsopn}

\begin{document}

\title{A Sparse Bayesian Learning Algorithm for Estimation of Interaction Kernels in  Motsch-Tadmor Model 
}


\author{Jinchao Feng        \and
        Sui Tang 
}


\institute{Jinchao Feng \at
              Department of Mathematics, Great Bay University, Dongguan, Guangdong, China \\
              \email{jcfeng@gbu.edu.cn}           
           \and
           Sui Tang \at
              Department of Mathematics, University of California, Santa Barbara, Isla Vista, CA, USA\\
              \email{suitang@math.ucsb.edu}
}

\date{Received: date / Accepted: date}

\maketitle

\begin{abstract}
In this paper, we investigate the data-driven identification of asymmetric interaction kernels in the Motsch–Tadmor model based on observed trajectory data. The model under consideration is governed by a class of semilinear evolution equations, where the interaction kernel defines a normalized, state-dependent Laplacian operator that governs collective dynamics. To address the resulting nonlinear inverse problem, we propose a variational framework that reformulates kernel identification using the implicit form of the governing equations, reducing it to a subspace identification problem. We establish an identifiability result that characterizes conditions under which the interaction kernel can be uniquely recovered up to scale. To solve the inverse problem robustly, we develop a sparse Bayesian learning algorithm that incorporates informative priors for regularization, quantifies uncertainty, and enables principled model selection. Extensive numerical experiments on representative interacting particle systems demonstrate the accuracy, robustness, and interpretability of the proposed framework across a range of noise levels and data regimes.
\end{abstract}

\section{Introduction}

Interacting dynamical systems are ubiquitous in natural and engineered systems, characterized by complex behaviors that emerge from local interactions and the temporal evolution of multiple agents or particles. Examples include planetary orbits, self-propelled motion, bird flocking, fish schooling, cell aggregation, opinion formation, and oscillator synchronization \cite{laskar1990chaotic,vicsek1995novel,cucker2007emergent,couzin2002collective,keller1971model,hegselmann2002opinion,acebron2005kuramoto}. Differential equation models have been instrumental in understanding how collective behaviors arise from individual interactions. However, the high dimensionality and heterogeneous interaction mechanisms inherent to these systems pose major challenges for mathematical modeling and data-driven learning \cite{lu2019nonparametric,lu2020learning,lu2021learning,miller2023learning}.

Recent advances in applied mathematics demonstrate that even simple interaction laws can generate complex emergent dynamics when embedded in mechanistic models \cite{schmidt2009distilling,BKP2017,zhang2018robust,BCCCCGLOPPVZ2008,BCGMSVW2012}. Consider a system of \(N\) agents, each defined by a position \(\mathbf{x}_i(t) \in \mathbb{R}^d\). For first-order dynamics of homogeneous agents, the governing equations are:

\begin{equation}\label{ode1}
\frac{d \mathbf{x}_i(t)}{dt} =  \sum_{j=1}^N a_t(i,j)(\mathbf{x}_j(t) - \mathbf{x}_i(t)),
\end{equation}
\begin{equation}
a_t(i,j) = \frac{\phi(\|\mathbf{x}_j(t) - \mathbf{x}_i(t)\|)}{c_i},
\end{equation}
where $a_t(i,j)$ represents the normalized influence of agent $j$ on agent $i$; \(\phi: \mathbb{R}^+ \rightarrow \mathbb{R}\) is the interaction kernel based on relative distance; and \(c_i\) is a normalization factor. For symmetric models, \(c_i = N\), while for asymmetric models, \(c_i = \sum_{j=1}^N \phi(\|\mathbf{x}_j(t) - \mathbf{x}_i(t)\|)\). These models can also be extended to second-order systems.

While symmetric models have been extensively analyzed, they often fail to capture far-from-equilibrium dynamics where interactions are localized and not uniformly distributed. To address this, Motsch and Tadmor introduced an asymmetric formulation \cite{motsch2011new,motsch2014heterophilious}, generalizing the Krause model \cite{krause2000discrete,blondel2009krause} for continuous-time opinion dynamics. In this framework, the influence between agents is scaled by their relative distances, eliminating dependence on the number of agents and emphasizing geometric configuration in phase space. Analogous mechanisms appear in recent advances in machine learning, such as graph transformers with non-symmetric attention mechanisms \cite{geshkovski2024emergence}, which show improved empirical performance over their symmetric counterparts.
Numerical studies \cite{blondel2005convergence,blondel2009krause,motsch2014heterophilious} have shown that asymmetric models capture phenomena such as local consensus in opinion dynamics \cite{jabin2014clustering} and mono- or multi-cluster flocking \cite{ha2024mono}. Despite their expressiveness, asymmetric models introduce mathematical difficulties due to the absence of conserved quantities (e.g., total momentum), which are central to the analysis of symmetric systems. While recent works \cite{motsch2011new,motsch2014heterophilious,jabin2014clustering,jin2018flocking,ha2024mono,geshkovski2024emergence} have advanced the understanding of emergent dynamics in asymmetric systems, rigorous analysis remains an active research area.

A central challenge in qualitative analysis is identifying interaction functions that lead to spontaneous pattern formation or self-organization. Traditionally, interaction potentials were empirically chosen to replicate observed dynamics. However, advancements in data acquisition technologies, such as digital imaging and GPS tracking \cite{lukeman2010inferring,katz2011inferring}, now enable the collection of density evolution data for large particle ensembles. This raises a critical question: can the underlying interaction laws governing the dynamics be inferred directly from empirical observations? Addressing this requires algorithms that align theoretical models with empirical observations. This paper tackles this challenge by bridging the gap between models and data. While data-driven approaches have extensively studied symmetric models, asymmetric models remain underexplored. This work focuses on filling this gap.

\subsection{Related Works}

Data-driven discovery of dynamical systems using machine learning has emerged as a core theme in scientific computing \cite{bongard2007automated, schmidt2009distilling, BPK2016, RBPK2017, BKP2017, han2015robust, kang2019ident, zhang2018robust, BCCCCGLOPPVZ2008, BCGMSVW2012, raissi2018deep, raissi2018hidden, long2017pde}, enabling the representation and analysis of complex functional data. These approaches typically formulate a loss functional comprising a data-fidelity term (enforcing consistency with governing equations) and a regularization term (embedding prior knowledge such as sparsity), transforming the learning task into an optimization problem. This paradigm effectively leverages machine learning techniques to tackle high-dimensional challenges.

Despite extensive progress in symmetric models \cite{lu2019nonparametric, lu2020learning, lu2021learning, miller2023learning, liu2021random}, data-driven discovery of asymmetric interaction laws remains underexplored. The lack of symmetry leads to a nonlinear inverse problem, precluding the direct application of prior methods and necessitating new theoretical and algorithmic tools. To the best of our knowledge, this paper is the first to address the data-driven discovery of asymmetric models. The asymmetric property introduces a nonlinear inverse problem, rendering previous techniques for symmetric models inapplicable and necessitating new analytical and algorithmic approaches.

A close line of work is the sparse identification of dynamical systems (SINDy) framework \cite{BPK2016}, and in particular its implicit formulation \cite{mangan2016inferring}, which considers learning ODEs with rational right-hand side functions. In this approach, the governing equations are reformulated in implicit form, leading to a homogeneous system that can be interpreted as a linear subspace identification problem. While our formulation shares a similar spirit, the asymmetric normalization in the Motsch--Tadmor model introduces a nonlinear dependence on the kernel, which prevents the direct application of standard SINDy or implicit SINDy methods. In addition, subspace identification approaches are known to be sensitive to noise and model misspecification, particularly in stochastic or noisy settings, which further motivates the Bayesian treatment adopted here.

To robustly address the ill-posed nature of the inverse problem, we adopt a sparse Bayesian learning (SBL) framework. This approach not only enables flexible incorporation of prior information (e.g., sparsity or smoothness) but also quantifies uncertainty in model estimates, leading to more interpretable and stable results. While deterministic sparsity-promoting techniques, such as SINDy, iterative greedy methods, and gradient-based methods \cite{BPK2016, TranWardExactRecovery, schaeffer2018extracting, boninsegna2018sparse}, dominate the literature, Bayesian methods remain underexplored \cite{zhang2018robust}. 

To clarify the relationship between the proposed method and existing sparse identification approaches, in Section~\ref{comparision:sindy} we adapt an implicit SINDy-type method to our setting and compare it with our sparse Bayesian approach; see Figure~\ref{fig:sindy_comparison}. The results show that when the true kernel is exactly representable in the chosen function space, the SINDy-based method can recover the correct structure in the noise-free setting. However, it does not provide uncertainty quantification and is more sensitive to noise. More importantly, when the true kernel is not contained in the chosen basis space, the SINDy-based approach fails to produce meaningful approximations, whereas the proposed Bayesian framework remains robust and continues to provide reliable estimates together with uncertainty quantification.

\paragraph{Contributions} Our contributions can be summarized as follows:
\begin{itemize}
    \item We propose a variational learning framework based on the implicit form of the governing ODE, enabling the identification of interaction kernels in asymmetric models via sparse Bayesian inference.
    \item We introduce a new model selection criterion for identifying relevant basis functions, which we show outperforms standard approaches under various noise conditions.
    \item We validate the proposed method on multiple synthetic datasets, demonstrating robustness to noise and efficiency in computation.
\end{itemize}

\section{Problem Formulation}
We begin by formulating the problem of estimating the interaction kernel 
$\phi$ from noise-free observational trajectory data. Let the dataset be denoted by \(\mathcal{D}_{M,L} = \{ \mathcal{D}^{(m,l)} \}_{m,l=1}^{M,L} := \{ \bx^{(m)}_i(t_l), \dot{\bx}_i^{(m)}(t_l), i=1,\cdots,N\}_{m,l=1}^{M,L}\), where \(0 = t_1 < \dots < t_L = T\) are observation time instances. Here, \(m\) denotes the trial number, with initial conditions sampled i.i.d. from a probability measure on \(\mathbb{R}^d\).


We assume that the interaction kernel \(\phi\) lies in the span of a finite set of basis functions \(\mathcal{S} = \{\xi_k\}_{k=1}^K\), i.e.,
\[
\phi = \sum_{k=1}^K c_k^{\mathrm{true}} \,\xi_k,
\]
for coefficients \(\mbf{c}^{\mathrm{true}} \in \mathbb{R}^K\). The choice of basis functions is not restricted in the proposed framework and can be adapted to the expected regularity of the kernel. In the numerical experiments presented in this work, we employ piecewise constant basis functions due to their flexibility and their ability to approximate kernels without imposing smoothness assumptions. However, smooth basis functions (e.g., radial basis functions or polynomial bases) can be incorporated directly without any modification to the formulation or the inference procedure.

Multiplying both sides of the normalized ODE \eqref{ode1} by the normalization constant \(\sum_{j \in \mathcal{N}_i} \phi(|\bx_j - \bx_i|)\), we eliminate the denominator and obtain the implicit constraint:
\begin{equation}\label{newode}
\sum_{j=1}^{N} \phi(|\bx_j - \bx_i|)(\dot\bx_i - (\bx_{i} - \bx_{j})) = 0, \quad i = 1, \ldots, N.
\end{equation}

To estimate $\phi$, we propose a variational approach that minimizes the empirical loss functional over the unit sphere in $\mathbb{R}^K$, enforcing non-trivial solutions:
\begin{align}
\min_{\|\mbf{c}\|_2=1} \mathcal{E}_{M,L}(\mbf{c}),
\end{align}
where
\begin{align}
\mathcal{E}_{M,L}(\mbf{c}) := \frac{1}{NML} \sum_{m,l,i=1,1,1}^{M,L,N} \Big\|\sum_{j=1}^{N} \sum_{k=1}^{K} c_{k} \xi_k(|\bx_j^{(m)}(t_l) - \bx_i^{(m)}(t_l)|) (\dot\bx_i^{(m)}(t_l) - (\bx_{i}^{(m)}(t_l) - \bx_{j}^{(m)}(t_l)))\Big\|^2.
\end{align}

The goal is to identify a function from the candidate library that satisfies \eqref{newode} as closely as possible on the observed data, ensuring \(\mathcal{E}_{M,L}(\mbf{c}^{\mathrm{true}}) = 0\). Since \(\mathcal{E}_{M,L}(\cdot)\) is quadratic in \(\mbf{c}\), it admits a closed-form solution described in Proposition \ref{prop1}. We note that directly using the original ODE form \eqref{ode1} as the loss function leads to a nonlinear least squares problem, which is more challenging to solve.

\begin{proposition}\label{prop1}
Let \(\hat{\mbf{c}}_{M,L} = \argmin{\|\mbf{c}\|_2=1} \mathcal{E}_{M,L}(\mbf{c})\). Then:
\begin{equation}\label{eq:prob1}
\frac{1}{NML} \mathbb{A}_{M,L}^{\top} \mathbb{A}_{M,L} \hat{\mbf{c}}_{M,L} = \mbf{0},
\end{equation}
where \(\mathbb{A}_{M,L} = \begin{pmatrix}\mbf{A}^{(1,1)} \\ \vdots \\ \mbf{A}^{(M,L)} \end{pmatrix} \in \mathbb{R}^{dNML \times K}\) with \(\mbf{A}^{(m,l)} = [\mbf{A}_{ik}^{(m,l)}] \in \mathbb{R}^{dN \times K}\), and
\begin{equation}\label{eq:assemble_A}
\mbf{A}_{ik}^{(m,l)} = \sum_{j=1}^{N} \xi_k(|\bx_j^{(m)}(t_l) - \bx_i^{(m)}(t_l)|) (\dot\bx_i^{(m)}(t_l) - (\bx_{i}^{(m)}(t_l) - \bx_{j}^{(m)}(t_l))) \in \mathbb{R}^d,
\end{equation}
for \(i=1,\cdots,N\) and \(k=1,\cdots,K\).
\end{proposition}

\begin{proof}
The loss function \(\mathcal{E}_{M,L}(\mbf{c})\) can be written as a quadratic form: 
\[
\mathcal{E}_{M,L}(\mbf{c}) = \frac{1}{NML} \|\mathbb{A}_{M,L} \mbf{c}\|^2 = \frac{1}{NML} \mbf{c}^\top \mathbb{A}_{M,L}^\top \mathbb{A}_{M,L} \mbf{c}.
\]
Differentiating with respect to \(\mbf{c}\) and setting it to zero gives:
\[
\mathbb{A}_{M,L}^\top \mathbb{A}_{M,L} \mbf{c} = \mbf{0},
\]
which proves the result.
\end{proof}

The null spaces of \(\mathbb{A}_{M,L}^\top \mathbb{A}_{M,L}\) and \(\mathbb{A}_{M,L}\) are identical, so it suffices to solve \(\mathbb{A}_{M,L} \mbf{c} = 0\), excluding the trivial solution \(\mbf{c} = 0\) due to normalization.

\begin{remark}
Estimating the right-hand side of a \(dN\)-dimensional ODE system is challenging. For instance, estimating a generic cubic polynomial \(\phi\) defined on \(\mathbb{R}^{Nd}\) requires a polynomial basis of dimension \(\mathcal{O}((Nd)^3)\), which grows exponentially with \(N\) and \(d\), necessitating large datasets. In contrast, our problem is feasible with limited data due to the 1D nature of \(\phi\), which reduces complexity. Furthermore, \(\phi\) couples the entire equation, leveraging the system's structure to constrain and regularize the estimation, enabling possible accurate recovery.
\end{remark}

\paragraph{Identifiability.} Next, we examine conditions under which the true coefficients can be recovered. The matrix \(\mathbb{A}_{M,L}\) is inherently random due to stochastic initial conditions. Using a statistical inverse problem framework, we analyze identifiability as the number of trials \(M\) tends to infinity while \(L\) remains fixed.

\begin{proposition}
Assuming the initial conditions \(\bx_1(0), \ldots, \bx_N(0)\) are i.i.d. samples from a probability distribution \(\mu_0\), we have:
\begin{align}
\lim_{M \to \infty} \frac{1}{NML} \mathbb{A}_{M,L}^\top \mathbb{A}_{M,L} = \mathbb{B} \in \mathbb{R}^{K \times K}\quad a.s.,
\end{align}
where 
\[
\mathbb{B}_{kk'} = 
\frac{1}{LN} \sum_{\ell=1}^{L} \mathbb{E}_{\mu_0} \left[\sum_{i=1}^{N} \sum_{j=1}^{N} \sum_{j'=1}^{N} \xi_k(\|\br_{ij}(t_\ell)\|) \xi_{k'}(\|\br_{ij'}(t_\ell)\|) \langle \dot{\bx}_i(t_\ell) - \br_{ij}(t_\ell), \dot{\bx}_i(t_\ell) - \br_{ij'}(t_\ell) \rangle \right].
\]
If \(\mathrm{rank}(\mathbb{B}) = K-1\), the true coefficients can be identified up to a scaling factor.
\end{proposition}

\begin{proof}[Proof sketch]
We decompose the empirical Gram matrix as
\[
\frac{1}{NML}\mathbb{A}_{M,L}^\top \mathbb{A}_{M,L}
=
\frac{1}{M}\sum_{m=1}^M \frac{1}{NL} \bigl(\mathbb{A}^{(m)}\bigr)^\top \mathbb{A}^{(m)},
\]
where \(\mathbb{A}^{(m)}\) denotes the contribution from the \(m\)-th independent trajectory. Since the initial conditions are i.i.d.\ samples from \(\mu_0\), the matrices \(\frac{1}{NL}(\mathbb{A}^{(m)})^\top \mathbb{A}^{(m)}\) are i.i.d.\ random matrices with finite second moments. Therefore, by the (entrywise) law of large numbers,
\[
\frac{1}{NML}\mathbb{A}_{M,L}^\top \mathbb{A}_{M,L}
\longrightarrow
\mathbb{E}_{\mu_0}\!\left[\frac{1}{NL}(\mathbb{A}^{(1)})^\top \mathbb{A}^{(1)}\right]
= \mathbb{B},
\quad \text{a.s. as } M \to \infty,
\]
which yields the stated expression for \(\mathbb{B}\) by direct expansion.
If \(\mathrm{rank}(\mathbb{B}) = K-1\), then its null space is one-dimensional. Since the true coefficient vector \(\mbf{c}^{\mathrm{true}}\) satisfies the underlying implicit relation, it lies in \(\ker(\mathbb{B})\), implying that \(\ker(\mathbb{B}) = \mathrm{span}\{\mbf{c}^{\mathrm{true}}\}\). Hence, the coefficients are identifiable up to a multiplicative constant.
This argument follows standard results on convergence of sample covariance (Gram) matrices in statistical inverse problems and random matrix theory; see, e.g., \cite{vershynin2018high, van2000asymptotic}.
\end{proof}

There are several implications from this proposition. First, one needs to construct the basis on the support of the measure
\begin{equation*}
    \rho(dr) = \frac{1}{L}\sum_{l=1}^{L} \mathbb{E}_{\mu_0}\left[\frac{1}{N^2} \sum_{i,i'=1}^{N,N} \delta_{||\mathbf{x}_{i'}(t_l) - \mathbf{x}_{i}(t_l)||}(dr)\right],
\end{equation*}
which describes the pairwise distances explored by the dynamical systems during the observation period. If a basis function $\xi_k$ is supported outside $\text{supp}(\rho)$, it will yield a zero column in $\mathbb{B}$, leading to rank deficiency.

Second, the rank constraint on $\mathbb{B}$ implies that it has a one-dimensional null space, and its second smallest singular value is positive. By applying classical random matrix theory, the second smallest singular value of \(\frac{1}{NML}\mathbb{A}_{M,L}^{\top}\mathbb{A}_{M,L}\) converges to that of $\mathbb{B}$ as $M \to \infty$ \cite{vershynin2018high}. This ensures the identifiability of \(\mathbf{c}^{\mathrm{true}}\) when \(M\) is sufficiently large. The coefficients can then be estimated using the SVD algorithm by identifying the singular vector corresponding to the smallest singular value, a standard condition appearing in sparse recovery and compressive sensing theory \cite{donoho2006compressed, candes2008introduction}.

A special case arises when $\phi$ is \textit{1-sparse} in the basis representation, meaning only one of the coefficient entries is nonzero. In this scenario, the identifiability condition requires that exactly one column of $\mathbb{B}$ is zero, while the remaining \(K-1\) columns are linearly independent.

In our numerical experiments, we also explore the use of piecewise constant basis functions to approximate a continuous kernel function $\phi$. In this case, $\phi$ does not lie in the span of the basis, and we instead find the eigenvector of $\mathbb{B}$ corresponding to the minimal eigenvalue. We remark that such a matrix $\mathbb{B}$ is a perturbation of the matrix constructed using a piecewise basis of the form $\phi(\cdot)\chi_{I_i}(\cdot)$ on each subinterval $I_i$, up to a similarity transformation. The spectral gap of the gram matrix determines the robustness of eigenvector recovery, as governed by the Davis-Kahan theorem \cite{davis1970rotation, stewart1990matrix}. Later, we demonstrate through numerical examples that in the case of small noise, one can achieve good recovery accuracy.


\section{Algorithm}

Building on the previous section, the estimation problem reduces to solving the homogeneous system:
\begin{equation}
    \mathbb{A} \mathbf{c} = 0,
\end{equation}
where, for brevity, we omit the indices $M, L$. To avoid the trivial solution $\mathbf{c} = 0$, a standard approach from prior works \cite{tipping2001sparse} is to fix one coefficient of $\mathbf{c}$ to a nonzero value. Specifically, we set $\mathbf{c}_{k^\ast} = 1$ for some $k^\ast \in \{1, \dots, K\}$, leading to the reformulated system:
\begin{equation}\label{eq:nonimplicit1}
    \mathbb{A}_{-k^\ast} \mathbf{c}_{-k^\ast} = -\mathbb{A}_{k^\ast},
\end{equation}
where $\mathbb{A}_{k^\ast}$ denotes the $k^\ast$-th column of $\mathbb{A}$, and $\mathbb{A}_{-k^\ast}$ is the matrix obtained by removing this column. This choice guarantees a nontrivial solution by ensuring at least one nonzero coefficient.

This approach is well justified because the estimation problem depends only on the relative magnitudes of the coefficients. Specifically, the dynamics of each $\mathbf{x}_i$ are governed by the ratio:
\begin{equation}
    \frac{\phi(|\mathbf{x}_{i'} - \mathbf{x}_i|)}{\sum_{j \in \mathcal{N}_i} \phi(|\mathbf{x}_j - \mathbf{x}_i|)},
\end{equation}
which remains unchanged under normalization. Consequently, setting
\begin{equation}
    \bar{\mathbf{c}} = \left( \frac{c_1}{c_{k^\ast}}, \dots, 1, \dots, \frac{c_K}{c_{k^\ast}} \right)
\end{equation}
yields an equivalent solution, provided that $\mathbf{c}_{k^\ast} \neq 0$. Thus, the coefficients can be rewritten as:
\begin{equation}
    \bar{\mathbf{c}} = \mathbf{e}_{k^\ast} + \hat{\mathbf{c}}_{-k^\ast},
\end{equation}
where $\mathbf{e}_{k^\ast}$ is the canonical basis vector in $\mathbb{R}^K$, and $\hat{\mathbf{c}}_{-k^\ast}$ is the coefficient estimate from \eqref{eq:nonimplicit1}, embedded in $\mathbb{R}^K$ with its $k^\ast$-th entry set to zero.

In summary, the estimation procedure consists of two key steps:
\begin{itemize}
    \item Step 1: {solving the reformulated system \eqref{eq:nonimplicit1} for each $k = 1, \dots, K$} using an appropriate optimization method, such as least squares.  
    \item Step 2: {model selection.} Identifying the optimal index $k^*$
  to balance numerical stability and robustness.
\end{itemize}

\subsection{Challenges in Estimation}\label{subsec:challenges}

Despite its conceptual simplicity, this approach faces several fundamental challenges:

\textbf{Essential ill-posedness.} The estimation problem could become inherently ill-posed when the dimension of the null space of $\mathbb{B}$—the expectation of the empirical matrix $\mathbb{A}^{T} \mathbb{A}$—exceeds one, i.e., when $\mathrm{rank}(\mathbb{B}) < K-1$. This could be caused by the nature of the dynamical systems, or by choosing basis functions outside of the support of $\rho$. In such cases, even with an appropriate choice of $k^\ast$, solving \eqref{eq:nonimplicit1} remains ill-posed, necessitating additional constraints or regularization to restore identifiability.

\textbf{Scarce and noisy data.} In practical applications, data is often limited ($M$ is small) and subject to measurement noise. Let $\tilde{\mathbf{x}}_i(t_l)$ and $\tilde{\dot{\mathbf{x}}}_i(t_l)$ denote the noisy observations of position and velocity, respectively. The $(i,k)$ entry of the empirical matrix $\mathbb{A}^{(m,l)}$ in \eqref{eq:assemble_A} is then perturbed as follows:
\begin{align*}
    \tilde{\mathbf{A}}_{ik}^{(m,l)} &= \sum_{j \in \mathcal{N}_i} \xi_k\left(|\tilde{\mathbf{x}}_j^{(m)}(t_l) - \tilde{\mathbf{x}}_i^{(m)}(t_l)|\right) 
    \left(\tilde{\dot{\mathbf{x}}}_i^{(m)}(t_l) - (\tilde{\mathbf{x}}_i^{(m)}(t_l) - \tilde{\mathbf{x}}_j^{(m)}(t_l))\right) \\
    &= \mathbf{A}_{ik}^{(m,l)} + \tilde{\mathbf{\epsilon}}_{ik}^{(m,l)},
\end{align*}
where $\tilde{\mathbf{\epsilon}}_{ik}^{(m,l)}$ represents the noise-induced perturbation. Consequently, the true system of equations is perturbed to:
\begin{equation}\label{newe}
    \tilde{\mathbb{A}}_{M,L} \mathbf{c} = 0,
\end{equation}
whereas the correct system should be:
\begin{equation}\label{eq:noiseterm}
    \tilde{\mathbb{A}}_{M,L} \mathbf{c} = \mathbf{\eta}_{M,L}, \quad \text{where} \quad \mathbf{\eta}_{M,L} := \tilde{\mathbf{\epsilon}}_{M,L} \mathbf{c}.
\end{equation}
Here, $\mathbf{\eta}_{M,L}$ represents an unknown noise vector, making direct calibration of \eqref{newe} infeasible. Furthermore, the right singular vectors of $\tilde{\mathbb{A}}_{M,L}$ may deviate significantly from those of the true matrix $\mathbb{A}_{M,L}$ (and thus $\mathbb{B}$), complicating the identification of the correct null space. This sensitivity to noise underscores the need for robust estimation techniques, such as regularization or noise-aware optimization strategies.

Finally, the estimation performance in step 1 would affect the model selection performance in step 2.

\subsection{A Sparse Bayesian Learning (SBL) Approach}
For clarity, we consider a universal form of \eqref{eq:noiseterm} in linear inverse problem as:
\begin{equation}\label{eq:newformulation}
    \mathbf{b} = \mathbf{\Phi} \mathbf{w} + \mathbf{\eta},
\end{equation}
where \(\mathbf{b} = -\tilde{\mathbb{A}}_{k^\ast} \in \mathbb{R}^{dNML}\) is a vector of indirect measurements/observations, \(\mathbf{w} = \mathbf{c}_{-k^\ast} \in \mathbb{R}^{K-1}\) is the vector of unknowns which we seek to recover, \(\mathbf{\Phi} = \tilde{\mathbb{A}}_{-k^\ast} \in \mathbb{R}^{dNML \times (K-1)}\) is a known linear forward operator, and \(\mbf{\eta}\) represents an unknown noise vector.  We employ the \textit{Sparse Bayesian Learning (SBL)} framework \cite{tipping2001sparse}, a principled and flexible framework that addresses the challenges in section \ref{subsec:challenges}.

SBL is particularly suited for high-dimensional, noisy data, leveraging probabilistic inference and adaptive regularization to mitigate ill-posedness. Unlike conventional methods requiring manual tuning, SBL infers regularization parameters directly from data, improving efficiency and stability. Beyond regularization, SBL naturally quantifies uncertainty, accounting for variability due to initial conditions and measurement noise. This facilitates principled model selection and provides confidence measures for inferred parameters.

\subsubsection{Data and Observation Models}

To incorporate measurement uncertainty in the observed trajectories, we introduce an additive observational noise model. In practice, the velocities $\dot{\bX}(t_l)$ are typically obtained either from numerical differentiation of position data or from noisy measurements, both of which introduce stochastic errors. Following a standard modeling assumption in inverse problems and statistical learning, we approximate these errors by independent Gaussian perturbations. The dataset is therefore modeled as:
\begin{align}\label{eq:noisemodel}
\tilde{\mathcal{D}}_{M,L} = \{\bX^{(m)}(t_l), \dot{\bX}^{(m)}(t_l) + \mbf{\zeta}^{(m,l)}\}_{m,l=1}^{M,L},
\end{align}
where \(\mbf{\zeta}^{(m,l)} = \{\mbf{\zeta}_i^{(m,l)}\}_{i=1}^N\) are i.i.d. random variables sampled from \(\mathcal{N}(\mbf{0}, \sigma_{\mathrm{velocity}}^2 I_{dN})\). This setup leads to a distribution for \(\tilde{\mbf{\epsilon}}_{ik}^{(m,l)}\):
\[
\tilde{\mbf{\epsilon}}_{ik}^{(m,l)} \sim \mathcal{N}\left(\mbf{0}, \sigma_{\mathrm{velocity}}^2 \sum_{j \in \mathcal{N}_i} \xi_k^2(|\bx_j^{(m)} - \bx_i^{(m)}|)\right).
\]

The observation noise \(\mbf{\eta}^{(m,l)} \in \mathbb{R}^{Nd}\) is then modeled as:
\begin{equation}\label{eq:noise}
\mbf{\eta}^{(m,l)} \sim \mathcal{N}\left(\mbf{0}, \mathrm{diag}\left(\sum_{j \in \mathcal{N}_i} \sigma_{\mathrm{velocity}}^2 c_k^2 \xi_k^2(|\bx_j^{(m)} - \bx_i^{(m)}|)\right)\right).
\end{equation}

Therefore, the realizations of each random variable in \eqref{eq:newformulation} respectively are the observable data $\mbf{b}$ and noise vector $\mbf{\eta}$, and a hierarchical Bayesian model based on Bayes’ theorem then provides an estimate of the full posterior distribution of $\mbf{w}$ given by
\begin{equation}
    p(\mbf{w},\mbf{\Theta}|\mbf{b}) \propto p(\mbf{b}|\mbf{w},\mbf{\Theta})p(\mbf{w}|\mbf{\Theta})p(\mbf{\Theta})
\end{equation}
where $\mbf{\Theta}$ are all the involved parameters in the random variables, where $p(\mbf{b}|\mbf{w},\mbf{\Theta})$ is the likelihood density function determined by the relationship between $\mbf{b}$ and $\mbf{w}$, and assumptions on $\mbf{\eta}$, $p(\mbf{w}|\mbf{\Theta})$ is the density of the prior distribution encoding a priori assumptions about the weights $\mbf{w}$, and $p(\mbf{\Theta})$ is the hyperprior density for all other involved parameters.

\subsubsection{The Likelihood}
The likelihood function $p(\mbf{b}|\mbf{w},\mbf{\Theta})$ models the connection between the weights $\mbf{w}$, the indirect measurements $\mbf{b}$, and the noise vector $\mbf{\eta}$. Based on the derivation in \eqref{eq:noise}, for simplicity, we can assume that \(\mbf{\eta}\) follows a zero-mean Gaussian distribution with diagonal identical covariance:
\[
\mbf{\eta} \sim \mathcal{N}\left(\mbf{0}, \sigma_{\mathrm{noise}}^2 I_{dNML}\right),
\]
where \(\sigma_{\mathrm{noise}}^2\) will be estimated from the data. Then, the likelihood is given by 
\begin{equation}
\label{eq:likelihood}
    p(\mbf{b}|\mbf{w},\sigma_{\mathrm{noise}}^2) = (2\pi)^{-dNLM/2}\sigma_{\mathrm{noise}}^{-dNLM} \exp\left\{ -\frac{||\mbf{b} - \mbf{\Phi} \mbf{w}||^2}{2\sigma_{\mathrm{noise}}^2}\right\}.
\end{equation}
\begin{remark}
The likelihood function given by \eqref{eq:likelihood} is a classical assumption for SBL that was considered, for instance, in \cite{tipping2001sparse}. More precisely, based on \eqref{eq:noise}, we can also restrict $\mbf{\eta}$ to be independent but not necessarily identically distributed, i.e., $\mbf{\eta} \sim \mathcal{N}(\mbf{0}, \Sigma)$ for some block diagonal positive definite matrix $\Sigma$ \cite{zhang2011sparse}. In practice, since $\mbf{\eta}$ is a weighted average of $\tilde{\mbf{\epsilon}}$, the numerical results show that we can still find good estimates of $\mbf{w}$ under the identical assumption.
\end{remark}

\subsubsection{The Sparse Prior and Hyperprior}
 Direct application of maximum likelihood estimation for $\mbf{w}$ and $\sigma_{\mathrm{noise}}^2$ from \eqref{eq:likelihood} would usually lead to severe overfitting. To address this issue and seek a sparse solution, sparse Bayesian learning `constrains' the targeted parameters $\mbf{w}$ by introducing an explicit sparse prior probability distribution $p(\mbf{w}|\mbf{\Theta})$.
There are a variety of sparsity-promoting priors to choose from, including but not limited to TV-priors, spike-and-slab priors, Gaussian priors,  Laplace priors, and hyper-Laplacian distributions based on $l^p$-quasinorms with $0 < p < 1$ \cite{sant2022block,babacan2009bayesian,krishnan2009fast}.

In our paper, we consider the commonly used prior, the Gaussian prior, where 
\begin{equation}
    p(\mbf{w}|\mbf{\Theta}) = \prod_{k=1}^K p(w_k|\gamma_k) \quad \mathrm{and}\quad p(w_k|\gamma_k) = \mathcal{N}(0,\gamma_k^{-1}),
\end{equation}
with some hyperparameters $\mbf{\gamma} = (\gamma_1,\dots,\gamma_K)$. Note that $\gamma_k$ must be allowed to have distinctly different values for the conditionally Gaussian prior to promote sparsity, and this can be achieved by treating $\mbf{\gamma}$ as random variables with uninformative hyperprior $p(\mbf{\gamma})$, like a uniform hyperprior (over a logarithmic scale). Since all scales are equally likely, a pleasing consequence of the use of such `improper' hyperpriors here is that of scale-invariance \cite{tipping2001sparse}. 
\begin{remark}
    It has been shown that the sparse Bayesian learning method is equivalent to an iterative reweighted $L_1$ method, which is more efficient to find the maximally sparse representations compared with the regular $L_1$ regularized least square methods \cite{wipf2010iterative}.
\end{remark}

In our paper, we consider both a flat hyperprior and another hyperprior that enhances the knowledge of sparsity in our experiments. We employ a hierarchical Laplace hyperprior as proposed in \cite{babacan2009fast}, which is constructed in two stages with hyperparameters $(\mbf{\gamma},\lambda)$:

\begin{itemize}
    \item \textbf{Stage 1:} Each \(\gamma_k\) follows a gamma distribution, which promotes sparsity:
    \begin{equation}
        p(\gamma_k | \lambda) = \Gamma(\gamma_k | 1, \lambda/2).
    \end{equation}
    
    \item \textbf{Stage 2:} The hyperparameter \(\lambda\) is modeled using a non-informative Jeffrey’s prior:
    \begin{equation}
        p(\lambda) \propto \frac{1}{\lambda}.
    \end{equation}
\end{itemize}

\begin{remark}
The $L_1$ regularization formulation is equivalent to using a direct Laplace prior on the coefficients $\mbf{w}$, i.e., $p(\mbf{w}|\lambda) = \frac{\lambda}{2}\exp(-\lambda|\mbf{w}|)$ \cite{babacan2009fast}. However, this formulation of the Laplace prior does not allow for a tractable Bayesian analysis, since it is not conjugate to the likelihood. To alleviate this, the hierarchical priors are employed.
\end{remark}

\begin{remark}
The Gamma distribution is a common choice for hyperpriors due to its conjugacy with conditionally Gaussian models, enabling tractable inference. Compared to Gaussian priors, the Laplace prior more strongly promotes sparsity by concentrating mass near zero and along coordinate axes. Its log-concavity ensures unimodal posteriors, reducing the risk of local minima and simplifying optimization \cite{babacan2009fast}.
\end{remark}

\subsubsection{Bayesian Inference}
\paragraph{Posterior of $\mbf{w}$ given $\mbf{\Theta}$}
Combining the likelihood and prior using Bayes’ rule, the posterior distribution of all the unknowns $(\mbf{w},\mbf{\Theta})$ with data $\mbf{b}$ is given by:
\begin{equation}
    p(\mbf{w}, \mbf{\Theta} | \mbf{b}) = \frac{p(\mbf{b} | \mbf{w}, \mbf{\Theta}) p(\mbf{w}, \mbf{\Theta})}{p(\mbf{b})}.
\end{equation}
Since the marginal likelihood \(p(\mbf{b})\) is intractable, we decompose the posterior as:
\begin{equation}
\label{eq:posterior_decomp}
    p(\mbf{w}, \mbf{\Theta} | \mbf{b}) = p(\mbf{w} | \mbf{b}, \mbf{\Theta}) p(\mbf{\Theta}| \mbf{b}).
\end{equation}

Therefore, given $\mbf{\Theta}$, the first term, i.e., the posterior distribution of $\mbf{w}$ given data, \(p(\mbf{w} | \mbf{b}, \mbf{\Theta})\), is Gaussian and can be computed analytically:
\begin{equation}
    p(\mbf{w} | \mbf{b}, \mbf{\Theta}) = \mathcal{N}(\mbf{\mu}, \mbf{\Sigma}),
\end{equation}
where
\begin{equation}
\label{eq:posterior}
    \mbf{\Sigma} = \left(\mathrm{diag}(\bgamma) + \sigma_{\mathrm{noise}}^{-2} \mbf{\Phi}^T \mbf{\Phi}\right)^{-1}, \quad \mbf{\mu} = \sigma_{\mathrm{noise}}^{-2} \mbf{\Sigma} \mbf{\Phi}^T \mbf{b}.
\end{equation}

\paragraph{Optimization of hyperparameters $\mbf{\Theta}$}
For the second term, instead of applying Bayesian inference over those hyperparameters $\mbf{\Theta}$ (which is analytically intractable), sparse Bayesian learning is formulated as a type-II maximum likelihood procedure, i.e., the (local) maximisation with respect to $\mbf{\Theta}$ of the marginal likelihood  \(p(\mbf{b}, \mbf{\Theta})\), or equivalently its logarithm \(\mathcal{L}(\mbf{\Theta})\),
\begin{equation}
\label{eq:loglik}
    \mathcal{L}(\mbf{\Theta}) = \log p(\mbf{b}, \mbf{\Theta}) = \log \int p(\mbf{b} | \mbf{w}, \mbf{\Theta}) p(\mbf{w}|\mbf{\Theta})\ d\mbf{w}.
\end{equation}

Once the most-probable hyperparameters $\Theta_{MP}$ have been found, they can be plugged into (27) to give the posterior mean, which serves as the point estimate for the parameters, $\hat{w}_{MP}=\mu_{MP}$.

Under the Gaussian likelihood and prior assumptions, the marginal likelihood can be evaluated in closed form and depends on the covariance matrix
\[
\mbf{C} =
\sigma_{\mathrm{noise}}^{2} I + \mbf{\Phi}\,\mathrm{diag}(\bgamma^{-1})\,\mbf{\Phi}^{T}.
\]
\indent Evaluating the log marginal likelihood and its updates requires solving linear systems involving $\mbf{C}$ or computing quantities derived from its inverse. Since $\mbf{C}\in\mathbb{R}^{n\times n}$, where $n$ is the number of observations, a naive implementation would require $O(n^{3})$ operations per update. Efficient implementations exploit matrix identities to update the required quantities without recomputing the full matrix inverse at every iteration.

Empirically, the local maximisation of the marginal likelihood \(p(\mbf{b}, \mbf{\Theta})\) with respect to \(\mbf{\Theta}\) has been seen to work highly effectively with the greedy algorithm which select a candidate $\gamma_k$ at each iteration and update $\gamma_k$ by maximizing $\mathcal{L}(\mbf{\Theta})$ when all components of $\mbf{\gamma}$ except $\gamma_k$ are kept fixed. For the flat hyperprior, we adopt the algorithm proposed in \cite{tipping2003fast}, and for the hierarchical Laplacian hyperprior, we adopt the algorithm proposed in \cite{babacan2009fast}. In these algorithms, the quantities required for the marginal likelihood evaluation can be updated efficiently using rank-one updates, reducing the per-iteration computational cost from cubic to approximately quadratic in the number of observations.

A notable property of this procedure is that the optimal values of many of the hyperparameters $\gamma_k$ would be infinite \cite{tipping2001sparse}, which leads to a posterior with many weights $w_k$ infinitely peaked at zero and results in the sparsity of the model.

Recent work has explored alternative inference strategies for hierarchical Bayesian sparsity models based on maximum a posteriori (MAP) estimation and block-coordinate descent algorithms. These approaches can offer improved computational scalability for certain large-scale inverse problems \cite{calvetti2020sparse,glaubitz2023generalized,lindbloom2025efficient}. While the present work adopts the classical type-II maximum likelihood formulation commonly used in sparse Bayesian learning \cite{tipping2001sparse,tipping2003fast,babacan2009fast}, integrating these MAP-based optimization strategies into the proposed framework is an interesting direction for future research.

Finally, for the noise variance \(\sigma_{\mathrm{noise}}^2\), differentiation of \eqref{eq:likelihood} leads to the re-estimate:
\begin{equation}
\label{eq:noise_hat}
    \hat\sigma_{\mathrm{noise}}^2 = \frac{||\mbf{b}-\mbf{\Phi}\hat{\mbf{w}}||}{dNML-\sum_k \gamma_k}
\end{equation}

\begin{remark}
Although a Gamma prior can be placed on $\sigma_{\mathrm{noise}}^2$ for full Bayesian estimation, this approach is unreliable in practice when combined with greedy algorithms. In early iterations, poor reconstructions lead to inaccurate noise estimates, which in turn degrade subsequent updates. To avoid this instability, we fix $\sigma_{\mathrm{noise}}^2 = 0.01||\mbf{b}||^2_2$ at initialization, following the heuristic in \cite{babacan2009fast}.
\end{remark}

\subsection{Model Selection} 
\label{sec: MS}

In general, the true basis function \(\xi_{k^\ast}\) in the representation of \(\phi\) is unknown. To address this, we test each candidate function \(\xi_{k^\ast}\) for \(k^\ast \in \{1, \dots, K\}\) until a model in \eqref{eq:nonimplicit1} yields a sparse and accurate solution. When the candidate function \(\xi_{k^\ast}\) is not part of the true representation, the resulting coefficient vector \(\mbf{w}\) will lack sparsity, and the prediction error will be large. Conversely, when the correct candidate is selected, the coefficient vector \(\mbf{w}\) will be sparse, and the prediction error will be small. This approach provides a clear indicator of when the correct model has been identified. Although redundant information arises—since each term in the correct model can be used as a candidate function on the left-hand side, and the resulting models can be cross-referenced—this process is highly parallelizable, allowing each candidate term to be tested simultaneously.

In practice, one may randomly select several (or all) candidate indices \(k\) and choose the one with the smallest error according to a pre-selected error measure as \(k^\ast\). Several valid measures for model selection are discussed in \cite{zhang2018robust,kaheman2020sindy}, e.g., the weighted predictive error, 
\begin{equation}
\label{eq:wPE}
    \mathrm{wPE} = \frac{||\Phi\mbf{w}-\Phi\hat{\mbf{w}}||_2^2}{||\hat{\mbf{w}}||_2^2},
\end{equation}
which evaluates the model fit, or the weighted estimation uncertainty,
\begin{equation}
\label{eq:wEU}
    \mathrm{wEU} = \frac{\mathrm{tr}(\hat{\mbf{\Sigma}}^{(k^\ast)})}{||\hat{\mbf{\mu}}^{(k^\ast)}||_2^2},
\end{equation}
which measures the relative uncertainty of the estimated coefficients. In this work, we propose a new criteria to select \(k^\ast\) as the candidate that minimizes the weighted total uncertainty (wTU) constructed by
\begin{equation}
\label{eq:loss}
\mathrm{wTU} = \frac{\hat{\sigma}^2_\mathrm{noise} + \mathrm{tr}(\hat{\mbf{\Sigma}}^{(k^\ast)})}{1+||\hat{\mbf{\mu}}^{(k^\ast)}||_2^2},
\end{equation}
which combines two sources of uncertainty. The term $\mathrm{tr}(\hat{\mbf{\Sigma}}^{(k^\ast)})$ measures the uncertainty in the estimated coefficients through the posterior covariance, while $\hat{\sigma}^2_\mathrm{noise}$ reflects the model misfit through the estimated noise variance. The normalization by $1+||\hat{\mbf{\mu}}^{(k^\ast)}||_2^2$ allows comparisons across different candidate choices of $k^\ast$.\\
\indent It is important to note that the covariance matrix $\hat{\mbf{\Sigma}}^{(k^\ast)}$ arises from the conditional posterior distribution \(p(\mbf{w}\mid \mbf{b},\mbf{\Theta})\), with the hyperparameters $\mbf{\Theta}$ fixed at their estimated values. Therefore, the resulting uncertainty measure should be interpreted as a conditional uncertainty estimate rather than a full Bayesian uncertainty quantification. A full Bayesian treatment would require exploring the joint posterior distribution \(p(\mbf{w},\mbf{\Theta}\mid \mbf{b})\), which accounts for uncertainty in both the coefficients and the hyperparameters; see, e.g., recent sampling-based approaches for sparse Bayesian learning \cite{glaubitz2025efficient}.\\
\indent Despite this limitation, the proposed wTU criterion provides a practical indicator for model selection in our setting. In the numerical experiments presented in Section~4, it consistently identifies the correct candidate \(k^\ast\) across different noise levels and data regimes. The entire procedure is summarized in Algorithm \ref{Algorithm:learning}, and the desired estimate of $\hat{\mbf{c}}_{-k^\ast}$ is given by $\hat{\mbf{\mu}}^{(k^\ast)}$ using \eqref{eq:posterior}.


\begin{algorithm}[!htb]
\algorithmicrequire\ $\tilde{\mathcal{D}}_{M,L}= \{\tilde{\bX}^{(m)}(t_l), \tilde{\dot{\bX}}^{(m)}(t_l)\}_{M,L}$ (training data), $\mathcal{S} = \{\xi_k\}_{k=1}^K$ (basis-functions)

\begin{algorithmic} [1]
\STATE compute $\Tilde{\mathbb{A}} = (\Tilde{\mbf{A}}^{(1,1)}, \dots, \Tilde{\mbf{A}}^{(M,L)})^T \in \mathbb{R}^{dNLM \times K}$ using \eqref{eq:assemble_A};
\FOR{$k^\ast$ = 1 to K}
    \STATE Calculate the posterior distribution $p(\mbf{w}_k|\Tilde{\mathbb{A}}_{k})$ with $ \Tilde{\mathbb{A}}_{k} = - \Tilde{\mathbb{A}}_{-k}\mbf{w}_k$ using the fast greedy algorithm, let the estimated posterior mean be $\hat{\mbf{\mu}}^{(k^\ast)}$, the variance be $\hat{\mbf{\Sigma}}^{(k^\ast)}$;
    \STATE Re-estimate the noise variance $\sigma_{\mathrm{noise}}^2$ using \eqref{eq:noise_hat};
\ENDFOR
\STATE pick $k^\ast$ with the smallest weighted total uncertainty defined by \eqref{eq:loss};
\STATE $\bar{\intkernel}^{\ast}=\xi_{k^\ast} + \mbf{\mu}^{(k^\ast)}\cdot \mbf{\xi}_{-k^\ast}$;
\end{algorithmic}
\algorithmicensure\ $\bar{\intkernel}^{\ast}$
\caption{{\bf Learning kernels}}
\label{Algorithm:learning}
\end{algorithm}

Note that the computational complexity for assembling the matrix $\Tilde{\mathbb{A}}$ is $\mathcal{O}(dN^2MLK)$, and for each iteration in the model selection procedure, the computational complexity is approximately $\mathcal{O}((dNML)^3K)$ since it involves computing an inversion of a matrix with dimension $dNML*dNML$ while evaluating $\mathcal{L}(\theta)$ and we assume $K \ll dNML$. However, in practice, we can reduce this order since the matrix is often very sparse, as shown in our numerical experiments.

\begin{remark}
Since one of the main goals of our strategy is to avoid the trivial solution/estimation, and for $\mathbb{A}_{k^\ast} \approx 0$ (equals zero at most entries), it implies that the $c_{k^\ast} = 0$ or we do not have enough information to learn the right coefficient of $\xi_k$ with the given observational data. Therefore, in our numerical examples, we exclude the candidate $k^\ast$ if more than half of the entries of $\mathbb{A}_{k^\ast}$ are zero. This can not only accelerate the model selection process, but also exclude the misspecification caused by inadequate information.     
\end{remark}

\subsection{Uncertainty Quantification}

The posterior variance \(\mbf{\Sigma}^{(k^\ast)}\) can be used as a good indicator for the uncertainty of the estimation for \(\hat{\mbf{c}}_{-k^\ast}\) based on our Bayesian approach. Therefore, for the estimation of \(\phi\), we can construct an uncertainty quantification region for any \(r \in \mathbb{R}^+\) by:
\begin{align}
    \xi_{k^\ast}(r) + & \sum_{k \neq k^\ast} \min\{ (\mbf{\mu}^{(k^\ast)}_k-2\mbf{\Sigma}^{(k^\ast)}_{kk}) \xi_k(r), (\mbf{\mu}^{(k^\ast)}_k+2\mbf{\Sigma}^{(k^\ast)}_{kk}) \xi_k(r)\} \leq \bar{\intkernel}^{\ast} \notag\\
    & \quad \leq \xi_{k^\ast}(r) + \sum_{k \neq k^\ast} \max\{ (\mbf{\mu}^{(k^\ast)}_k-2\mbf{\Sigma}^{(k^\ast)}_{kk}) \xi_k(r), (\mbf{\mu}^{(k^\ast)}_k+2\mbf{\Sigma}^{(k^\ast)}_{kk}) \xi_k(r)\}.
\end{align}
where the mean $\mbf{\mu}^{(k^\ast)}$ and the covariance matrix $\mbf{\Sigma}^{(k^\ast)}$ are again given by \eqref{eq:posterior} with the estimated $\mbf{\gamma}$ and $\sigma_{\mathrm{noise}}^2$.

\section{Numerical Experiments}

\paragraph{Numerical setup.} We consider a system of $N$ interacting agents, where each agent's state vector lies in $\mathbb{R}^d$. We simulate agent trajectories over the time interval $[0,T]$, with parameters specified in Tables \ref{t:OD_params}, \ref{t:CK_params}, and with i.i.d. initial conditions sampled from a prescribed uniform distribution. For training data, we generate $M$ independent trajectories, each observed at $L$ equidistant time points $0 = t_1 < t_2 < \cdots < t_L = T$. Additive i.i.d. Gaussian noise is applied to the agent velocities, $\dot\bX^{(m)}(t_l)$, as specified in \eqref{eq:noisemodel} with different noise levels w.r.t. the average velocity, i.e. the standard deviation of the noise equals to the percentages $(0\%, 5\%, 10\%, \dots)$ of the mean
\begin{equation*}
    \bar{v} = \frac{1}{NML}\sum_{i,m,l=1}^{N,M,L} \dot\bx_i^{(m)}(t_l).
\end{equation*}
And the empirical distribution of the pairwise distance, denoted $\rho_r$, is computed by the pairwise distance in each training dataset respectively,
\begin{equation*}
    \rho_r^{\mathcal{D}_{M,L}}(dr) = \frac{1}{ML}\sum_{m,l=1}^{M,L} \frac{1}{N(N-1)} \sum_{i=1}^N\sum_{i' \in \mathcal{N}{i}} \delta_{||\bx_{i'}^{(m)}(t_l) - \bx_{i}^{(m)}(t_l)||}(dr),
\end{equation*}
where $\delta$ is the Dirac distribution.
Since the learned interaction kernel $\hat{\phi}$ may differ from the ground truth 
$\phi_{true}$ by a multiplicative constant, we normalize all estimations using a relative $L_1$-norm weighted by the empirical distribution $\rho_r$:
\begin{equation*}
    \frac{\int |\phi_{true}(r)|\ d\rho_r}{\int |\hat{\phi}(r)|\ d\rho_r}.
\end{equation*}
This normalization enables consistent comparison across noise levels and data sizes.
All ODE systems are numerically integrated using MATLAB\textsuperscript{\textregistered} R2023a's \textrm{ode$15$s} solver, with relative and absolute tolerances set to $10^{-5}$ and $10^{-6}$ respectively.

\subsection{First-order systems: 1D and 2D opinion dynamics model}

Opinion dynamics models are widely used to understand collective behavior in social systems, including consensus formation, voting, and decentralized decision-making. We consider a first-order system of $N$ interacting agents, where interactions are inherently nonsymmetric. The evolution of each agent's opinion is governed by the following first-order differential equation:

\begin{eqnarray}
  \dot\bx_i &=&\sum_{i'=1}^N \frac{\phi(|\bx_{i'} - \bx_{i}|)}{\sum_{j \neq i}^N\phi(|\bx_j-\bx_i|)} (\bx_{i'} - \bx_i).
  \label{eq:1storderOD}
\end{eqnarray}

A widely used kernel in modeling opinion exchange involves employing a discontinuous interaction kernel. Such kernels effectively capture threshold-based interactions, where individuals adjust their opinions depending on the proximity of their current stance to others \cite{motsch2014heterophilious}. One example of such an interaction function is defined as:

\begin{equation}
  \intkernel(r) = 
  \begin{cases}
  1 & \textrm{if } 0 \leq r < 0.5\\
  0.1 & \textrm{if } 0.5 \leq r < 1\\
  0 & \textrm{if } r \geq 1
  \end{cases}
\label{eq:ODS_phi1}
\end{equation}

The interaction kernel $\phi$ characterizes non-repulsive interactions among agents, ensuring that each individual seeks to align its opinion with those of its connected neighbors. This framework effectively encapsulates key aspects of opinion formation, where both close-range and long-range interactions contribute to the evolution of collective consensus (Figure 1).

\begin{table}[H]
\centering
\begin{tabular}{|c|c|c|c|c|c|c|}
\hline 
$N$ & $M$ & $L$ & $[0,T]$ & Learning domain & $K$ & $\xi_k(r)$ \\  
\hline 
100 & 3 & 6 & $[0,5]$ & $[0,10]$ & 100 & $\chi_{[\frac{10(k-1)}{K},\frac{10k}{K}]}(r)$ \\  
\hline
\end{tabular}
\caption{System and learning parameters for the first-order opinion dynamics model.}
\label{t:OD_params}
\end{table}

\begin{figure}[!htb]
\centering
(a)\begin{minipage}{.45\textwidth} \centering  
\includegraphics[scale=0.21]{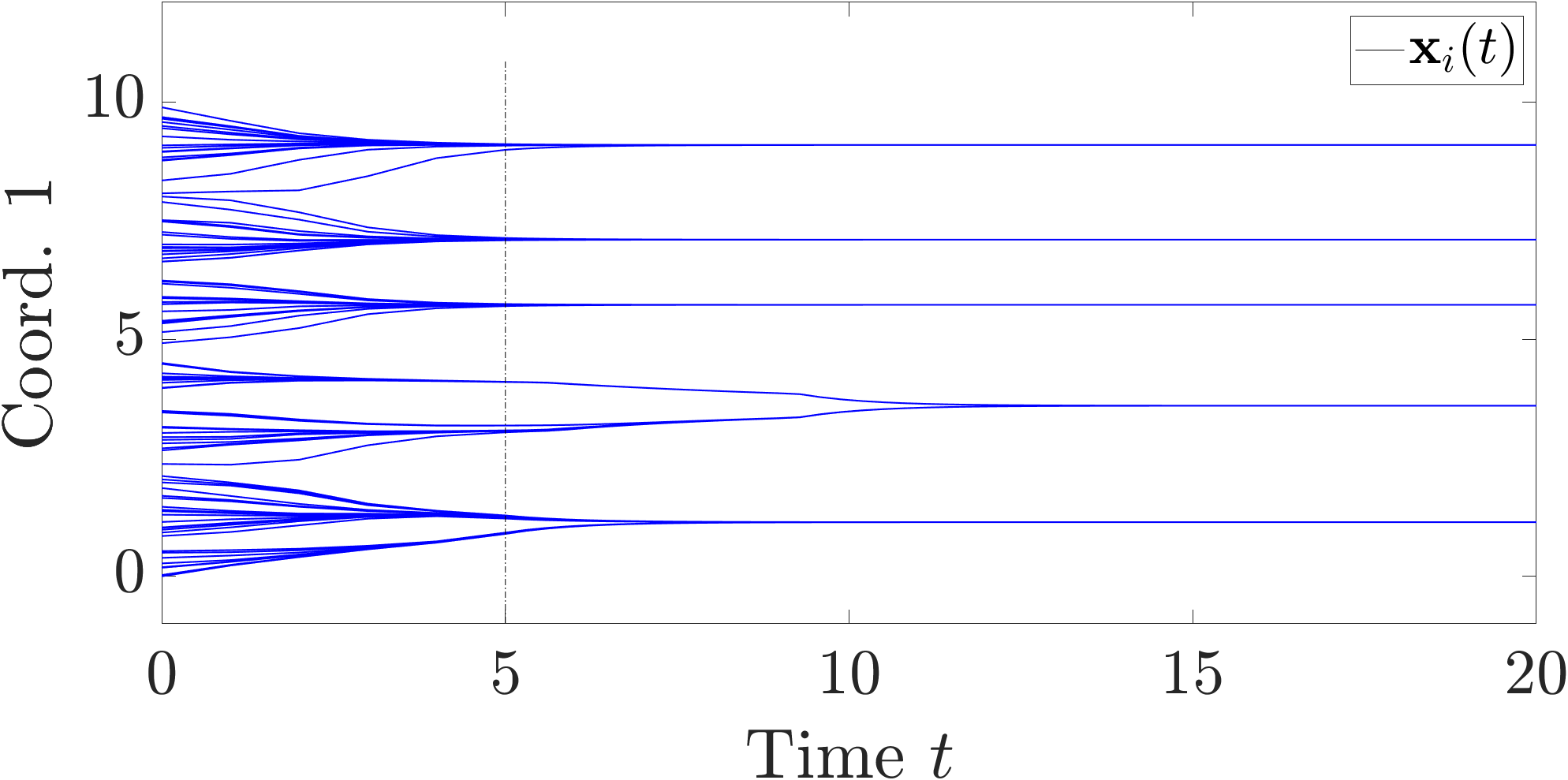}
\end{minipage}
(b)\begin{minipage}{.45\textwidth} \centering  
\includegraphics[scale=0.23]{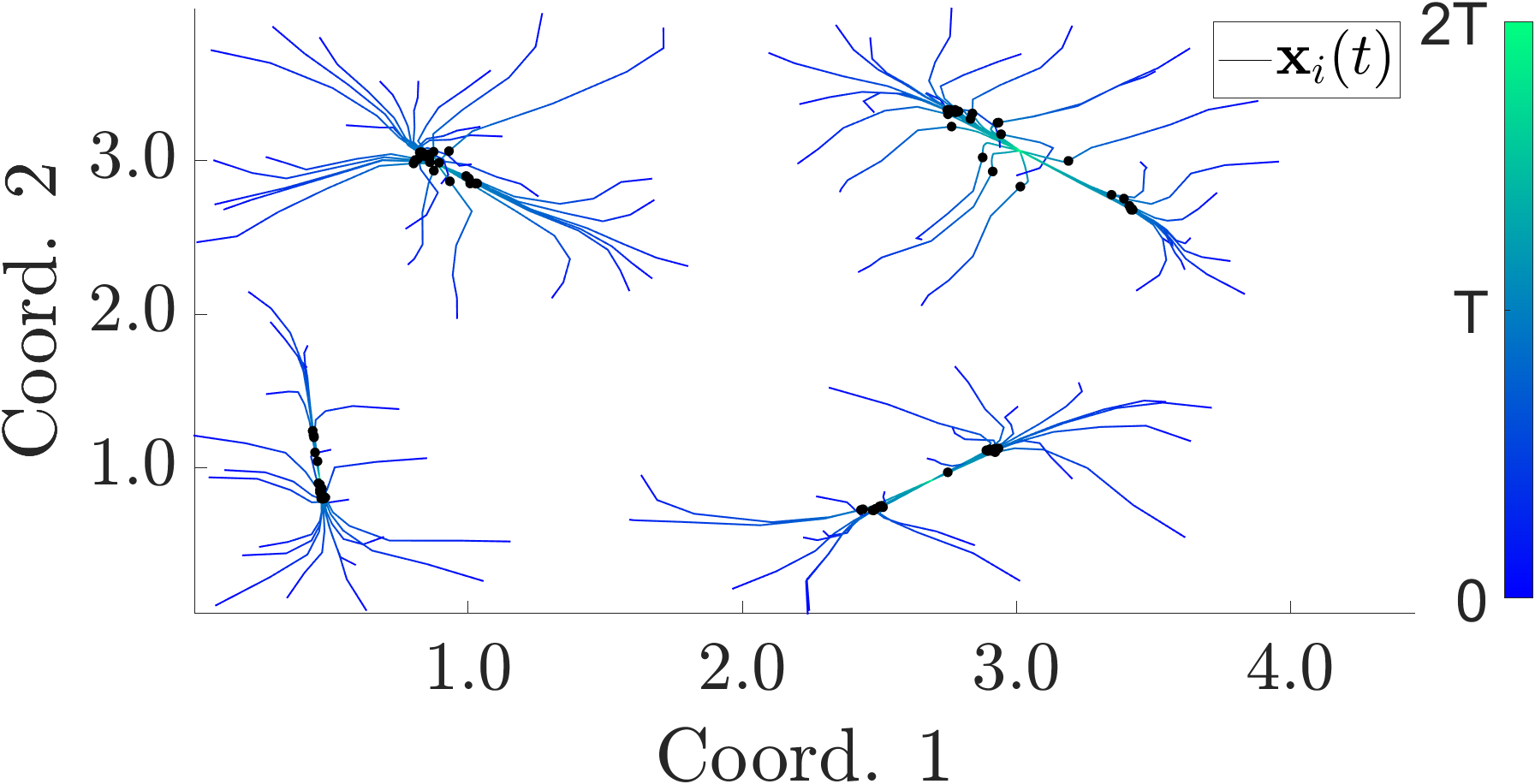}
\end{minipage}
\caption{Trajectory profiles for the opinion dynamics for 1D and 2D examples}
\label{ODtraj}
\end{figure}

We approximate the true interaction kernel using a piecewise constant expansion:
\begin{equation}
  \phi = \sum_{k=1}^{100}c_k\xi_k,
\end{equation}
where $\xi_k = \chi_{[\frac{10(k-1)}{K},\frac{10k}{K}]}(r)$ and the coefficients $c_k$ are given by:
\begin{equation}
  c_k = 
  \begin{cases}
  1 & \textrm{if } k = 1,\dots,5\\
  0.1 & \textrm{if } k = 6,\dots,10\\
  0 & \textrm{if } k = 11,\dots,100
  \end{cases}
\end{equation}
This representation matches the ground truth kernel in \eqref{eq:ODS_phi1} and enables exact recovery under ideal conditions.


Figure \ref{fig:OD_ex1} demonstrates that in both 1D and 2D dynamics, our method can capture the exact function when the noise level is small. And when noise increases, although high noise affects the estimation accuracy near $r=0$, where pairwise data is sparse, the model still captures critical structural features of $\phi$ (e.g., discontinuities), ensuring accurate long-term trajectory prediction.

\begin{figure}[!htb]
\centering
(a) \begin{minipage}{.46\textwidth} \centering    
\includegraphics[scale=0.18]{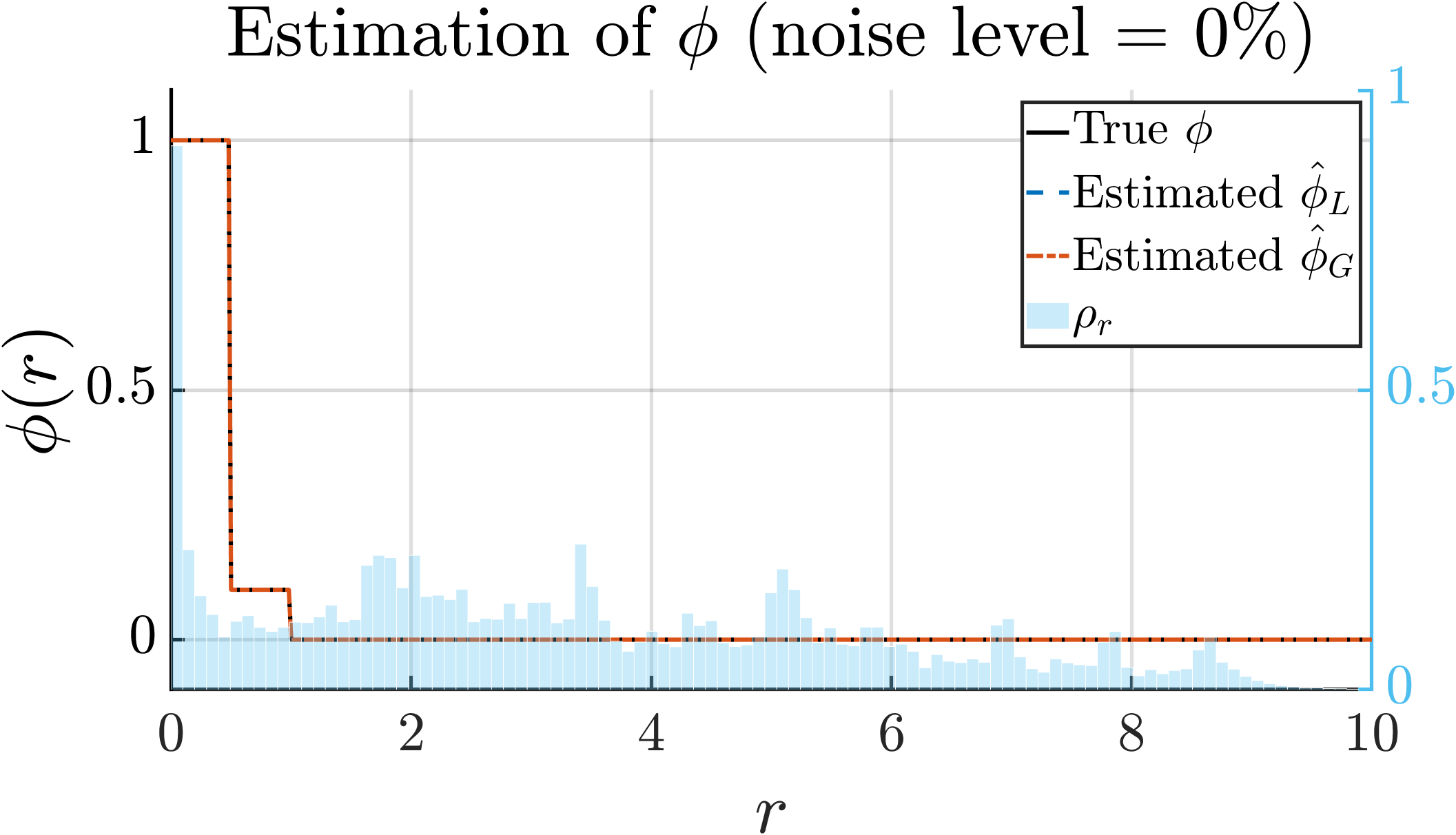}
\end{minipage}
(d)\begin{minipage}{.46\textwidth} \centering   
\includegraphics[scale=0.18]{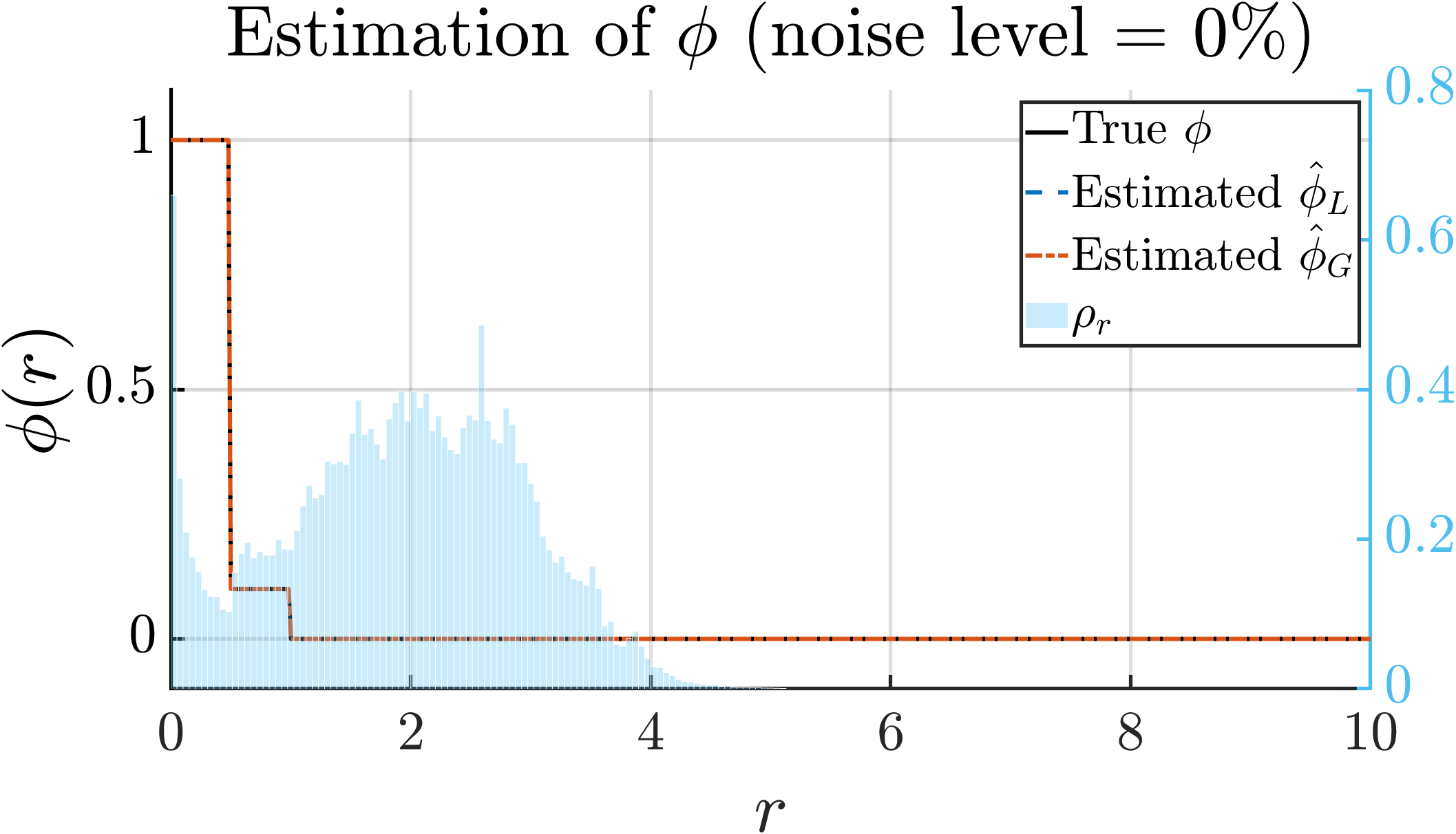}
\end{minipage}
(b) \begin{minipage}{.46\textwidth} \centering    
\includegraphics[scale=0.18]{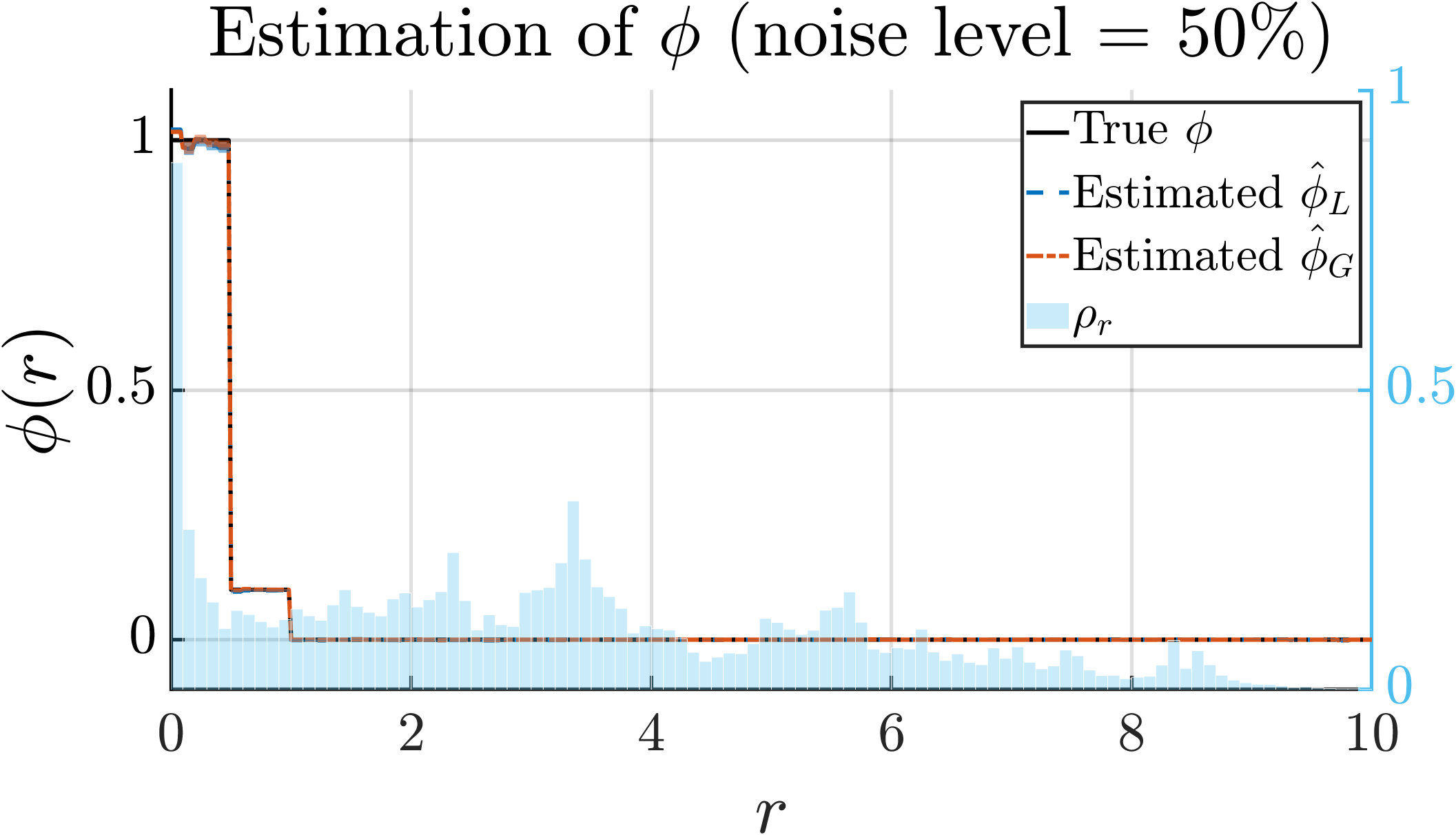}
\end{minipage}
(e)\begin{minipage}{.46\textwidth} \centering   
\includegraphics[scale=0.18]{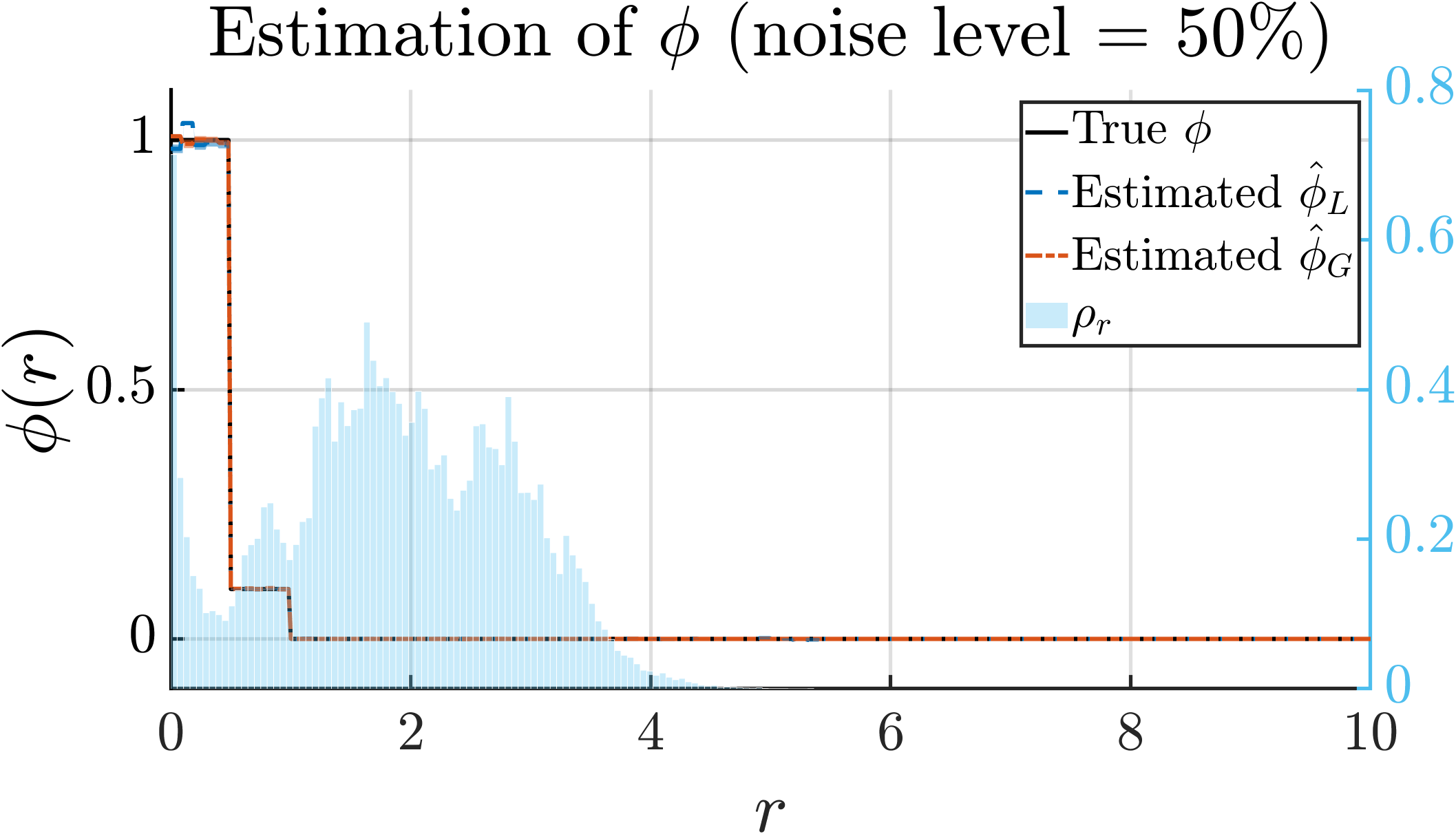}
\end{minipage}
(c) \begin{minipage}{.46\textwidth} \centering    
\includegraphics[scale=0.18]{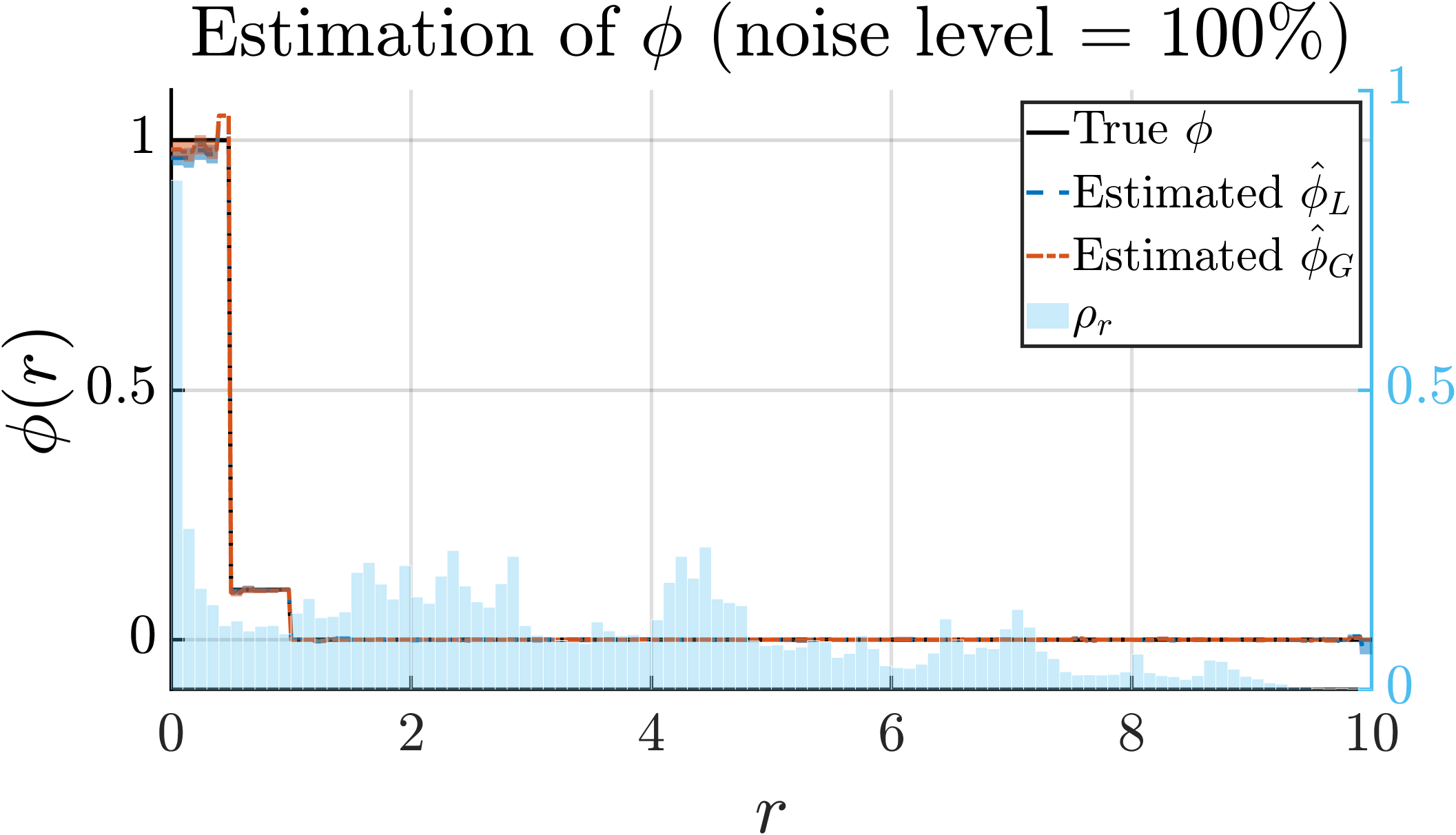}
\end{minipage}
(f)\begin{minipage}{.46\textwidth} \centering   
\includegraphics[scale=0.18]{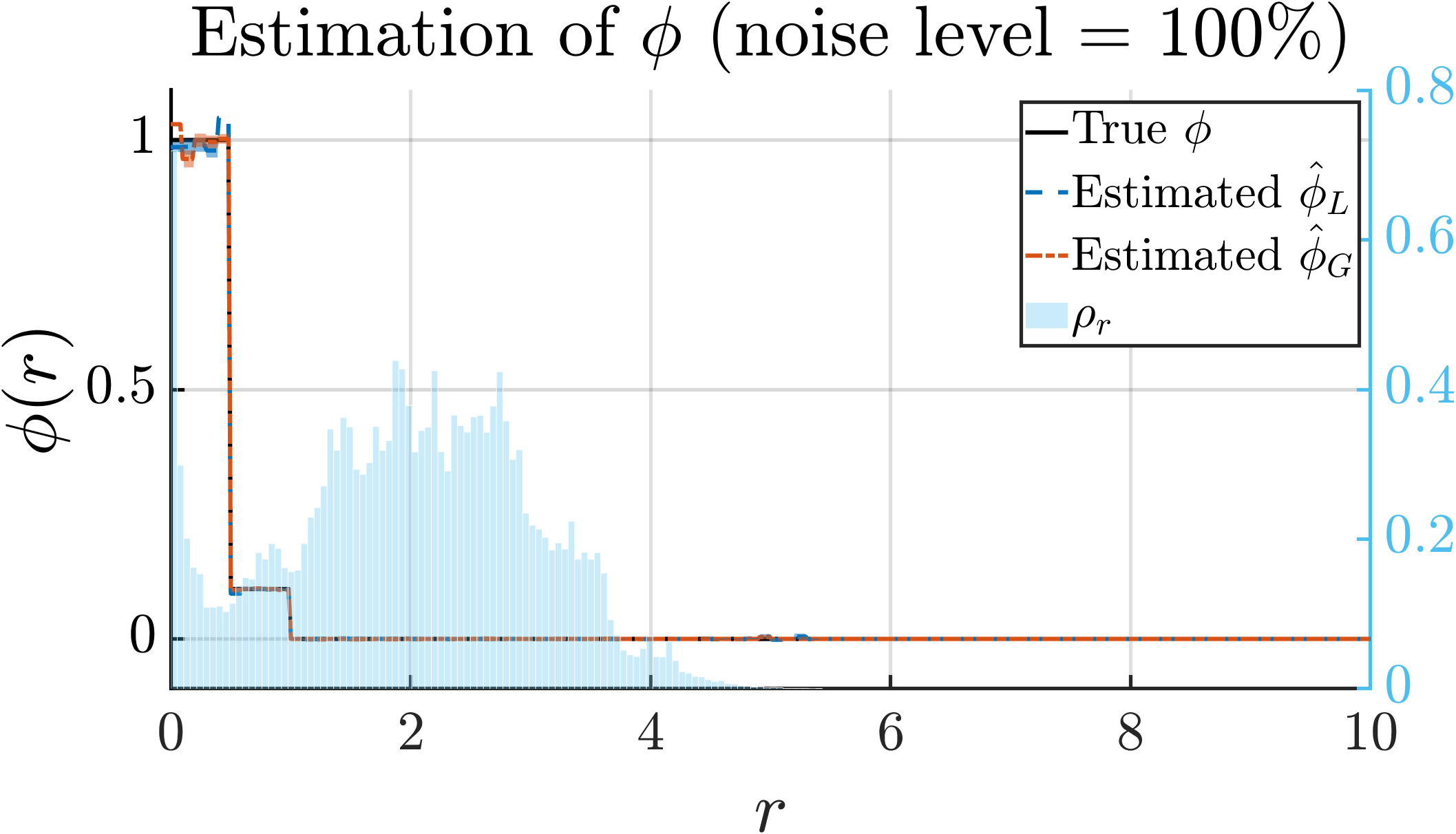}
\end{minipage}
\caption{Estimated interaction kernels $\hat{\phi}$ for 1D (left) and 2D (right) opinion dynamics under increasing noise levels ($(M,L) = (3,6)$). The curves $\hat{\phi}_L$ and $\hat{\phi}_G$ denote the estimates obtained using the Fast Laplace method (hierarchical Laplace hyperprior) and the Gaussian sparse Bayesian learning method (flat hyperprior), respectively, as described in Section~3. The solid black curve represents the true interaction kernel $\phi$. Different colored curves correspond to different noise levels in the observed data.}
\label{fig:OD_ex1}
\end{figure}

The predictive error of $\hat{\phi}_L$ (hierarchical Laplace hyperprior) in the relative $L_\infty$-norm with different training data sizes varied by $M$ is shown in Figure \ref{fig:OD_errors}. When the noise level is small, our method can precisely estimate $\phi$ with only a small number of observations. And when the noise level is large, increasing the number of observations would reduce the predictive error in $\phi$ as we expected. 

We also evaluate model selection performance using three criteria: wTU in \eqref{eq:loss} (our proposed thresholded utility loss), wPE in \eqref{eq:wPE} (predictive error), and wEU in \eqref{eq:wEU} (estimation uncertainty). Table \ref{tab:OD1_rate} summarizes the average success rate for each criterion in identifying the correct support $k^\ast \in \{1,\dots, 10\}$. Notably, wTU consistently selects the correct model across all noise levels, while the other two criteria fail under moderate to high noise, especially for wEU, which aligns with observations in \cite{zhang2018robust}, where small-magnitude coefficients dominate uncertainty estimates. Although we can add a threshold for the estimated coefficients to correct the model selection with wEU as suggested in \cite{zhang2018robust}, the wTU we proposed leverages thresholding over weighted test loss, offering more robustness without additional tuning.

\begin{figure}[!htb]
\centering
(a) \begin{minipage}{.45\textwidth} \centering    
\includegraphics[scale=0.5]{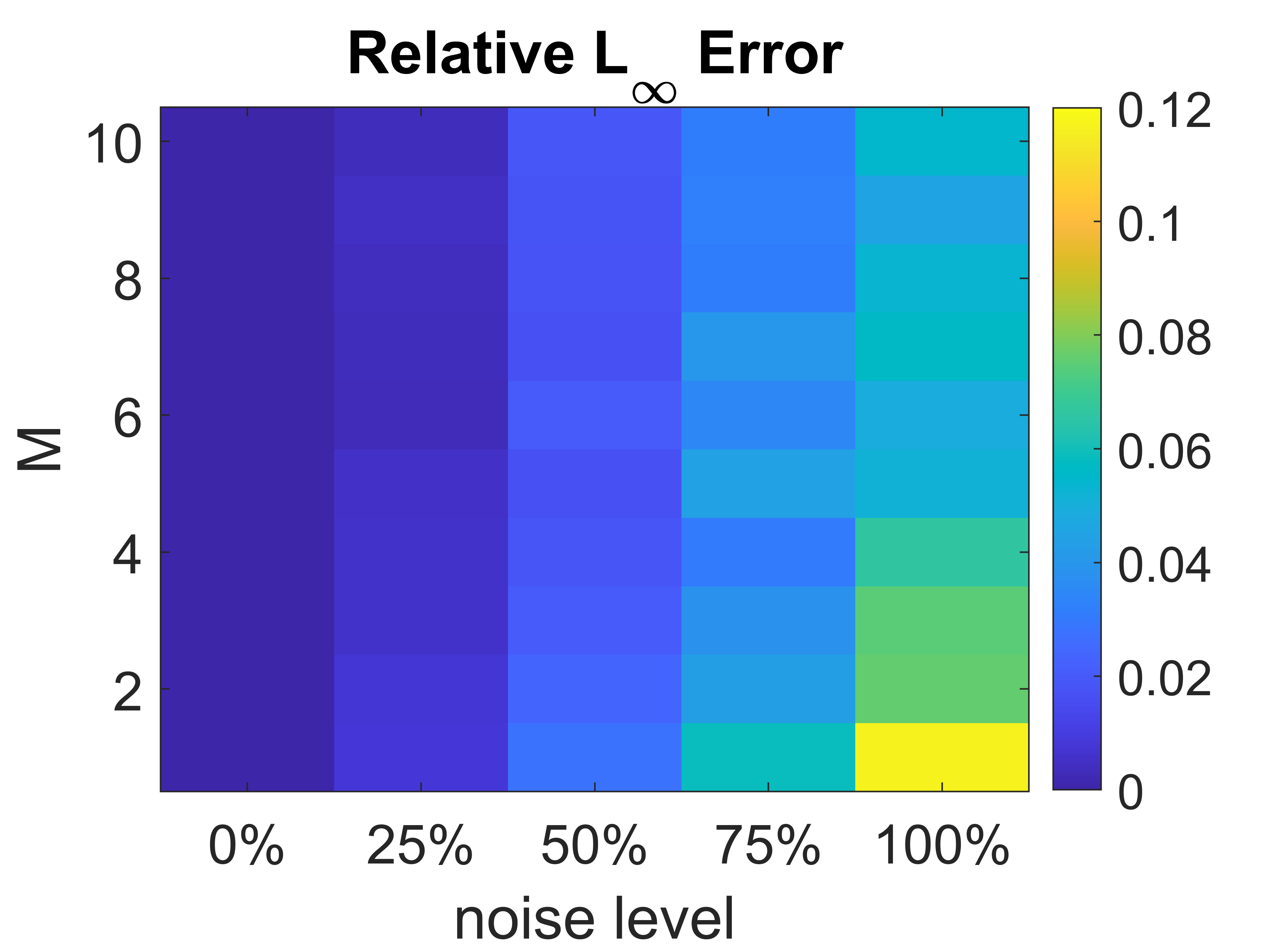}
\end{minipage}
(b) \begin{minipage}{.45\textwidth} \centering    
\includegraphics[scale=0.5]{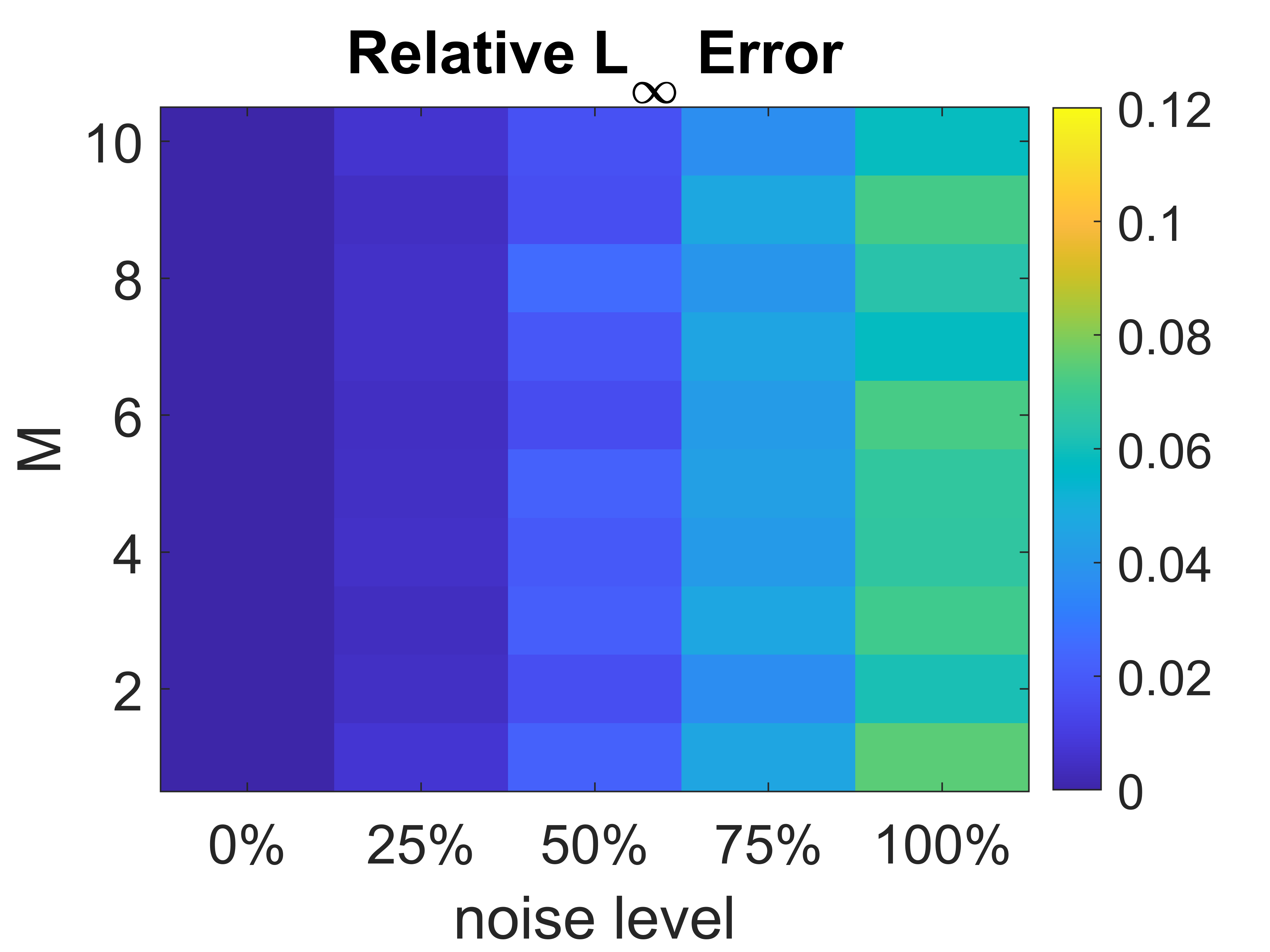}
\end{minipage}
\caption{Prediction error of the estimated interaction kernel $\hat{\phi}_L$ (hierarchical Laplace hyperprior) for (a) 1D and (b) 2D systems under varying numbers of trajectories $M$ and noise levels, with $L=6$ fixed. The error $\|\hat{\phi}_L - \phi\|$ quantifies the discrepancy between the estimate and the true kernel $\phi$.}
\label{fig:OD_errors}
\end{figure}

\begin{table}[tbhp]
\caption{The average successful rate of different model selection criteria.} \label{tab:OD1_rate}
\begin{center}
\begin{tabular}{cccccc}
\toprule
$d=1$ & noise level & & & & \\
\cmidrule(lr){1-1}\cmidrule(lr){2-6}
Criterion & $0\%$ & $25\%$ & $50\%$ & $75\%$ & $100\%$ \\
\cmidrule(lr){1-1}\cmidrule(lr){2-6}
wTU & 1.00 & 1.00 &	1.00 & 1.00 & 1.00\\
wPE & 1.00 & 1.00 &	1.00 & 0.98 & 0.98\\
wEU & 1.00 & 0.04 &	0 & 0 & 0\\
\midrule
$d=2$ & noise level & & & & \\
\cmidrule(lr){1-1}\cmidrule(lr){2-6}
Criterion & $0\%$ & $25\%$ & $50\%$ & $75\%$ & $100\%$ \\
\cmidrule(lr){1-1}\cmidrule(lr){2-6}
wTU & 1.00 & 1.00 &	1.00 & 1.00 & 1.00\\
wPE & 1.00 & 1.00 &	1.00 & 0.99 & 0.99\\
wEU & 1.00 & 0.11 &	0 & 0 & 0\\
\bottomrule
\end{tabular} 
\end{center}
\end{table}

Figure \ref{fig:OD_runtime} shows the average running time for estimating $\phi$ with the Laplace hyperprior given different sizes of datasets. The results are measured with $50$ trials for different $M$, and each of them includes $5$ different levels of noise as we measure the prediction errors above. The shadow regions indicate the standard deviation of the running times, and the results show that when $d=1$, the average running time is roughly controlled by $\mathcal{O}(M^{3/2})$, and $\mathcal{O}(M^{5/2})$ for $d=2$. As mentioned in section \ref{sec: MS}, the theoretical computational complexity is $\mathcal{O}(M^3)$, however, we can get lower complexity in practice because of the sparsity of our
assembled matrices.

\begin{figure}[!htb]
\centering
(a) \begin{minipage}{.45\textwidth} \centering    
\includegraphics[scale=0.5]{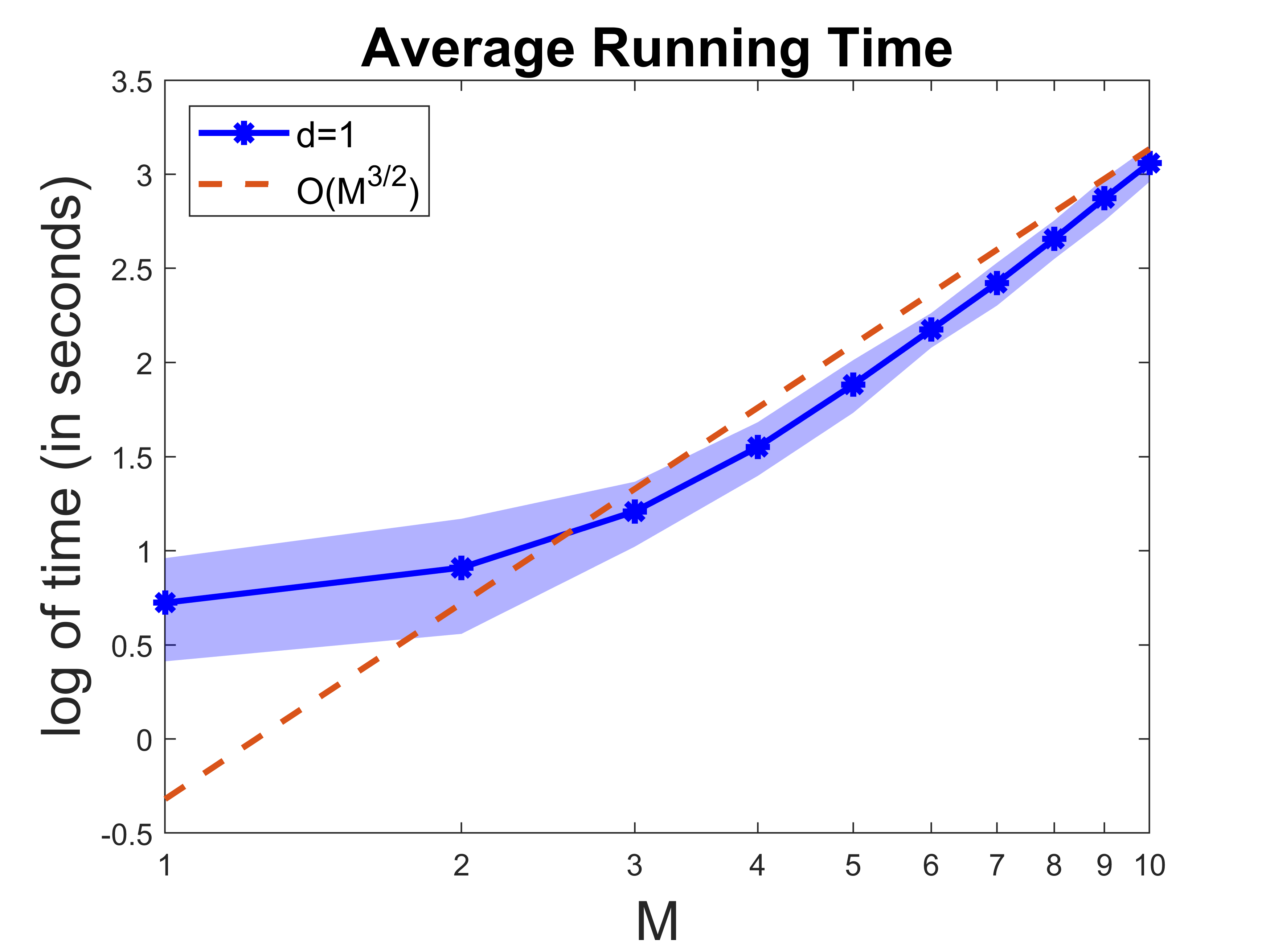}
\end{minipage}
(b)\begin{minipage}{.45\textwidth} \centering   
\includegraphics[scale=0.5]{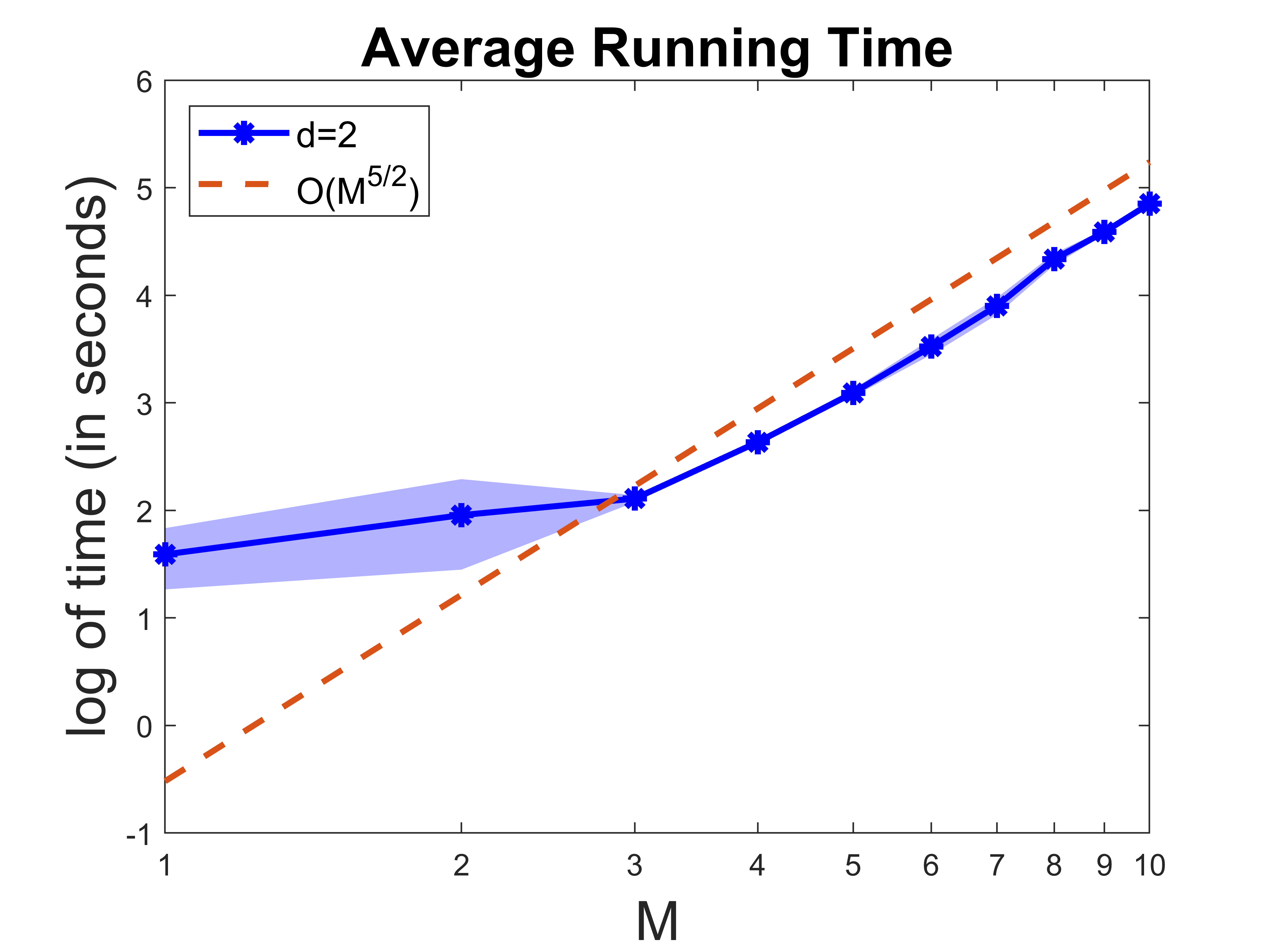}
\end{minipage}
\caption{Average running time for learning opinion dynamics models with two different dimensions.}
\label{fig:OD_runtime}
\end{figure}

\paragraph{Trajectory prediction error.} To complement the kernel-recovery results we measure the relative max-time trajectory error obtained by integrating the dynamics with the learned kernel $\widehat\phi$ and comparing against the ground-truth trajectory from the same initial condition. For a time window $\mathcal{W}$,
\begin{equation}
\mathrm{err}(\mathcal{W}) \;=\; \frac{\displaystyle\max_{t\in\mathcal{W}}\bigl\|x_{\text{true}}(t)-\widehat x(t)\bigr\|_2}{\displaystyle\max_{t\in\mathcal{W}}\|x_{\text{true}}(t)\|_2}.
\label{eq:err-pred}
\end{equation}
We report values on the prediction window $\mathcal{W}_{\text{pred}}=(5,T_f]$, i.e.\ outside the observation interval, evaluated at a fresh initial condition drawn from $\mu_0$ and held fixed across $M$, $\sigma_{\text{nsr}}$, and trials. Figure~\ref{fig:OD_traj_error} shows the resulting prediction error as a function of $M$ for both $d=1$ and $d=2$ at five noise levels. As more training trajectories are observed, the prediction error decreases steadily and its trial-to-trial variability shrinks, particularly under high noise. Figure~\ref{fig:OD2_trajEstimation} further demonstrates that the predicted trajectories remain accurate even in challenging settings with $M=3$ and $50\%$ noise. Tables~\ref{tab:OD-d1-traj}--\ref{tab:OD-d2-traj} list the corresponding numerical values.

\begin{figure}[!htb]
\centering
(a) \begin{minipage}{.45\textwidth} \centering
\includegraphics[width=\linewidth]{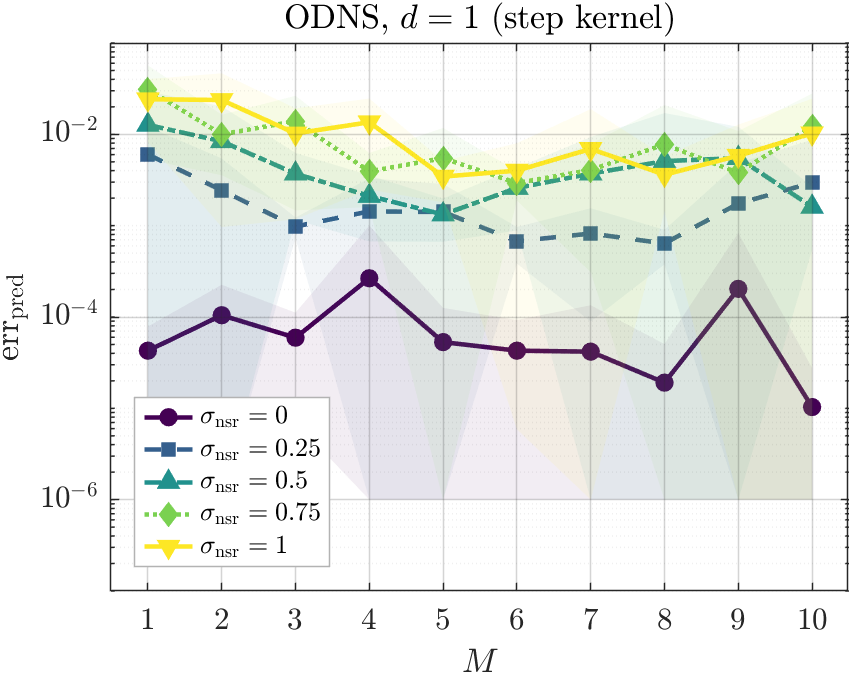}
\end{minipage}
(b) \begin{minipage}{.45\textwidth} \centering
\includegraphics[width=\linewidth]{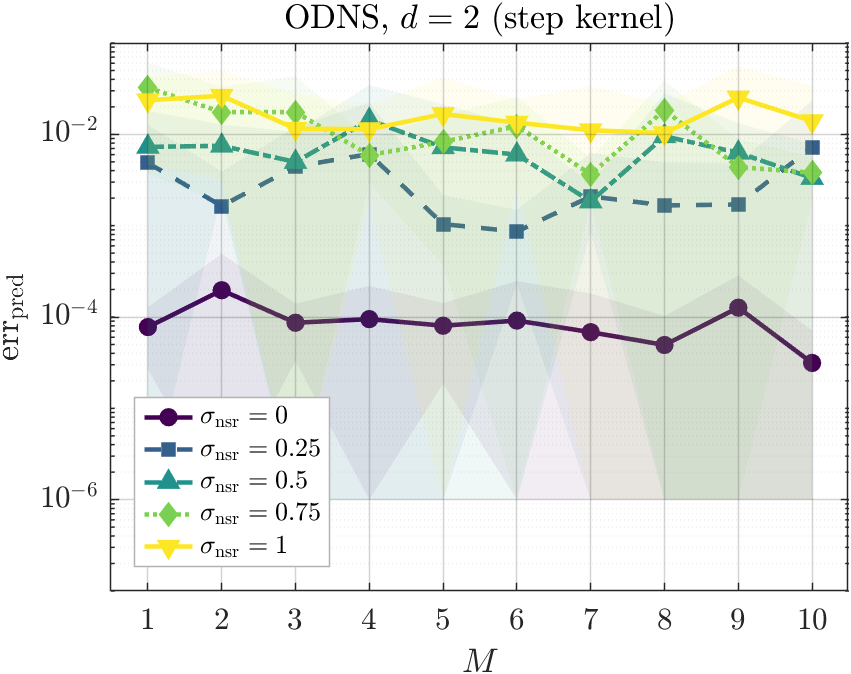}
\end{minipage}
\caption{Trajectory prediction error $\mathrm{err}(\mathcal{W}_{\text{pred}})$ for the opinion dynamics model under varying number of training trajectories $M$ and noise levels, with $L=6$ fixed. Solid lines show the mean across $T=10$ trials and shaded bands show the $\pm 1$ standard-deviation envelope.}
\label{fig:OD_traj_error}
\end{figure}

\begin{figure}[!htb]
\centering
\includegraphics[scale=0.4]{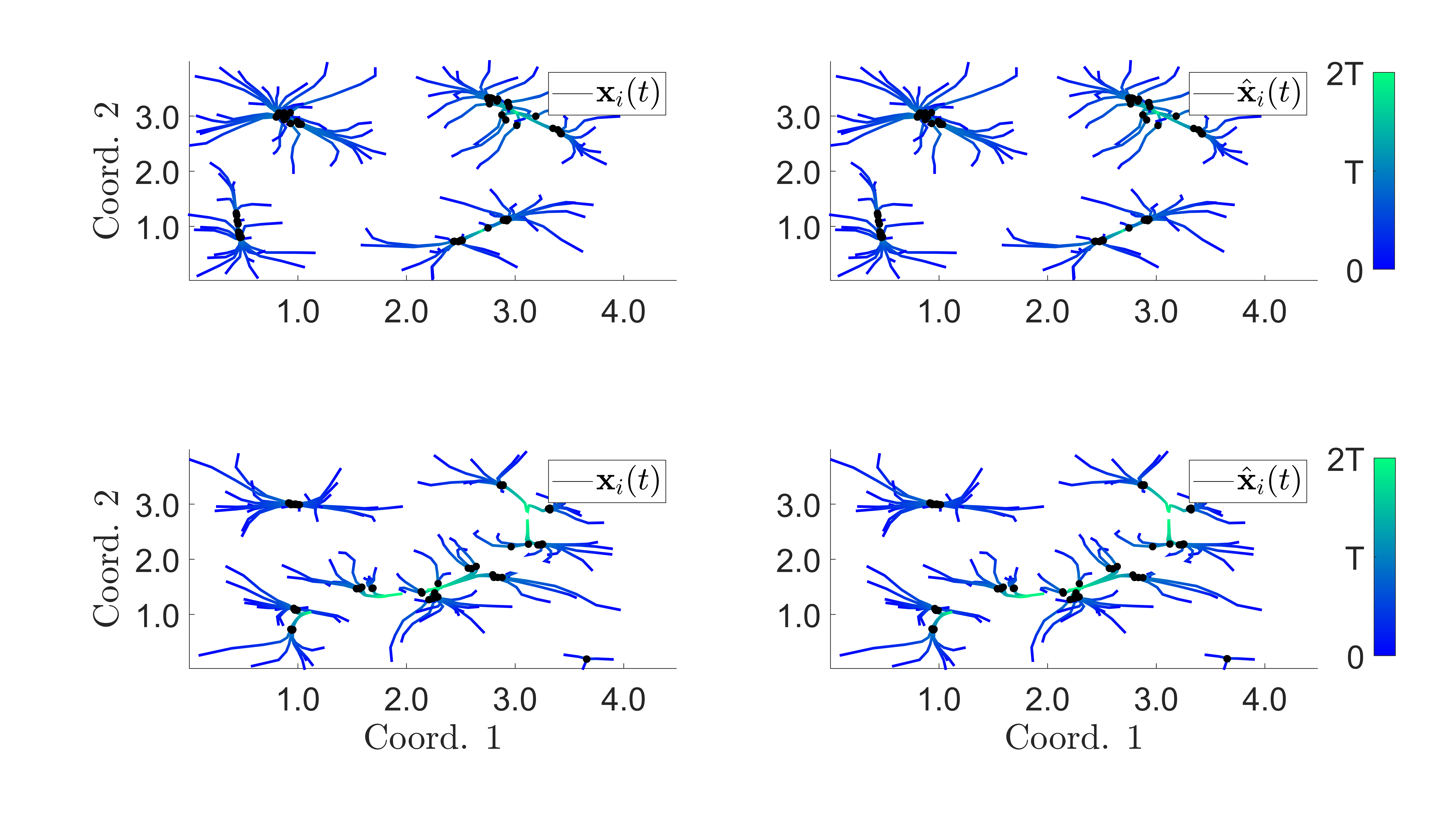}
\caption{Comparison of true (left) and predicted (right) agent trajectories using the estimated interaction kernel $\hat{\phi}_L$ with $(M,L) = (3,6)$ under 50\% noise. Despite local discrepancies in $\hat{\phi}$, the predicted system dynamics remain consistent with the ground truth trajectories.}
\label{fig:OD2_trajEstimation}
\end{figure}

\begin{table}[!htb]
\centering
\caption{ODNS, $d=1$. Relative max-time trajectory prediction error
$\mathrm{err}(\mathcal{W}_{\text{pred}})$ on the prediction window
$(5,T_f]$, evaluated at a fresh initial condition. Rows: number of
training trajectories $M$; columns: noise-to-signal ratio
$\sigma_{\text{nsr}}$. Each cell reports $\mu \pm s$ across $T=10$
independent trials.}
\label{tab:OD-d1-traj}
\renewcommand{\arraystretch}{1.15}
\setlength{\tabcolsep}{4pt}
\footnotesize
\begin{tabular}{c|ccccc}
\toprule
\multirow{2}{*}{$M$} & \multicolumn{5}{c}{$\sigma_{\text{nsr}}$} \\
\cmidrule(lr){2-6}
  & $0$ & $0.25$ & $0.5$ & $0.75$ & $1$ \\
\midrule
1 & $(4.3 \pm 3.6)\!\times\!10^{-5}$ & $(6.0 \pm 6.1)\!\times\!10^{-3}$ & $(1.3 \pm 1.7)\!\times\!10^{-2}$ & $(3.1 \pm 2.6)\!\times\!10^{-2}$ & $(2.4 \pm 1.6)\!\times\!10^{-2}$ \\
2 & $(1.0 \pm 1.2)\!\times\!10^{-4}$ & $(2.4 \pm 3.0)\!\times\!10^{-3}$ & $(0.8 \pm 1.2)\!\times\!10^{-2}$ & $(1.0 \pm 0.6)\!\times\!10^{-2}$ & $(2.4 \pm 2.3)\!\times\!10^{-2}$ \\
3 & $(5.9 \pm 5.2)\!\times\!10^{-5}$ & $(9.8 \pm 2.5)\!\times\!10^{-4}$ & $(3.7 \pm 2.5)\!\times\!10^{-3}$ & $(1.4 \pm 1.2)\!\times\!10^{-2}$ & $(1.0 \pm 0.9)\!\times\!10^{-2}$ \\
4 & $(2.7 \pm 7.4)\!\times\!10^{-4}$ & $(1.4 \pm 1.9)\!\times\!10^{-3}$ & $(2.1 \pm 1.4)\!\times\!10^{-3}$ & $(3.9 \pm 2.3)\!\times\!10^{-3}$ & $(1.4 \pm 1.1)\!\times\!10^{-2}$ \\
5 & $(5.3 \pm 7.3)\!\times\!10^{-5}$ & $(1.4 \pm 1.4)\!\times\!10^{-3}$ & $(1.3 \pm 0.7)\!\times\!10^{-3}$ & $(5.5 \pm 6.3)\!\times\!10^{-3}$ & $(3.4 \pm 1.6)\!\times\!10^{-3}$ \\
6 & $(4.3 \pm 5.0)\!\times\!10^{-5}$ & $(6.8 \pm 2.9)\!\times\!10^{-4}$ & $(2.6 \pm 1.7)\!\times\!10^{-3}$ & $(2.9 \pm 1.0)\!\times\!10^{-3}$ & $(4.0 \pm 4.0)\!\times\!10^{-3}$ \\
7 & $(4.2 \pm 9.3)\!\times\!10^{-5}$ & $(8.2 \pm 7.3)\!\times\!10^{-4}$ & $(3.7 \pm 5.6)\!\times\!10^{-3}$ & $(4.1 \pm 3.8)\!\times\!10^{-3}$ & $(0.7 \pm 1.2)\!\times\!10^{-2}$ \\
8 & $(1.9 \pm 3.1)\!\times\!10^{-5}$ & $(6.4 \pm 2.6)\!\times\!10^{-4}$ & $(0.5 \pm 1.2)\!\times\!10^{-2}$ & $(0.8 \pm 1.3)\!\times\!10^{-2}$ & $(3.6 \pm 2.2)\!\times\!10^{-3}$ \\
9 & $(2.0 \pm 6.3)\!\times\!10^{-4}$ & $(1.7 \pm 3.1)\!\times\!10^{-3}$ & $(5.5 \pm 6.6)\!\times\!10^{-3}$ & $(3.8 \pm 7.1)\!\times\!10^{-3}$ & $(5.8 \pm 6.9)\!\times\!10^{-3}$ \\
10 & $(1.0 \pm 1.7)\!\times\!10^{-5}$ & $(3.0 \pm 5.8)\!\times\!10^{-3}$ & $(1.6 \pm 1.1)\!\times\!10^{-3}$ & $(1.2 \pm 1.6)\!\times\!10^{-2}$ & $(1.0 \pm 1.4)\!\times\!10^{-2}$ \\
\bottomrule
\end{tabular}
\end{table}

\begin{table}[!htb]
\centering
\caption{ODNS, $d=2$. Relative max-time trajectory prediction error
$\mathrm{err}(\mathcal{W}_{\text{pred}})$ on the prediction window
$(5,T_f]$, evaluated at a fresh initial condition. Rows: number of
training trajectories $M$; columns: noise-to-signal ratio
$\sigma_{\text{nsr}}$. Each cell reports $\mu \pm s$ across $T=10$
independent trials.}
\label{tab:OD-d2-traj}
\renewcommand{\arraystretch}{1.15}
\setlength{\tabcolsep}{4pt}
\footnotesize
\begin{tabular}{c|ccccc}
\toprule
\multirow{2}{*}{$M$} & \multicolumn{5}{c}{$\sigma_{\text{nsr}}$} \\
\cmidrule(lr){2-6}
  & $0$ & $0.25$ & $0.5$ & $0.75$ & $1$ \\
\midrule
1 & $(7.8 \pm 5.1)\!\times\!10^{-5}$ & $(4.9 \pm 7.9)\!\times\!10^{-3}$ & $(0.7 \pm 1.1)\!\times\!10^{-2}$ & $(3.2 \pm 2.9)\!\times\!10^{-2}$ & $(2.4 \pm 1.9)\!\times\!10^{-2}$ \\
2 & $(2.0 \pm 3.0)\!\times\!10^{-4}$ & $(1.6 \pm 2.3)\!\times\!10^{-3}$ & $(7.5 \pm 5.3)\!\times\!10^{-3}$ & $(1.8 \pm 1.6)\!\times\!10^{-2}$ & $(2.6 \pm 2.3)\!\times\!10^{-2}$ \\
3 & $(8.7 \pm 5.3)\!\times\!10^{-5}$ & $(4.4 \pm 7.6)\!\times\!10^{-3}$ & $(5.0 \pm 5.7)\!\times\!10^{-3}$ & $(1.8 \pm 2.6)\!\times\!10^{-2}$ & $(1.2 \pm 1.6)\!\times\!10^{-2}$ \\
4 & $(0.9 \pm 1.2)\!\times\!10^{-4}$ & $(0.6 \pm 1.6)\!\times\!10^{-2}$ & $(1.4 \pm 2.0)\!\times\!10^{-2}$ & $(5.9 \pm 3.0)\!\times\!10^{-3}$ & $(1.1 \pm 1.0)\!\times\!10^{-2}$ \\
5 & $(8.0 \pm 6.2)\!\times\!10^{-5}$ & $(1.0 \pm 1.1)\!\times\!10^{-3}$ & $(0.7 \pm 1.4)\!\times\!10^{-2}$ & $(8.3 \pm 7.9)\!\times\!10^{-3}$ & $(1.7 \pm 2.5)\!\times\!10^{-2}$ \\
6 & $(0.9 \pm 1.6)\!\times\!10^{-4}$ & $(8.6 \pm 6.3)\!\times\!10^{-4}$ & $(6.0 \pm 6.7)\!\times\!10^{-3}$ & $(1.2 \pm 1.3)\!\times\!10^{-2}$ & $(1.3 \pm 1.1)\!\times\!10^{-2}$ \\
7 & $(0.7 \pm 1.1)\!\times\!10^{-4}$ & $(2.1 \pm 4.2)\!\times\!10^{-3}$ & $(1.8 \pm 1.0)\!\times\!10^{-3}$ & $(3.7 \pm 1.8)\!\times\!10^{-3}$ & $(1.1 \pm 2.3)\!\times\!10^{-2}$ \\
8 & $(4.9 \pm 5.3)\!\times\!10^{-5}$ & $(1.7 \pm 3.2)\!\times\!10^{-3}$ & $(1.0 \pm 1.9)\!\times\!10^{-2}$ & $(1.8 \pm 2.0)\!\times\!10^{-2}$ & $(1.0 \pm 1.1)\!\times\!10^{-2}$ \\
9 & $(1.3 \pm 1.6)\!\times\!10^{-4}$ & $(1.7 \pm 3.1)\!\times\!10^{-3}$ & $(6.3 \pm 7.0)\!\times\!10^{-3}$ & $(4.4 \pm 4.5)\!\times\!10^{-3}$ & $(2.5 \pm 3.1)\!\times\!10^{-2}$ \\
10 & $(3.1 \pm 4.0)\!\times\!10^{-5}$ & $(0.7 \pm 1.6)\!\times\!10^{-2}$ & $(3.3 \pm 4.5)\!\times\!10^{-3}$ & $(3.8 \pm 1.9)\!\times\!10^{-3}$ & $(1.4 \pm 2.0)\!\times\!10^{-2}$ \\
\bottomrule
\end{tabular}
\end{table}

\subsection{Second-order systems: Cucker-Smale models}

Our framework naturally extends to second-order systems such as the Cucker-Smale model, which describes the collective motion of interacting agents governed by second-order dynamics. Such models capture richer dynamics observed in physical and biological swarming systems, where alignment is velocity-driven. Consider the system of $N$ agents evolving according to the equation:
\begin{equation}\label{eq:2ndOrder}
    m_i\ddot{\mathbf{x}}_i = \sum_{i'=1}^N \frac{\phi(\|\mathbf{x}_{i'} - \mathbf{x}_i\|)}{\sum_{j=1}^{N}\phi(\|\mathbf{x}_{j} - \mathbf{x}_i\|)}(\dot{\mathbf{x}}_{i'} - \dot{\mathbf{x}}_i),
\end{equation} where $m_i$ denotes the mass of agent $i$, and $\ddot{\mathbf{x}}_i$ represents its acceleration. This formulation captures how each agent's acceleration is influenced by the relative velocities of its neighbors, with interactions weighted by distance. Notably, this model is well-suited for systems with non-homogeneous phase spaces, as introduced in \cite{motsch2011new}, where local interactions play a dominant role in shaping the global dynamics.

Our method can be adapted to infer the interaction dynamics in second-order systems with minor modifications. Based on \eqref{eq:2ndOrder}, given training data $\{\mathbf{X}^{(m,l)}, \dot{\mathbf{X}}^{(m,l)}, \ddot{\mathbf{X}}_{\sigma^2}^{(m,l)}\}$, we define the regression matrix by isolating $\phi(|\bx_j - \bx_i|)$ and projecting onto a set of basis functions ${\xi_k(r)}$, i.e., we construct the regression matrix $\mathbf{A} = [\mathbf{A}_{ik}] \in \mathbb{R}^{2dN \times K}$ as follows:
\begin{equation}\label{eq:assemble_A_2nd}
    \mathbf{A}_{ik} = \sum_{j \in \mathcal{N}_i} \xi_k(\|\mathbf{x}_j - \mathbf{x}_i\|)(m_i\ddot{\mathbf{x}}_i - (\dot{\mathbf{x}}_i  - \dot{\mathbf{x}}_j)),
\end{equation} where $\mathcal{N}_i$ denotes the neighborhood of agent $i$, and $\xi_k(\|\mathbf{x}_j - \mathbf{x}_i\|)$ is a basis function that depends on the inter-agent distance. 

To evaluate our framework, we consider two representative examples of second-order interaction dynamics using prototype kernels that model either multi-cluster or mono-cluster flocking behavior. Figure \ref{CKtraj} depicts the resulting trajectory profiles.

\begin{figure}[!htb]
\centering
(a)\begin{minipage}{.45\textwidth} \centering  
\includegraphics[scale=0.19]{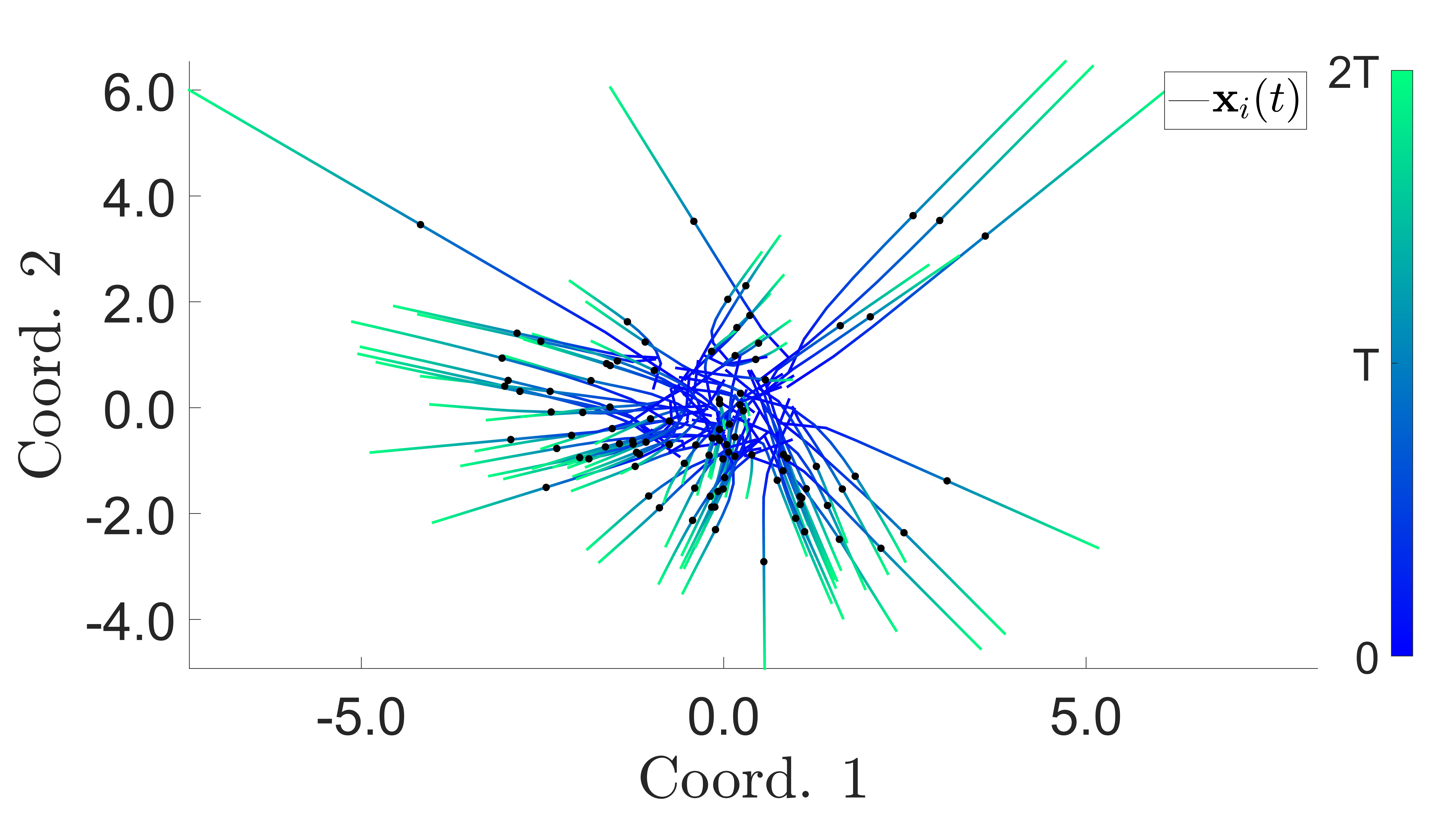}
\end{minipage}
(b)\begin{minipage}{.45\textwidth} \centering  
\includegraphics[scale=0.19]{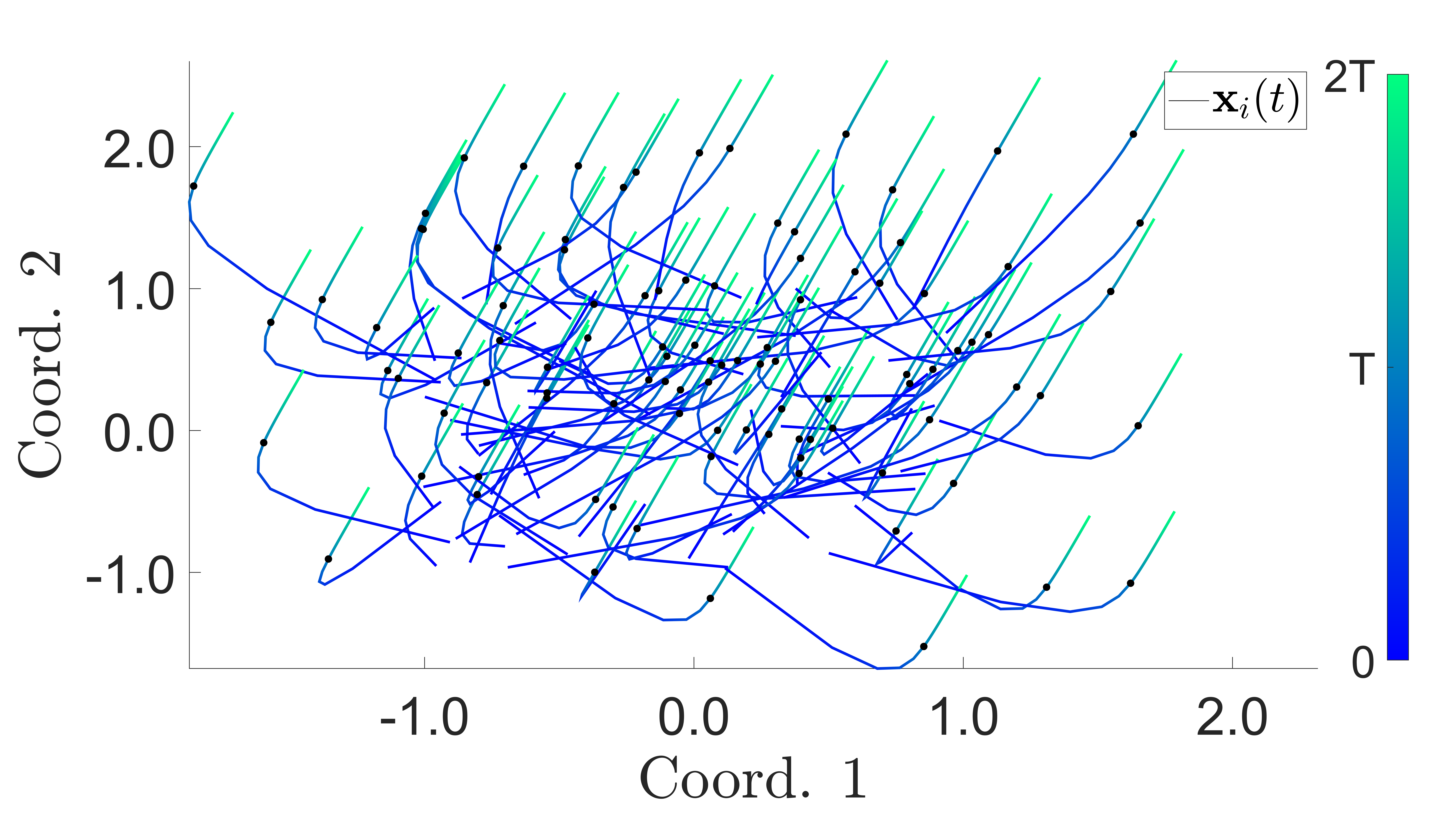}
\end{minipage}
\caption{Trajectory profiles for the Cucker-Smale system under different interaction kernels.}
\label{CKtraj}
\end{figure}

We consider a system of dimension $dN=200$ and summarize the common system parameters in Table \ref{t:CK_params}.

\begin{table}[H]
\centering
\begin{tabular}{|c|c|c|c|c|c|c|c|c|}
\hline 
$m_i$&$N$ & $d$ & $M$ & $L$ & $[0,T;T_f]$ & Learning domain & $K$ & $\xi_k(r)$ \\ 
\hline 
1&100 & 2 & 5 & 6 & $[0,5;10]$ & $[0,5]$ & 100 & $\chi_{[\frac{10(k-1)}{K},\frac{10k}{K}]}(r)$ \\ 
\hline
\end{tabular}
\caption{System and learning parameters for the Cucker-Smale model.}
\label{t:CK_params}
\end{table}

\paragraph{Example 1 (Cut-off  kernel at a finite distance)}
In the first example, we use a compactly supported interaction kernel:
\begin{equation*}
   \phi(r) = \chi_{[0,0.5]}(r),
\end{equation*}
which enforces interactions only between agents within a radius of 0.5. This leads to locally influenced multi-cluster flocking behavior, as shown in Figure \ref{CKtraj}(a). The corresponding estimation results using a piecewise constant basis are shown in Figure \ref{fig:CK1Estimation} and 9(a).

\begin{figure}[!htb]
\centering

(a) \begin{minipage}{.46\textwidth} \centering    
\includegraphics[scale=0.18]{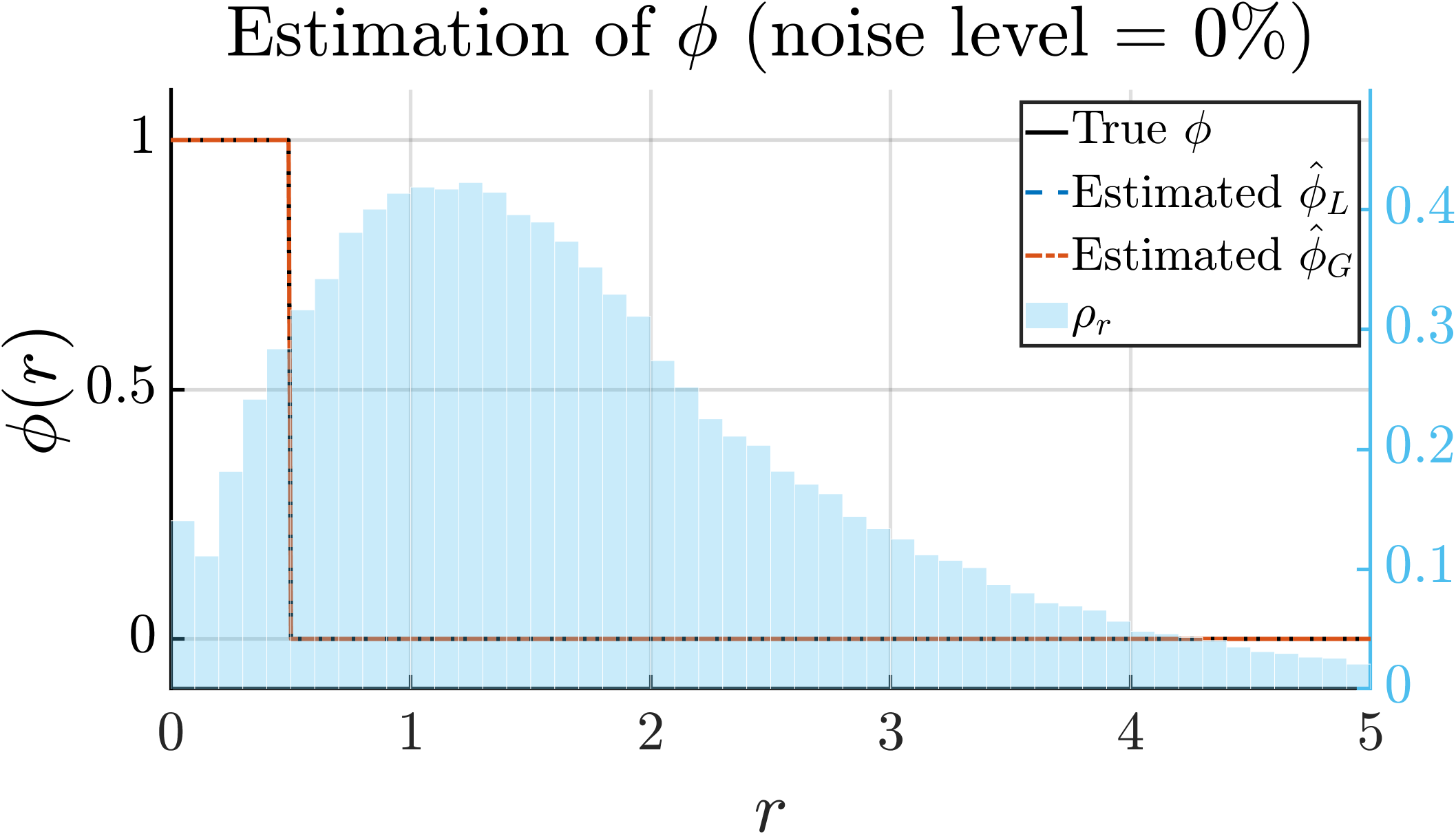}
\end{minipage}
(b)\begin{minipage}{.46\textwidth} \centering   
\includegraphics[scale=0.18]{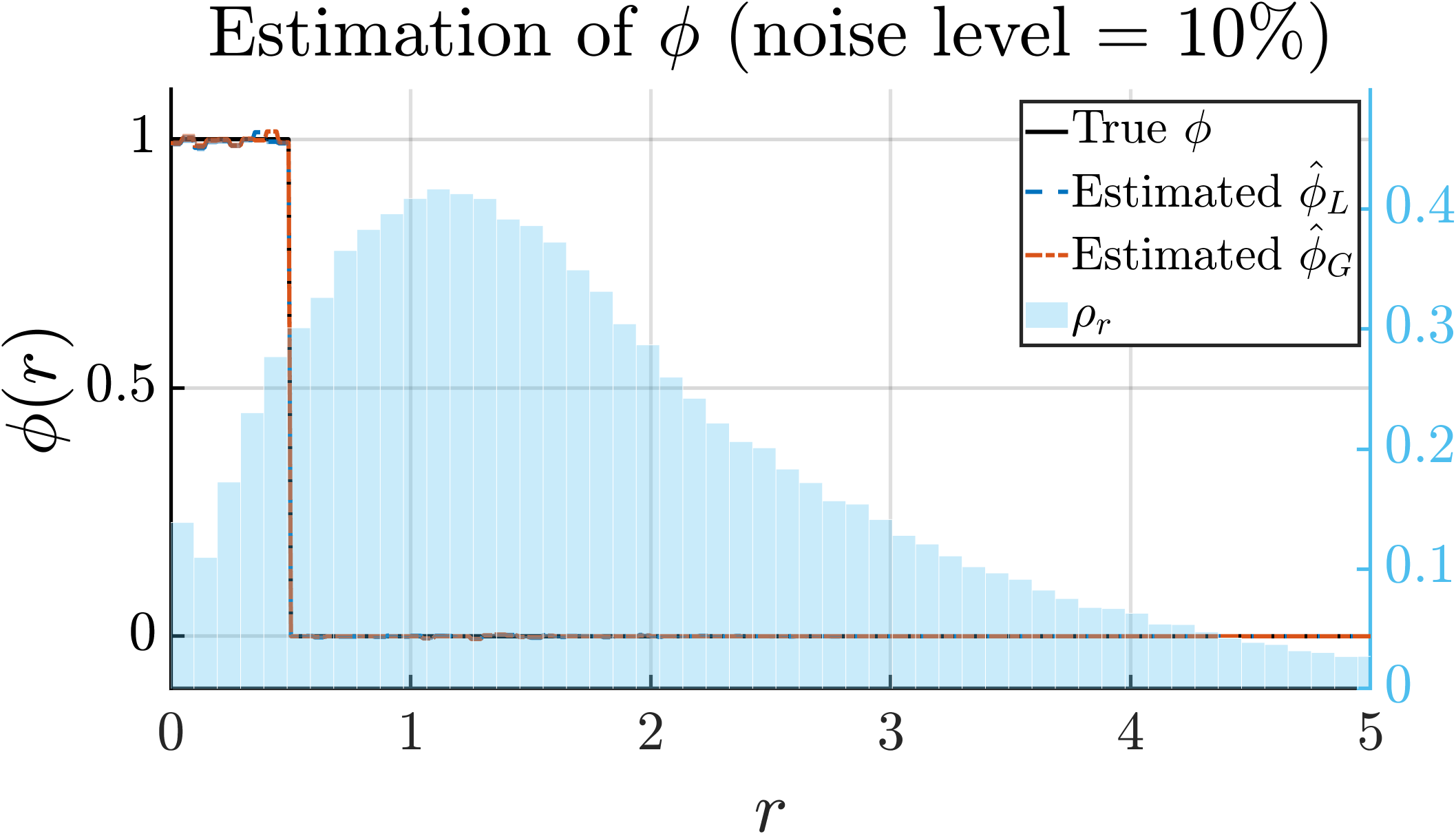}
\end{minipage}
(c) \begin{minipage}{.46\textwidth} \centering    
\includegraphics[scale=0.18]{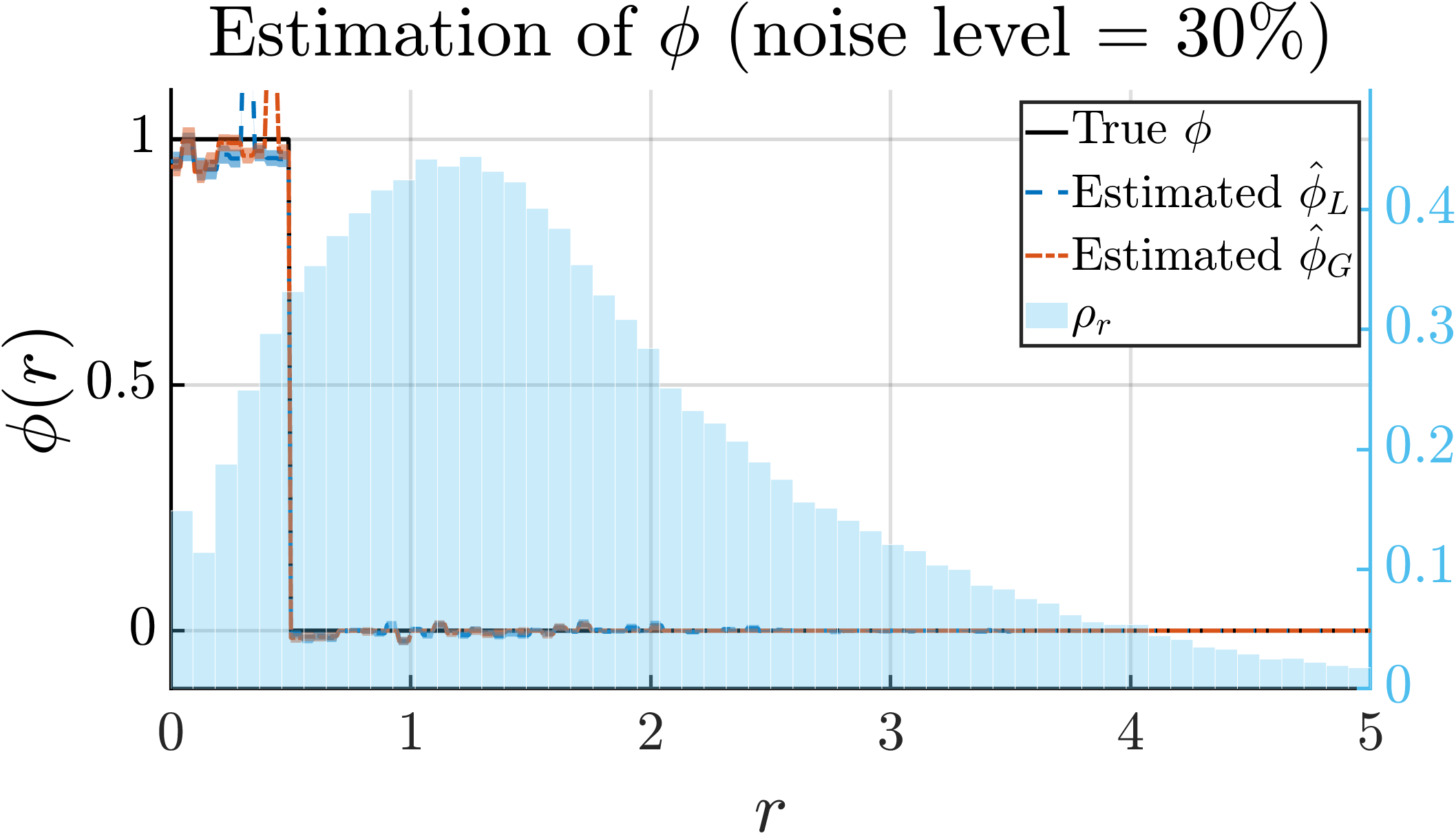}
\end{minipage}
(d)\begin{minipage}{.46\textwidth} \centering   
\includegraphics[scale=0.18]{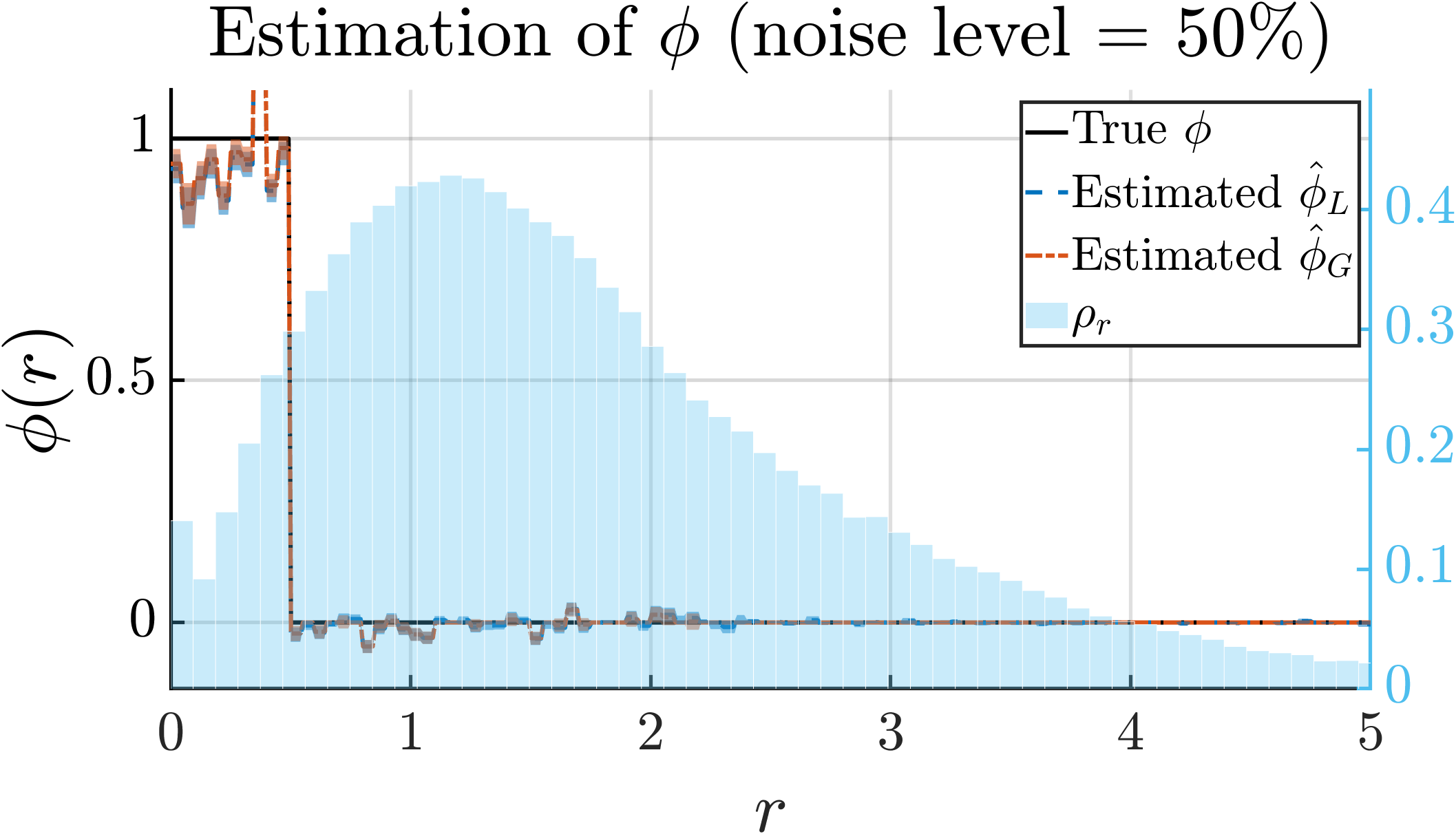}
\end{minipage}
\caption{Estimated interaction kernels $\hat{\phi}$ for the compactly supported kernel case under increasing noise levels ($(M,L) = (5,6)$). The solid black curve represents the true interaction kernel $\phi$, while the estimated kernels (e.g., $\hat{\phi}_L$ and $\hat{\phi}_G$) correspond to the different hyperpriors described in Section~3. Different colored curves indicate different noise levels in the observed data.}
\label{fig:CK1Estimation}
\end{figure}

\begin{figure}[!htb]
\centering
(a) \begin{minipage}{.45\textwidth} \centering    
\includegraphics[scale=0.5]{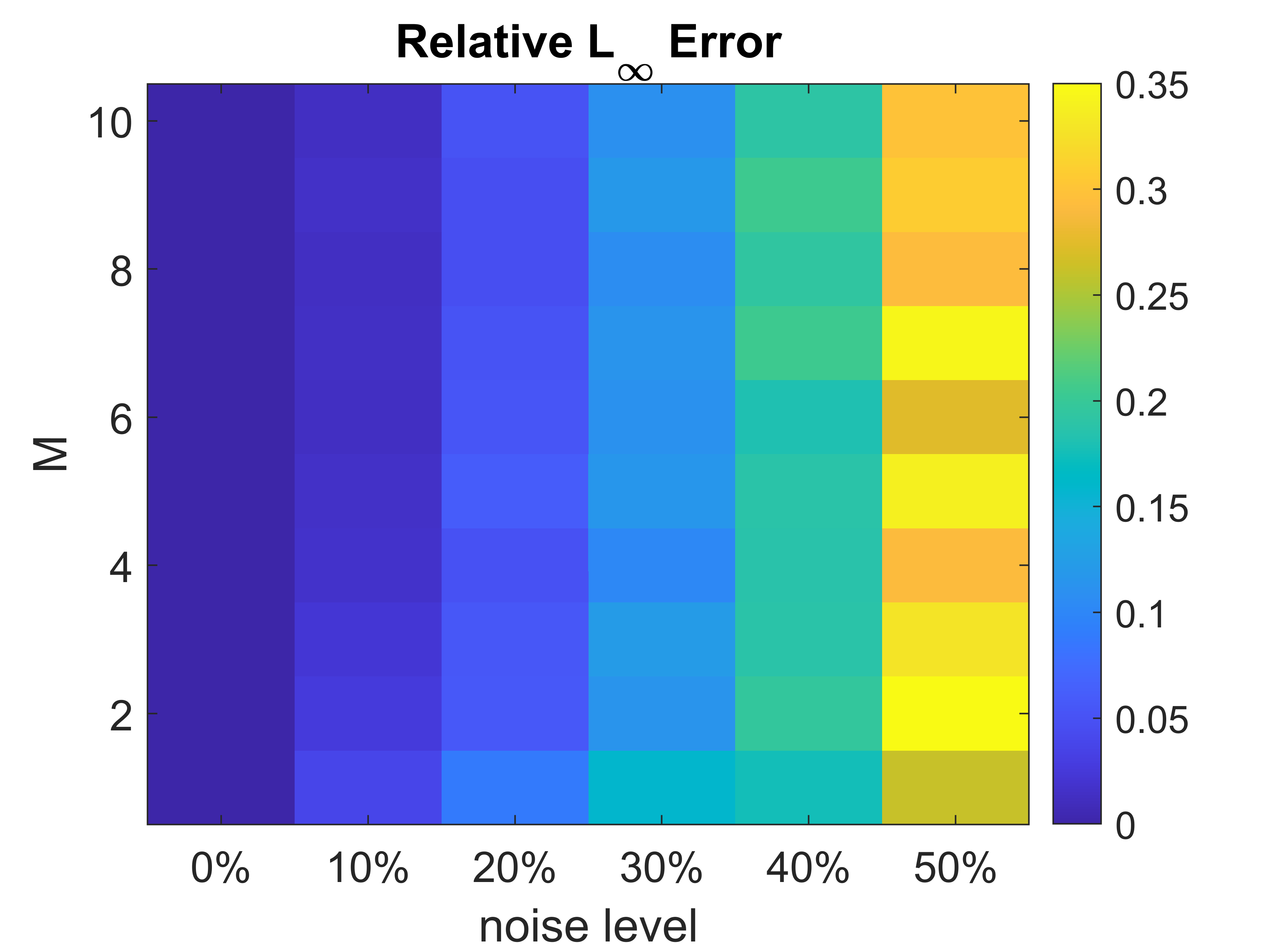}
\end{minipage}
(b) \begin{minipage}{.45\textwidth} \centering    
\includegraphics[scale=0.51]{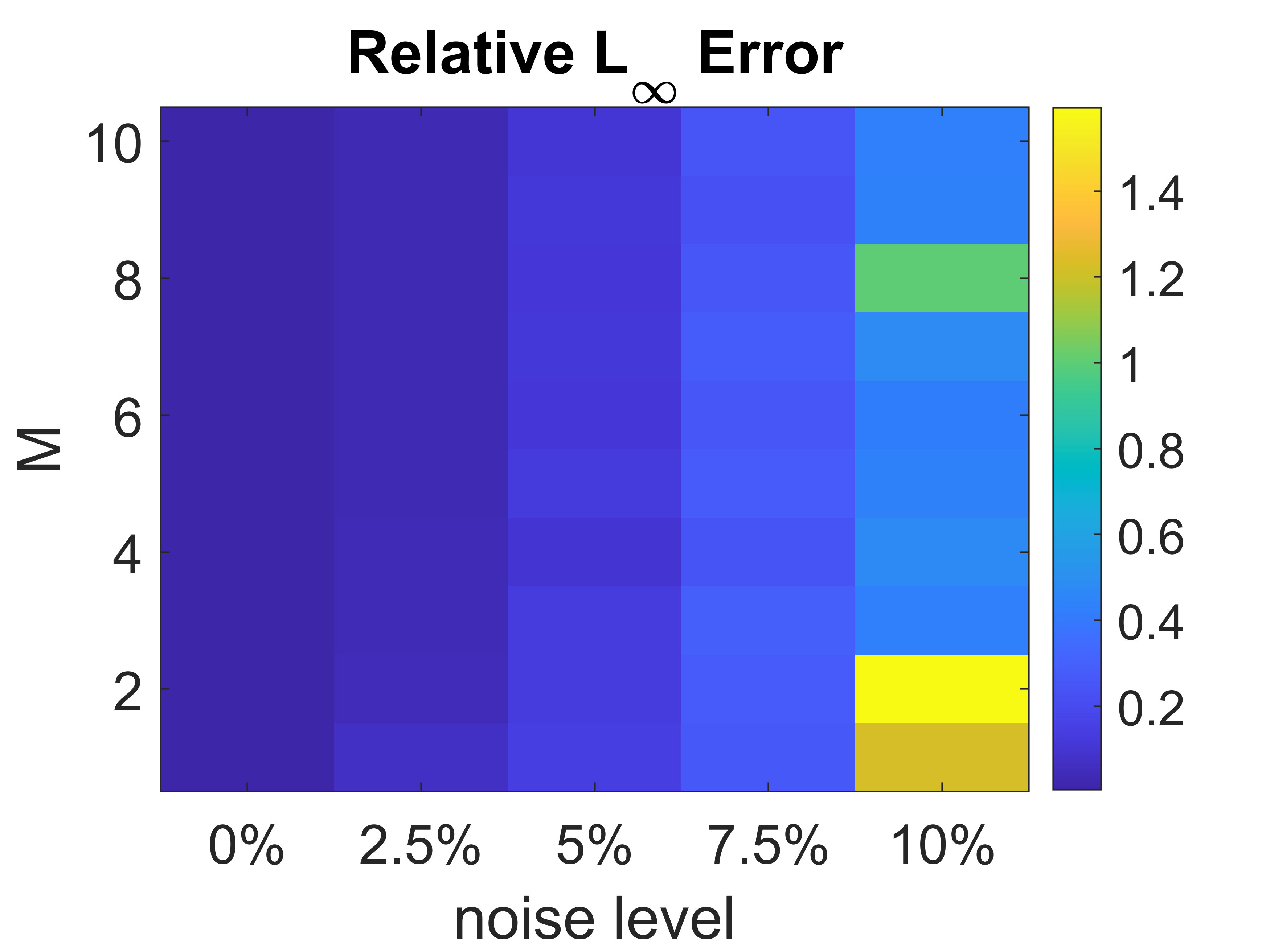}
\end{minipage}
\caption{Prediction error of the estimated interaction kernel $\hat{\phi}_L$ for the Cucker--Smale model with (a) compactly supported (cut-off) kernel and (b) rapidly decaying kernel, under varying numbers of trajectories $M$ and noise levels, with $L=6$ fixed. The error $\|\hat{\phi}_L - \phi\|$ quantifies the discrepancy between the Fast Laplace estimate and the true kernel $\phi$.}
\label{fig:CS_errors}
\end{figure}

\paragraph{Example 2 (Rapid decay kernel)} In the second example, we examine a smooth, rapidly decaying interaction kernel:
\begin{equation*} 
\phi(r) = \frac{1}{(1+r^2)^{1/4}},
\end{equation*}
which maintains long-range influence while emphasizing nearby interactions. This kernel enables a balance between local coherence and global flocking. The true trajectories and resulting estimations are presented in Figures \ref{CKtraj}(b), 9(b), and \ref{fig:CK2Estimation}, respectively.

\begin{figure}[!htb]
\centering

(a) \begin{minipage}{.46\textwidth} \centering    
\includegraphics[scale=0.18]{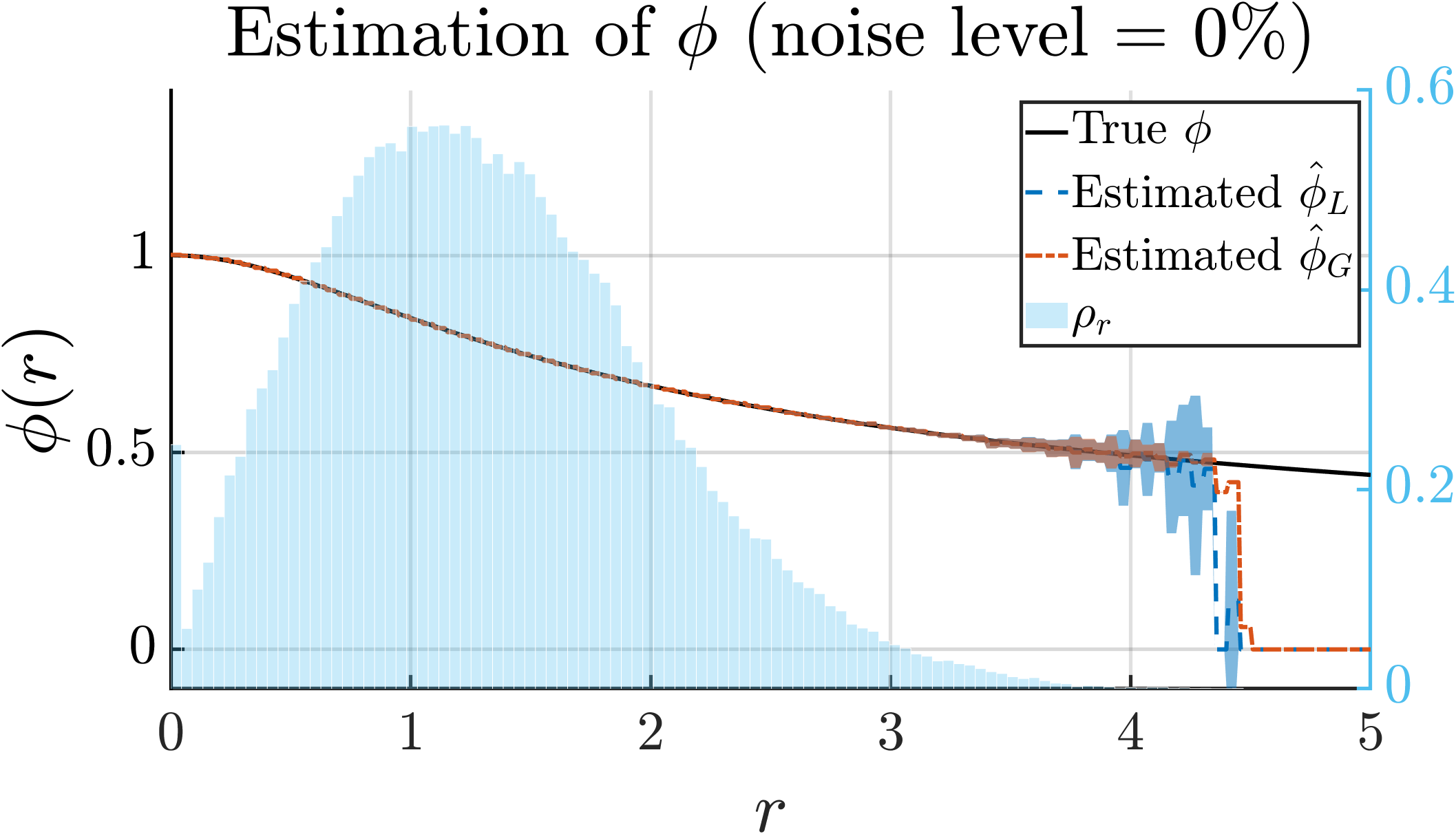}
\end{minipage}
(b)\begin{minipage}{.46\textwidth} \centering   
\includegraphics[scale=0.18]{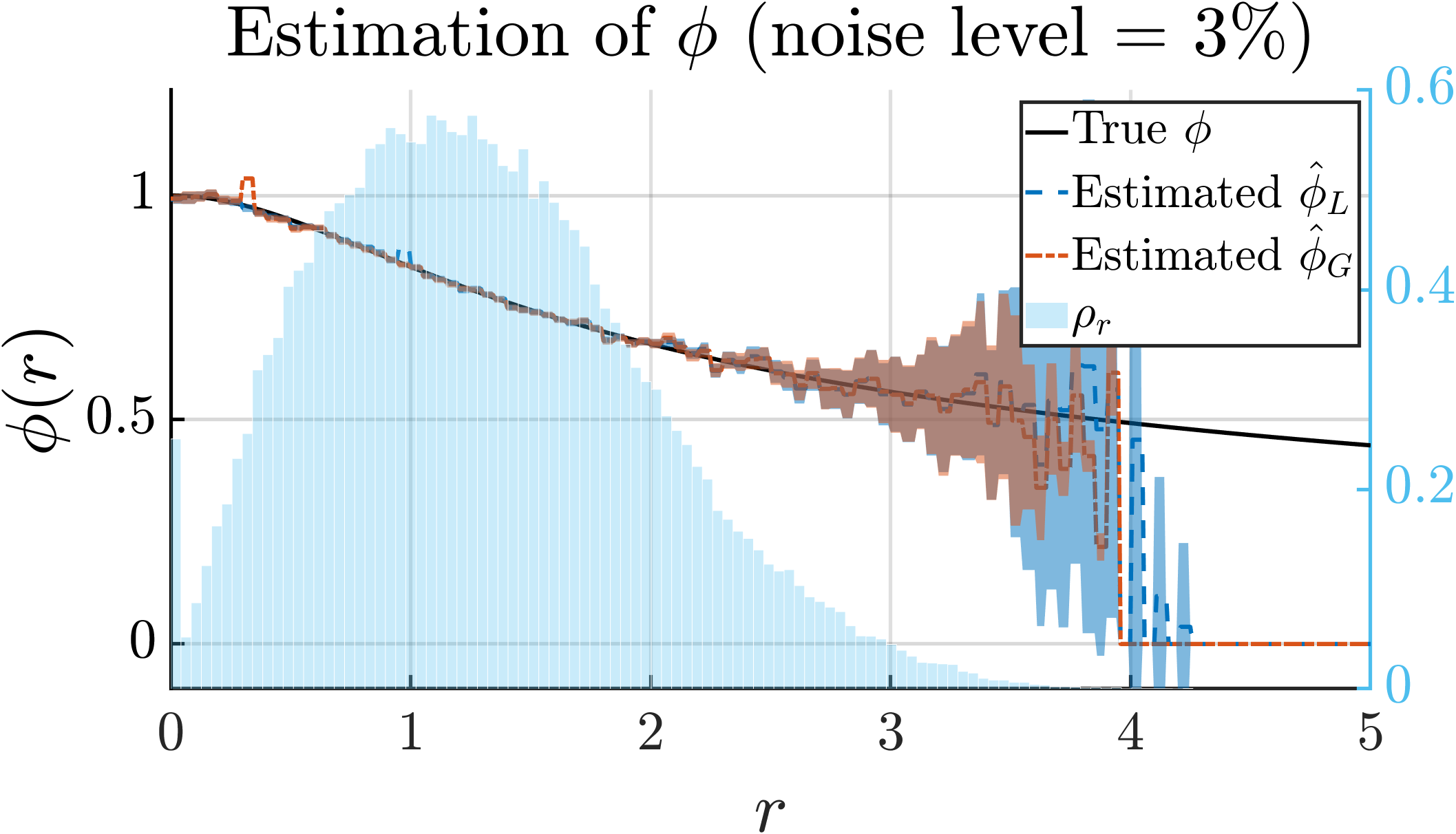}
\end{minipage}
(c) \begin{minipage}{.46\textwidth} \centering    
\includegraphics[scale=0.18]{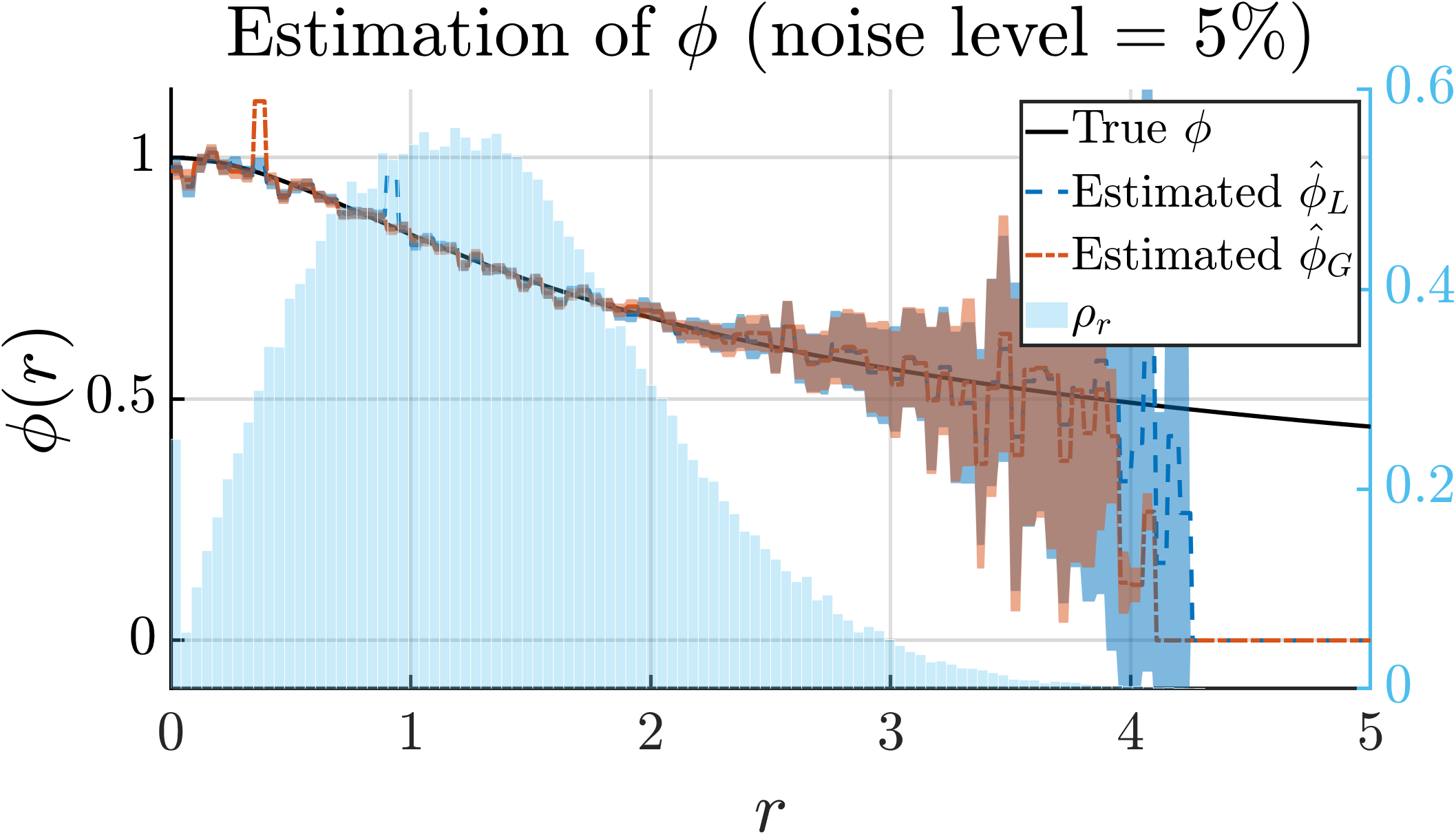}
\end{minipage}
(d)\begin{minipage}{.46\textwidth} \centering   
\includegraphics[scale=0.18]{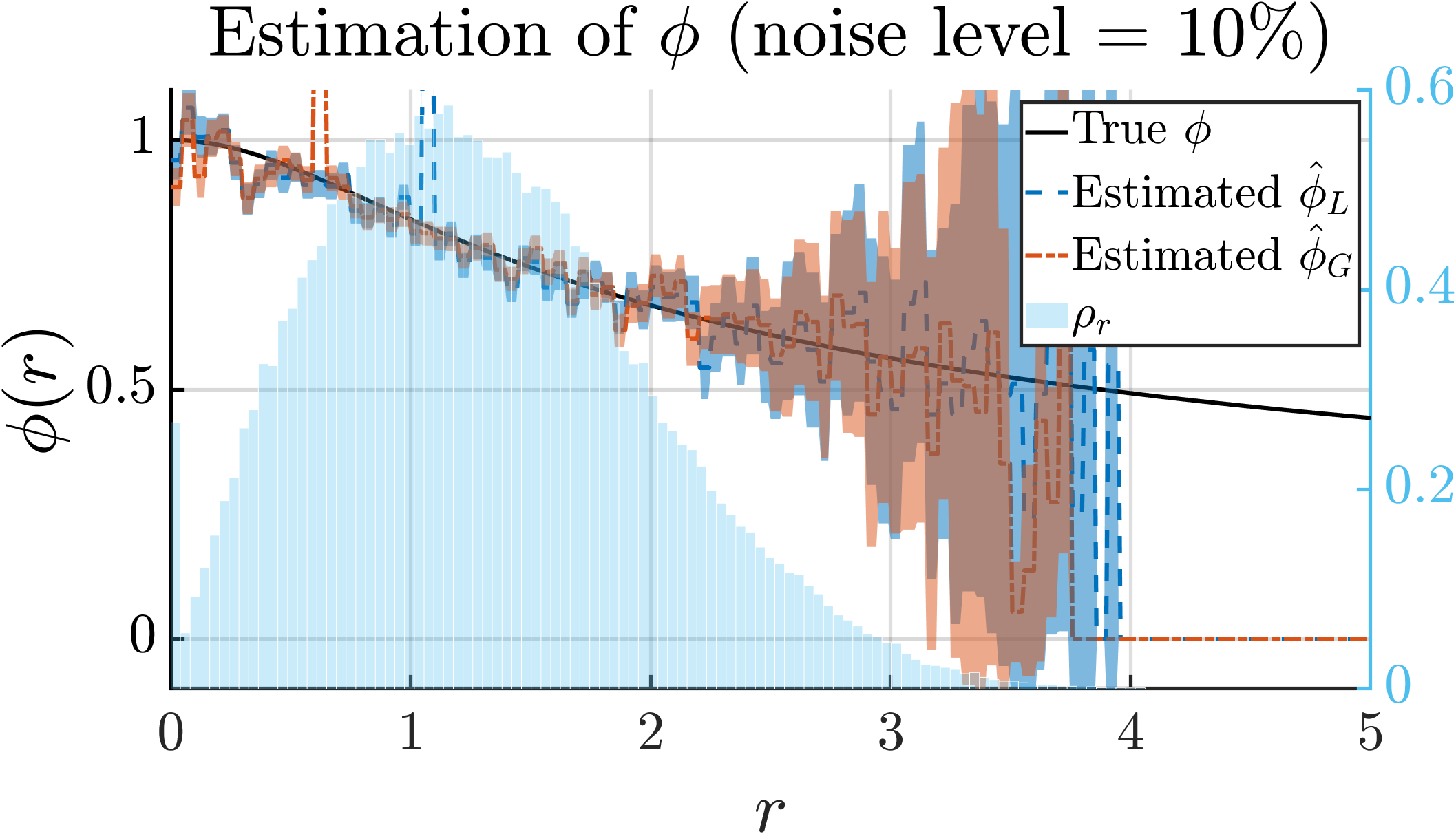}
\end{minipage}
\caption{Estimated interaction kernels $\hat{\phi}$ for the rapidly decaying kernel under increasing noise levels ($(M,L) = (5,6)$). The solid black curve represents the true interaction kernel $\phi$, while the estimated kernels (e.g., $\hat{\phi}_L$ and $\hat{\phi}_G$) correspond to the different hyperpriors described in Section~3. Different colored curves indicate different noise levels in the observed data.}
\label{fig:CK2Estimation}
\end{figure}

As shown in Figures \ref{fig:CK2Estimation}(c)–(d), high noise levels degrade the estimation quality, particularly in intervals with limited data support. This is expected, as we approximate a smooth kernel using piecewise constant basis functions. Nevertheless, Figure \ref{fig:CK2_trajEstimation} demonstrates that despite local discrepancies in $\hat\phi$, the predicted trajectories remain accurate, highlighting the model’s robustness in forward simulation.

\begin{figure}[!htb]
\centering
\includegraphics[scale=0.4]{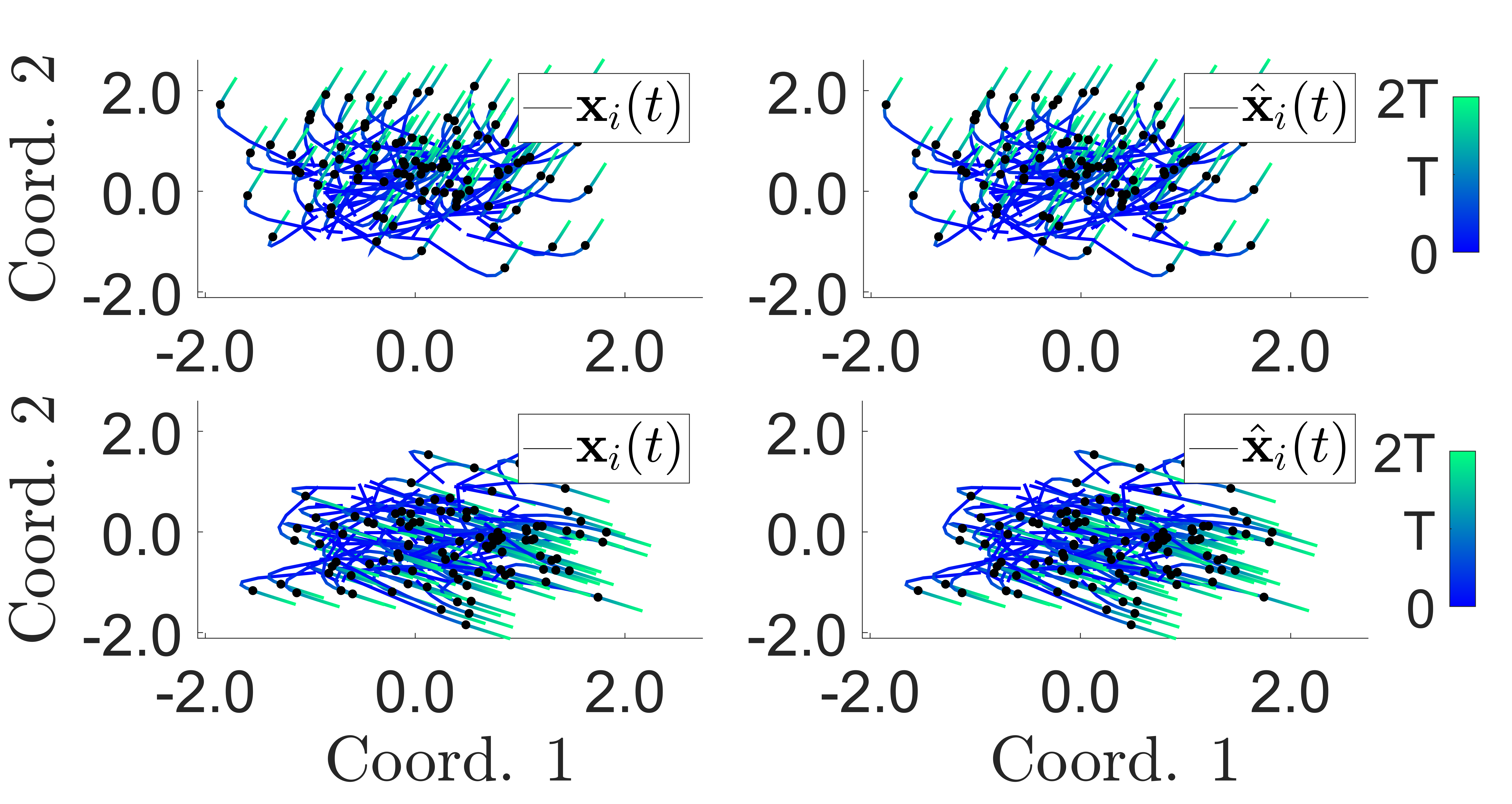}
\caption{Comparison of true (left) and predicted (right) agent trajectories using the estimated interaction kernel $\hat{\phi}_L$ under 10\% noise ($(M,L) = (5,6)$). Despite local discrepancies in $\hat{\phi}$, the predicted system dynamics remain consistent with the ground truth trajectories.}
\label{fig:CK2_trajEstimation}
\end{figure}

\begin{figure}[!htb]
\centering
(a) \begin{minipage}{.45\textwidth} \centering    
\includegraphics[scale=0.5]{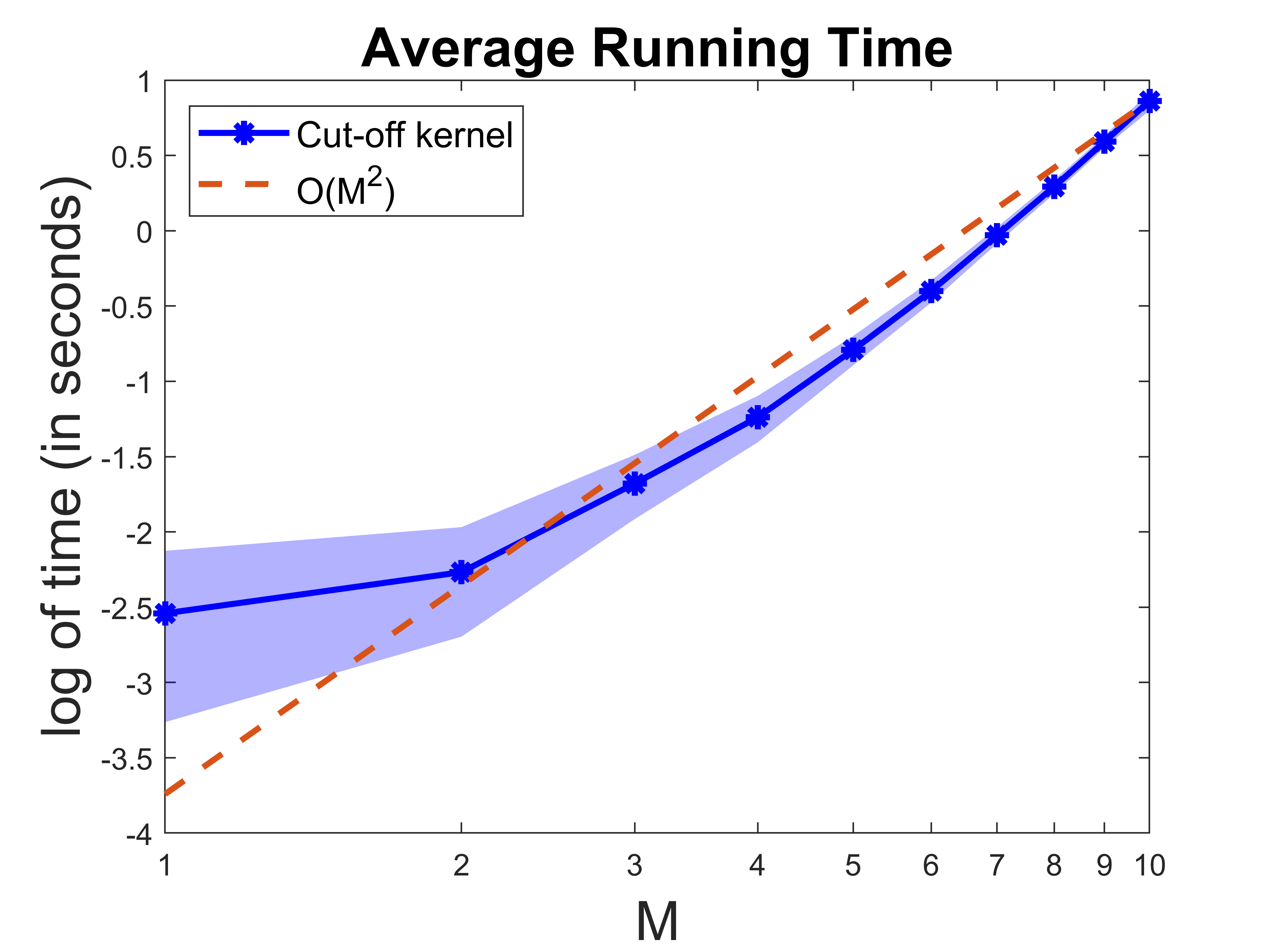}
\end{minipage}
(b)\begin{minipage}{.45\textwidth} \centering   
\includegraphics[scale=0.5]{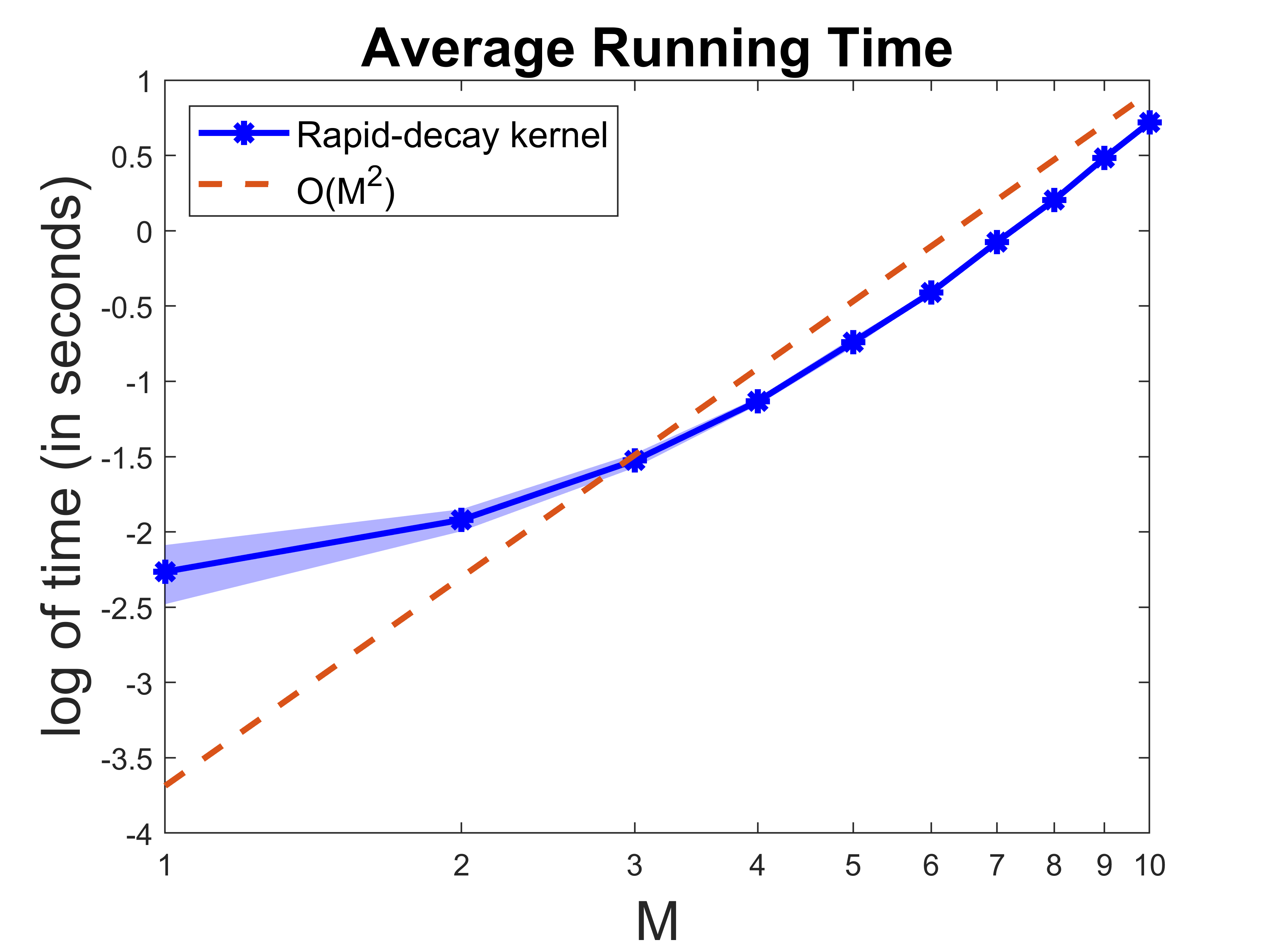}
\end{minipage}
\caption{Average running time for learning Cucker-Smale models with two different kernels.}
\label{fig:CS_runtime}
\end{figure}

As in the first-order opinion dynamics case, we report the average runtime of our algorithm under different dataset sizes for both kernel types (cut-off and rapid decay). As shown in Figure \ref{fig:CS_runtime}, the computational cost grows approximately as $\mathcal{O}(M^2)$ in both cases, reflecting the practical benefits of the sparsity structure in the assembled regression matrices.

\paragraph{Trajectory prediction error.} Mirroring the first-order analysis, we evaluate the relative max-time trajectory error \eqref{eq:err-pred} on the prediction window $\mathcal{W}_{\text{pred}}=(5,T_f]$ for both Cucker--Smale kernels. Figure~\ref{fig:CS_traj_error} reports the prediction error against the number of training trajectories $M$ at multiple noise levels for the cut-off and the smooth-decay kernels. Two distinctive behaviors emerge: for the smooth kernel (panel (b)) the error decreases monotonically with $M$ and remains tightly concentrated even at $\sigma_{\text{nsr}}=0.1$; for the cut-off kernel (panel (a)) the single-trajectory regime ($M=1$) at high noise occasionally produces a learned $\widehat\phi$ that destabilizes the forward integration, giving rise to large outlier errors and large standard deviations, but this pathology vanishes once $M\!\geq\!5$. Tables~\ref{tab:CS-cutoff-traj}--\ref{tab:CS-smooth-traj} list the corresponding numerical values.

\begin{figure}[!htb]
\centering
(a) \begin{minipage}{.45\textwidth} \centering
\includegraphics[width=\linewidth]{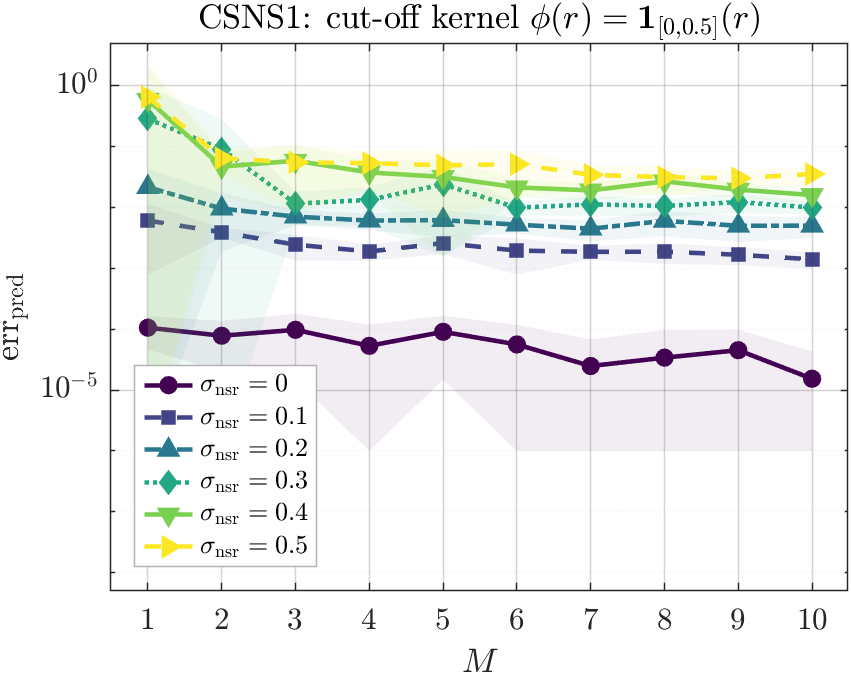}
\end{minipage}
(b) \begin{minipage}{.45\textwidth} \centering
\includegraphics[width=\linewidth]{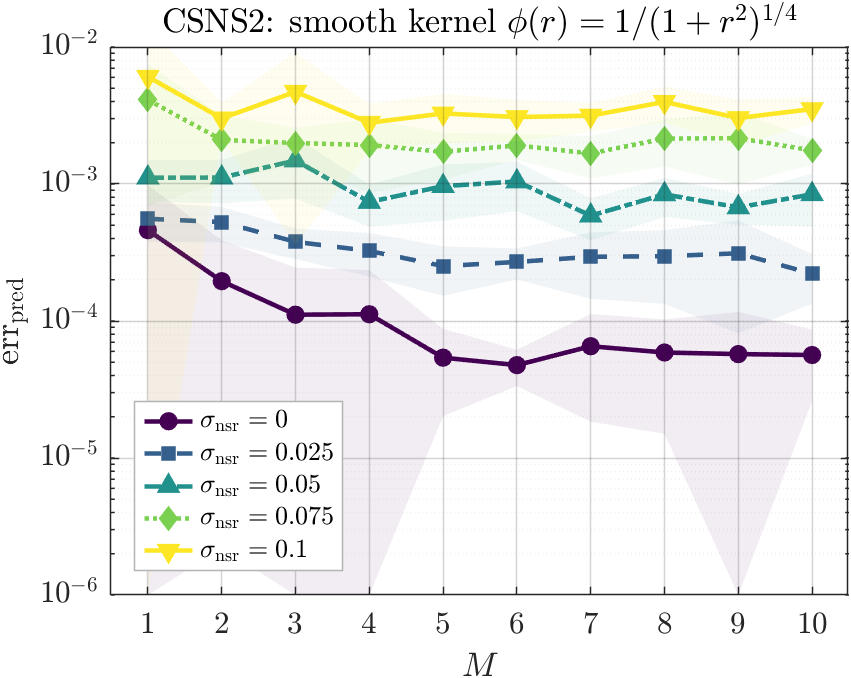}
\end{minipage}
\caption{Trajectory prediction error $\mathrm{err}(\mathcal{W}_{\text{pred}})$ for the Cucker--Smale model under varying number of training trajectories $M$ and noise levels, with $L=6$ fixed. (a) Cut-off kernel $\phi(r)=\mathbf{1}_{[0,0.5]}(r)$. (b) Smooth-decay kernel $\phi(r)=1/(1+r^2)^{1/4}$. Solid lines show the mean across $T=10$ trials and shaded bands show the $\pm 1$ standard-deviation envelope.}
\label{fig:CS_traj_error}
\end{figure}


\begin{table}[!htb]
\centering
\caption{CSNS1: cut-off $\phi(r)=\mathbf{1}_{[0,0.5]}(r)$. Relative max-time trajectory prediction error
$\mathrm{err}(\mathcal{W}_{\text{pred}})$ on the prediction window
$(5,T_f]$, evaluated at a fresh initial condition. Rows: number of
training trajectories $M$; columns: noise-to-signal ratio
$\sigma_{\text{nsr}}$. Each cell reports $\mu \pm s$ across $T=10$
independent trials.}
\label{tab:CS-cutoff-traj}
\renewcommand{\arraystretch}{1.15}
\setlength{\tabcolsep}{4pt}
\footnotesize
\begin{tabular}{c|cccccc}
\toprule
\multirow{2}{*}{$M$} & \multicolumn{6}{c}{$\sigma_{\text{nsr}}$} \\
\cmidrule(lr){2-7}
  & $0$ & $0.1$ & $0.2$ & $0.3$ & $0.4$ & $0.5$ \\
\midrule
1 & $(1.1 \pm 0.6)\!\times\!10^{-4}$ & $(6.1 \pm 5.3)\!\times\!10^{-3}$ & $(2.1 \pm 2.2)\!\times\!10^{-2}$ & $(2.8 \pm 6.5)\!\times\!10^{-1}$ & $0.6 \pm 1.4$ & $0.6 \pm 1.7$ \\
2 & $(7.8 \pm 5.8)\!\times\!10^{-5}$ & $(3.9 \pm 0.7)\!\times\!10^{-3}$ & $(9.5 \pm 7.6)\!\times\!10^{-3}$ & $(0.9 \pm 1.9)\!\times\!10^{-1}$ & $(4.7 \pm 3.2)\!\times\!10^{-2}$ & $(6.4 \pm 3.0)\!\times\!10^{-2}$ \\
3 & $(9.7 \pm 8.3)\!\times\!10^{-5}$ & $(2.4 \pm 1.1)\!\times\!10^{-3}$ & $(7.0 \pm 1.9)\!\times\!10^{-3}$ & $(1.2 \pm 0.6)\!\times\!10^{-2}$ & $(5.8 \pm 5.0)\!\times\!10^{-2}$ & $(5.5 \pm 4.6)\!\times\!10^{-2}$ \\
4 & $(5.3 \pm 6.6)\!\times\!10^{-5}$ & $(1.9 \pm 0.6)\!\times\!10^{-3}$ & $(6.1 \pm 1.8)\!\times\!10^{-3}$ & $(1.3 \pm 0.8)\!\times\!10^{-2}$ & $(3.7 \pm 2.0)\!\times\!10^{-2}$ & $(5.3 \pm 3.3)\!\times\!10^{-2}$ \\
5 & $(9.0 \pm 7.5)\!\times\!10^{-5}$ & $(2.6 \pm 0.9)\!\times\!10^{-3}$ & $(6.2 \pm 2.9)\!\times\!10^{-3}$ & $(2.4 \pm 2.2)\!\times\!10^{-2}$ & $(3.2 \pm 3.0)\!\times\!10^{-2}$ & $(4.9 \pm 4.2)\!\times\!10^{-2}$ \\
6 & $(5.6 \pm 6.2)\!\times\!10^{-5}$ & $(1.9 \pm 1.2)\!\times\!10^{-3}$ & $(5.2 \pm 1.5)\!\times\!10^{-3}$ & $(9.8 \pm 2.4)\!\times\!10^{-3}$ & $(2.1 \pm 0.7)\!\times\!10^{-2}$ & $(5.1 \pm 3.6)\!\times\!10^{-2}$ \\
7 & $(2.5 \pm 4.3)\!\times\!10^{-5}$ & $(1.9 \pm 0.5)\!\times\!10^{-3}$ & $(4.5 \pm 1.7)\!\times\!10^{-3}$ & $(1.1 \pm 0.4)\!\times\!10^{-2}$ & $(1.9 \pm 0.6)\!\times\!10^{-2}$ & $(3.4 \pm 1.5)\!\times\!10^{-2}$ \\
8 & $(3.4 \pm 6.3)\!\times\!10^{-5}$ & $(1.9 \pm 0.7)\!\times\!10^{-3}$ & $(6.0 \pm 2.7)\!\times\!10^{-3}$ & $(1.1 \pm 0.3)\!\times\!10^{-2}$ & $(2.7 \pm 1.0)\!\times\!10^{-2}$ & $(3.1 \pm 1.6)\!\times\!10^{-2}$ \\
9 & $(4.5 \pm 5.3)\!\times\!10^{-5}$ & $(1.7 \pm 0.5)\!\times\!10^{-3}$ & $(5.0 \pm 2.3)\!\times\!10^{-3}$ & $(1.2 \pm 0.4)\!\times\!10^{-2}$ & $(2.0 \pm 0.5)\!\times\!10^{-2}$ & $(3.0 \pm 0.9)\!\times\!10^{-2}$ \\
10 & $(1.5 \pm 2.8)\!\times\!10^{-5}$ & $(1.4 \pm 0.4)\!\times\!10^{-3}$ & $(5.0 \pm 1.9)\!\times\!10^{-3}$ & $(9.9 \pm 2.3)\!\times\!10^{-3}$ & $(1.6 \pm 0.4)\!\times\!10^{-2}$ & $(3.5 \pm 1.3)\!\times\!10^{-2}$ \\
\bottomrule
\end{tabular}
\end{table}

\begin{table}[!htb]
\centering
\caption{CSNS2: smooth $\phi(r)=1/(1+r^2)^{1/4}$. Relative max-time trajectory prediction error
$\mathrm{err}(\mathcal{W}_{\text{pred}})$ on the prediction window
$(5,T_f]$, evaluated at a fresh initial condition. Rows: number of
training trajectories $M$; columns: noise-to-signal ratio
$\sigma_{\text{nsr}}$. Each cell reports $\mu \pm s$ across $T=10$
independent trials.}
\label{tab:CS-smooth-traj}
\renewcommand{\arraystretch}{1.15}
\setlength{\tabcolsep}{4pt}
\footnotesize
\begin{tabular}{c|ccccc}
\toprule
\multirow{2}{*}{$M$} & \multicolumn{5}{c}{$\sigma_{\text{nsr}}$} \\
\cmidrule(lr){2-6}
  & $0$ & $0.025$ & $0.05$ & $0.075$ & $0.1$ \\
\midrule
1 & $(4.6 \pm 5.2)\!\times\!10^{-4}$ & $(5.6 \pm 1.8)\!\times\!10^{-4}$ & $(1.1 \pm 0.4)\!\times\!10^{-3}$ & $(4.1 \pm 3.4)\!\times\!10^{-3}$ & $(6.1 \pm 8.0)\!\times\!10^{-3}$ \\
2 & $(2.0 \pm 1.9)\!\times\!10^{-4}$ & $(5.3 \pm 1.5)\!\times\!10^{-4}$ & $(1.1 \pm 0.4)\!\times\!10^{-3}$ & $(2.1 \pm 1.1)\!\times\!10^{-3}$ & $(3.0 \pm 0.9)\!\times\!10^{-3}$ \\
3 & $(1.1 \pm 1.3)\!\times\!10^{-4}$ & $(3.8 \pm 1.0)\!\times\!10^{-4}$ & $(1.5 \pm 0.7)\!\times\!10^{-3}$ & $(2.0 \pm 0.6)\!\times\!10^{-3}$ & $(4.7 \pm 4.4)\!\times\!10^{-3}$ \\
4 & $(1.1 \pm 1.2)\!\times\!10^{-4}$ & $(3.2 \pm 1.1)\!\times\!10^{-4}$ & $(7.4 \pm 2.5)\!\times\!10^{-4}$ & $(1.9 \pm 1.1)\!\times\!10^{-3}$ & $(2.8 \pm 1.1)\!\times\!10^{-3}$ \\
5 & $(5.4 \pm 3.4)\!\times\!10^{-5}$ & $(2.5 \pm 1.0)\!\times\!10^{-4}$ & $(9.6 \pm 4.3)\!\times\!10^{-4}$ & $(1.7 \pm 0.6)\!\times\!10^{-3}$ & $(3.3 \pm 1.3)\!\times\!10^{-3}$ \\
6 & $(4.8 \pm 1.4)\!\times\!10^{-5}$ & $(2.7 \pm 0.7)\!\times\!10^{-4}$ & $(1.0 \pm 0.4)\!\times\!10^{-3}$ & $(1.9 \pm 0.4)\!\times\!10^{-3}$ & $(3.1 \pm 1.0)\!\times\!10^{-3}$ \\
7 & $(6.5 \pm 4.7)\!\times\!10^{-5}$ & $(2.9 \pm 1.5)\!\times\!10^{-4}$ & $(5.9 \pm 2.0)\!\times\!10^{-4}$ & $(1.7 \pm 0.6)\!\times\!10^{-3}$ & $(3.2 \pm 0.7)\!\times\!10^{-3}$ \\
8 & $(5.9 \pm 4.4)\!\times\!10^{-5}$ & $(3.0 \pm 1.6)\!\times\!10^{-4}$ & $(8.4 \pm 2.7)\!\times\!10^{-4}$ & $(2.1 \pm 0.8)\!\times\!10^{-3}$ & $(4.0 \pm 1.2)\!\times\!10^{-3}$ \\
9 & $(5.7 \pm 5.9)\!\times\!10^{-5}$ & $(3.1 \pm 2.3)\!\times\!10^{-4}$ & $(6.7 \pm 1.7)\!\times\!10^{-4}$ & $(2.2 \pm 1.2)\!\times\!10^{-3}$ & $(3.0 \pm 1.1)\!\times\!10^{-3}$ \\
10 & $(5.6 \pm 3.0)\!\times\!10^{-5}$ & $(2.2 \pm 0.9)\!\times\!10^{-4}$ & $(8.4 \pm 3.5)\!\times\!10^{-4}$ & $(1.8 \pm 0.4)\!\times\!10^{-3}$ & $(3.5 \pm 0.8)\!\times\!10^{-3}$ \\
\bottomrule
\end{tabular}
\end{table}

\subsubsection{Comparison with implicit SINDy.}\label{comparision:sindy}
We compare our approach with an implicit SINDy-type method \cite{mangan2016inferring}. In our setting, the implicit formulation
\[
\tilde{\mathbb{A}}_{M,L}\mathbf{c} = 0
\]
is constructed directly from \eqref{newe}. Following \cite{mangan2016inferring}, we compute an approximate null space of $\tilde{\mathbb{A}}_{M,L}$ via singular value decomposition, and then identify a sparse vector in this subspace using an alternating directions method (ADM).

Figure~\ref{fig:sindy_comparison} (a)-(b) shows the comparison under different noise levels. When the true kernel $\phi$ lies exactly in the chosen functional space (e.g., piecewise constant basis functions $\{\chi_{[\frac{10(k-1)}{20},\frac{10k}{20}]}(r)\}$), the implicit SINDy approach can recover the kernel with behavior similar to our method in the noise-free case. However, it does not provide uncertainty quantification and becomes sensitive as noise increases.

\begin{figure}[!ht]
\centering
(a) \begin{minipage}{.43\textwidth} \centering    
    \includegraphics[width=\linewidth]{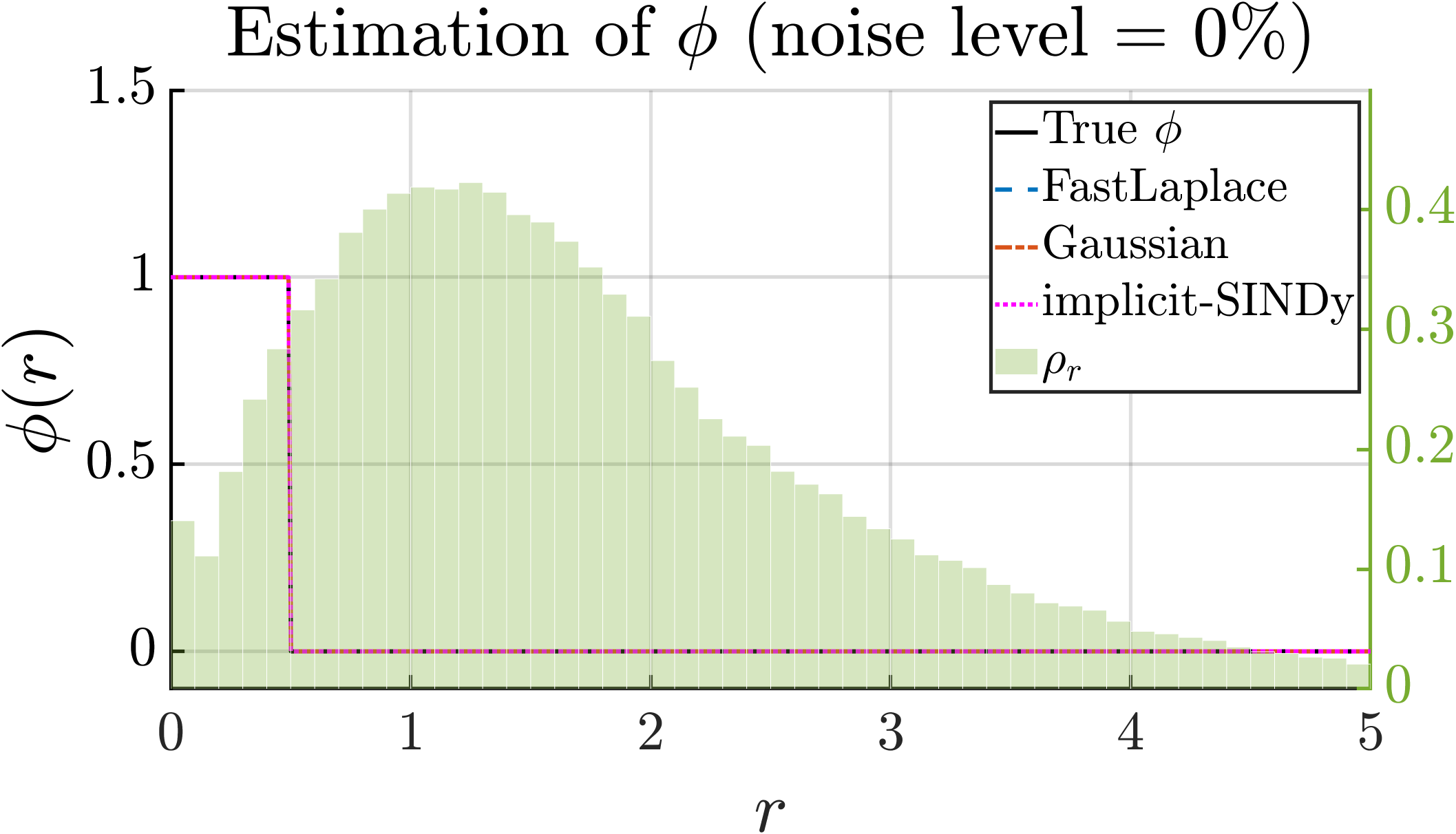}
\end{minipage}
(b) \begin{minipage}{.43\textwidth} 
    \includegraphics[width=\linewidth]{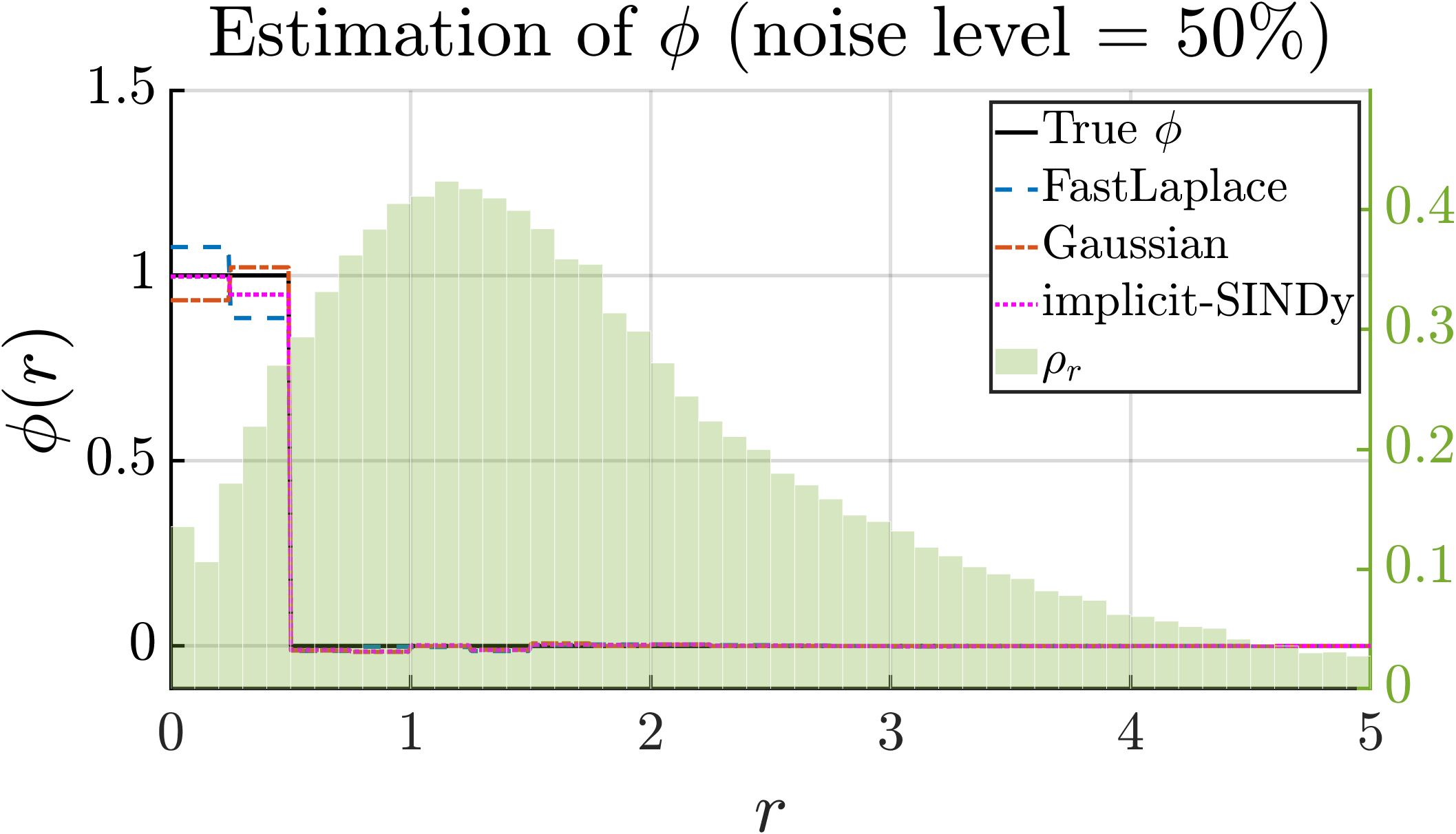}
\end{minipage}

(c) \begin{minipage}{.43\textwidth} \centering    
    \includegraphics[width=\linewidth]{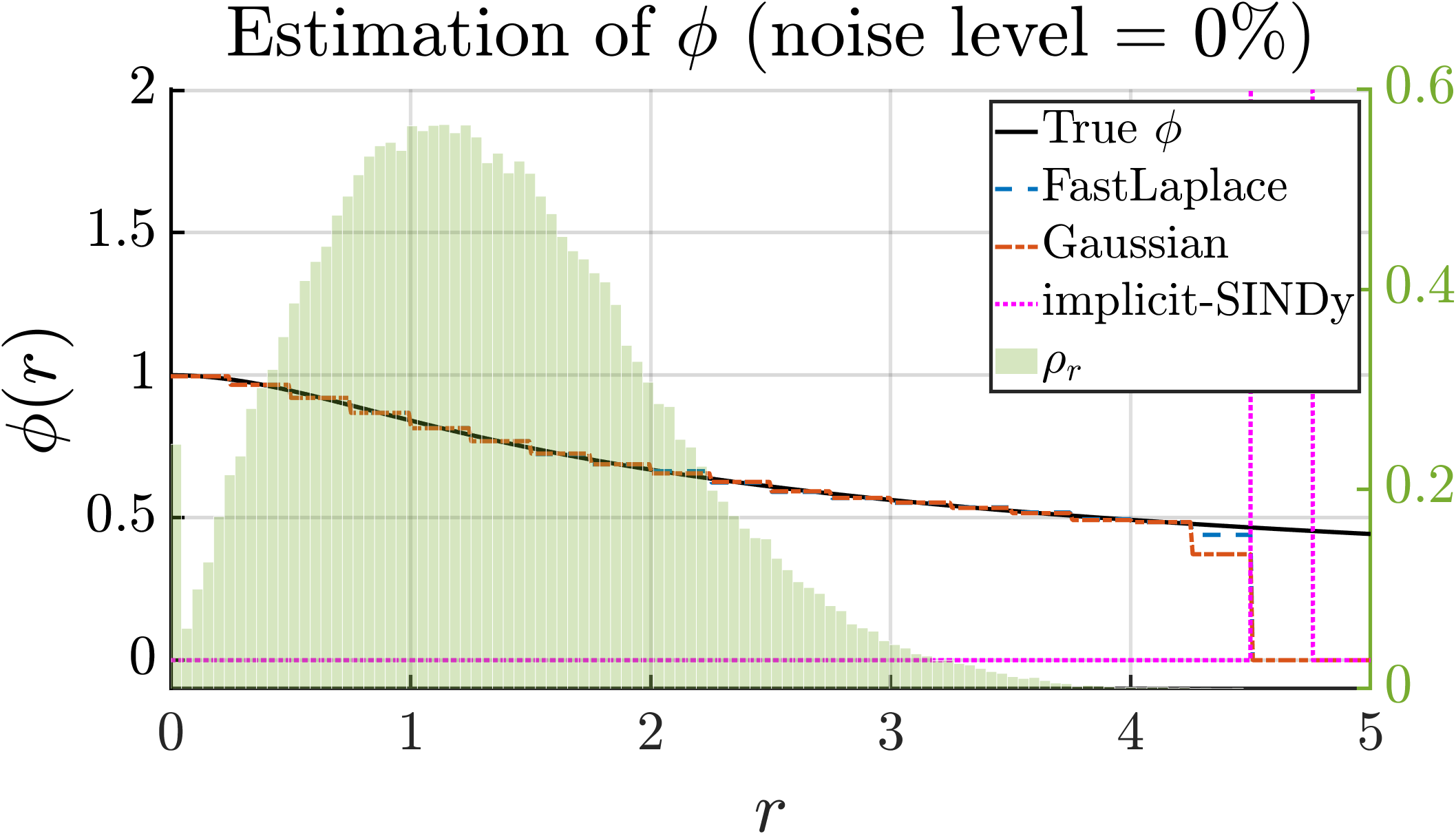}
\end{minipage}
(d) \begin{minipage}{.43\textwidth} \centering    
    \includegraphics[width=\linewidth]{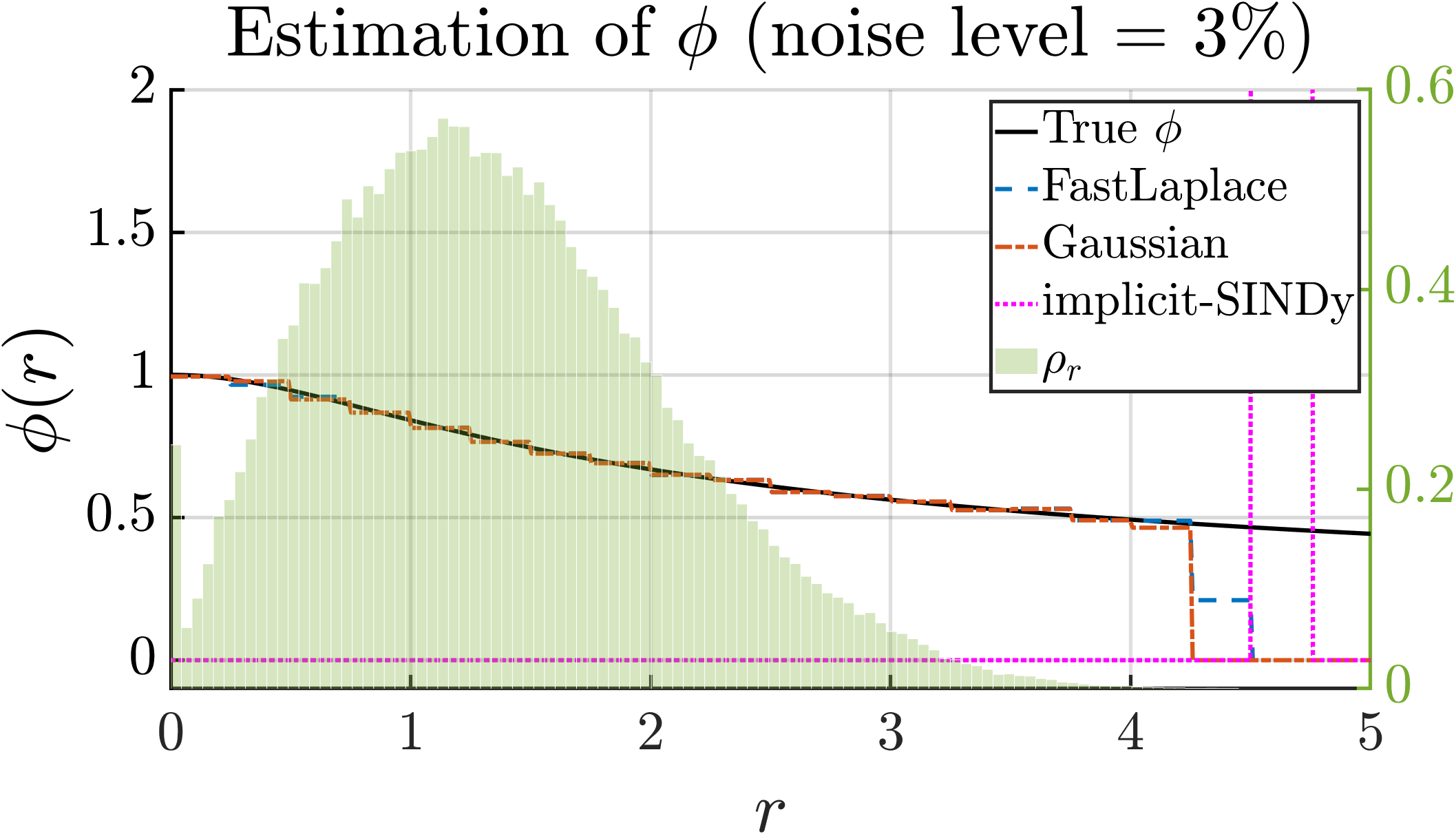}
\end{minipage}
\caption{
Comparison of interaction kernel estimation using the proposed sparse Bayesian learning approach (Fast Laplace and Gaussian hyperpriors) and an implicit SINDy-type method.
The green histogram represents the empirical distance distribution $\rho_r$.
When the true kernel $\phi$ is exactly representable in the chosen basis (bottom row), the implicit SINDy method can recover the correct structure in the noise-free case, but it does not provide uncertainty quantification and degrades under noise.
When the true kernel is not contained in the chosen functional space (top row), the implicit SINDy method fails to produce meaningful estimates, whereas the proposed method remains robust.
}
\label{fig:sindy_comparison}
\end{figure}

When the true kernel is not exactly representable in the chosen basis space, Figure~\ref{fig:sindy_comparison} (c)-(d) shows that the implicit SINDy approach fails to produce a meaningful approximation in our experiments. In contrast, the proposed sparse Bayesian learning framework remains robust, yielding stable estimates and providing meaningful uncertainty quantification even in the presence of noise and model mismatch.

\section{Conclusion and Discussion}
This work presents a variational framework for learning asymmetric interaction kernels in the Motsch–Tadmor model using trajectory data. By reformulating the governing equations in implicit form, the inverse problem is reduced to a subspace identification task. We establish an identifiability result under suitable conditions on the data distribution and basis support, providing theoretical guarantees for kernel recovery.

To address ill-posedness and noise sensitivity, we develop a sparse Bayesian learning (SBL) algorithm that incorporates prior structure, enables model selection via a new weighted total uncertainty (wTU) criterion, and provides uncertainty quantification. Numerical experiments demonstrate accurate recovery of interaction laws and robustness across noise levels and limited data.

Future work includes extensions to position-only observations, heterogeneous or time-varying interactions, and finite-sample recovery guarantees. The proposed framework contributes a principled, interpretable approach to data-driven modeling of asymmetric dynamical systems.

\section*{Data Availability Statement}
The numerical results presented in this manuscript are based on synthetic data generated by computational simulations. The complete reproducibility package---including the implementation of the proposed sparse Bayesian learning algorithm, scripts to regenerate every figure and table in the manuscript, and the pre-computed trajectory-error sweep results---is publicly available at \url{https://github.com/stang33/GPs_IPS} (directory \texttt{learning\_nonsymmetric\_interactions/}). A citable archival snapshot will be deposited on Zenodo with a permanent DOI upon publication.

\begin{acknowledgements}
J. Feng was partially supported by the National Natural Science Foundation of China (Grant No. 12501632), the Cross‑Disciplinary Research Team on Data Science and Intelligent Medicine (Grant No. 2023KCXTD054), and the Dongguan Key Laboratory for AI and Dynamical Systems. S. Tang was partially supported by the National Science Foundation under NSF CAREER Award DMS‑2340631 and by the UCSB Faculty Early Career Development Award.
\end{acknowledgements}

%
%

\bibliographystyle{spmpsci}      

\bibliography{reference}

@article{ha2024mono,
  title={Mono-cluster flocking and uniform-in-time stability of the discrete {M}otsch--{T}admor model},
  author={Ha, Seung-Yeal and Hoffmann, Franca and Kim, Dohyeon and Yoon, Wook},
  journal={arXiv preprint arXiv:2408.10213},
  year={2024}
}

@article{jin2018flocking,
  title={Flocking of the {M}otsch--{T}admor model with a cut-off interaction function},
  author={Jin, Chunyin},
  journal={Journal of Statistical Physics},
  volume={171},
  number={2},
  pages={345--360},
  year={2018},
  publisher={Springer}
}

@article{motsch2011new,
  title={A new model for self-organized dynamics and its flocking behavior},
  author={Motsch, Sebastien and Tadmor, Eitan},
  journal={Journal of Statistical Physics},
  volume={144},
  pages={923--947},
  year={2011},
  publisher={Springer}
}

@article{motsch2014heterophilious,
  title={Heterophilious dynamics enhances consensus},
  author={Motsch, Sebastien and Tadmor, Eitan},
  journal={SIAM Review},
  volume={56},
  number={4},
  pages={577--621},
  year={2014},
  publisher={SIAM}
}

@article{jabin2014clustering,
  title={Clustering and asymptotic behavior in opinion formation},
  author={Jabin, Pierre-Emmanuel and Motsch, Sebastien},
  journal={Journal of Differential Equations},
  volume={257},
  number={11},
  pages={4165--4187},
  year={2014},
  publisher={Elsevier}
}

@inproceedings{blondel2005convergence,
  title={Convergence in multiagent coordination, consensus, and flocking},
  author={Blondel, Vincent D and Hendrickx, Julien M and Olshevsky, Alex and Tsitsiklis, John N},
  booktitle={Proceedings of the 44th IEEE Conference on Decision and Control},
  pages={2996--3000},
  year={2005},
  organization={IEEE}
}

@article{blondel2009krause,
  title={On {K}rause's multi-agent consensus model with state-dependent connectivity},
  author={Blondel, Vincent D and Hendrickx, Julien M and Tsitsiklis, John N},
  journal={IEEE Transactions on Automatic Control},
  volume={54},
  number={11},
  pages={2586--2597},
  year={2009},
  publisher={IEEE}
}

@article{cucker2007emergent,
  title={Emergent behavior in flocks},
  author={Cucker, Felipe and Smale, Steve},
  journal={IEEE Transactions on Automatic Control},
  volume={52},
  number={5},
  pages={852--862},
  year={2007},
  publisher={IEEE}
}

@article{geshkovski2024emergence,
  title={The emergence of clusters in self-attention dynamics},
  author={Geshkovski, Borjan and Letrouit, Cyril and Polyanskiy, Yury and Rigollet, Philippe},
  journal={Advances in Neural Information Processing Systems},
  volume={36},
  year={2024}
}

@article{couzin2002collective,
  title={Collective memory and spatial sorting in animal groups},
  author={Couzin, Iain D and Krause, Jens and James, Richard and Ruxton, Graeme D and Franks, Nigel R},
  journal={Journal of Theoretical Biology},
  volume={218},
  number={1},
  pages={1--11},
  year={2002},
  publisher={Elsevier}
}

@article{katz2011inferring,
  title={Inferring the structure and dynamics of interactions in schooling fish},
  author={Katz, Yael and Tunstr{\o}m, Kolbj{\o}rn and Ioannou, Christos C and Huepe, Cristi{\'a}n and Couzin, Iain D},
  journal={Proceedings of the National Academy of Sciences},
  volume={108},
  number={46},
  pages={18720--18725},
  year={2011},
  publisher={National Academy of Sciences}
}

@article{lukeman2010inferring,
  title={Inferring individual rules from collective behavior},
  author={Lukeman, Ryan and Li, Yue-Xian and Edelstein-Keshet, Leah},
  journal={Proceedings of the National Academy of Sciences},
  volume={107},
  number={28},
  pages={12576--12580},
  year={2010},
  publisher={National Academy of Sciences}
}

@article{vicsek1995novel,
  title={Novel type of phase transition in a system of self-driven particles},
  author={Vicsek, Tam{\'a}s and Czir{\'o}k, Andr{\'a}s and Ben-Jacob, Eshel and Cohen, Inon and Shochet, Ofer},
  journal={Physical Review Letters},
  volume={75},
  number={6},
  pages={1226},
  year={1995},
  publisher={APS}
}

@article{keller1971model,
  title={Model for chemotaxis},
  author={Keller, Evelyn F and Segel, Lee A},
  journal={Journal of Theoretical Biology},
  volume={30},
  number={2},
  pages={225--234},
  year={1971},
  publisher={Elsevier}
}

@article{laskar1990chaotic,
  title={The chaotic motion of the solar system: A numerical estimate of the size of the chaotic zones},
  author={Laskar, Jacques},
  journal={Icarus},
  volume={88},
  number={2},
  pages={266--291},
  year={1990},
  publisher={Elsevier}
}

@article{hegselmann2002opinion,
  title={Opinion dynamics and bounded confidence: Models, analysis and simulation},
  author={Hegselmann, Rainer and Krause, Ulrich},
  journal={Journal of Artificial Societies and Social Simulation},
  volume={5},
  number={3},
  year={2002}
}

@article{krause2000discrete,
  title={A discrete nonlinear and non-autonomous model of consensus formation},
  author={Krause, Ulrich},
  journal={Communications in Difference Equations},
  pages={227--236},
  year={2000},
  publisher={Gordon and Breach, Amsterdam}
}

@article{acebron2005kuramoto,
  title={The {K}uramoto model: A simple paradigm for synchronization phenomena},
  author={Acebr{\'o}n, Juan A and Bonilla, Luis L and P{\'e}rez Vicente, Conrad J and Ritort, F{\'e}lix and Spigler, Renato},
  journal={Reviews of Modern Physics},
  volume={77},
  number={1},
  pages={137--185},
  year={2005},
  publisher={APS}
}

@article{lu2019nonparametric,
  title={Nonparametric inference of interaction laws in systems of agents from trajectory data},
  author={Lu, Fei and Zhong, Ming and Tang, Sui and Maggioni, Mauro},
  journal={Proceedings of the National Academy of Sciences},
  volume={116},
  number={29},
  pages={14424--14433},
  year={2019},
  publisher={National Academy of Sciences}
}

@article{lu2020learning,
  title={Learning interaction kernels in stochastic systems of interacting particles from multiple trajectories},
  author={Lu, Fei and Maggioni, Mauro and Tang, Sui},
  journal={arXiv preprint arXiv:2007.15174},
  year={2020}
}

@article{lu2021learning,
  title={Learning interaction kernels in heterogeneous systems of agents from multiple trajectories},
  author={Lu, Fei and Maggioni, Mauro and Tang, Sui},
  journal={The Journal of Machine Learning Research},
  volume={22},
  number={1},
  pages={1518--1584},
  year={2021},
  publisher={JMLR.org}
}

@article{miller2023learning,
  title={Learning theory for inferring interaction kernels in second-order interacting agent systems},
  author={Miller, Jason and Tang, Sui and Zhong, Ming and Maggioni, Mauro},
  journal={Sampling Theory, Signal Processing, and Data Analysis},
  volume={21},
  number={1},
  pages={21},
  year={2023},
  publisher={Springer}
}

@article{liu2021random,
  title={Random features for kernel approximation: A survey on algorithms, theory, and beyond},
  author={Liu, Fanghui and Huang, Xiaolin and Chen, Yudong and Suykens, Johan A K},
  journal={IEEE Transactions on Pattern Analysis and Machine Intelligence},
  volume={44},
  number={10},
  pages={7128--7148},
  year={2021},
  publisher={IEEE}
}

@article{mangan2016inferring,
  title={Inferring biological networks by sparse identification of nonlinear dynamics},
  author={Mangan, Niall M and Brunton, Steven L and Proctor, Joshua L and Kutz, J Nathan},
  journal={IEEE Transactions on Molecular, Biological and Multi-Scale Communications},
  volume={2},
  number={1},
  pages={52--63},
  year={2016},
  publisher={IEEE}
}

@article{schaeffer2018extracting,
  title={Extracting sparse high-dimensional dynamics from limited data},
  author={Schaeffer, Hayden and Tran, Giang and Ward, Rachel},
  journal={SIAM Journal on Applied Mathematics},
  volume={78},
  number={6},
  pages={3279--3295},
  year={2018},
  publisher={SIAM}
}

@article{zhang2018robust,
  title={Robust data-driven discovery of governing physical laws with error bars},
  author={Zhang, Sheng and Lin, Guang},
  journal={Proceedings of the Royal Society A: Mathematical, Physical and Engineering Sciences},
  volume={474},
  number={2217},
  pages={20180305},
  year={2018},
  publisher={The Royal Society Publishing}
}

@article{kaheman2020sindy,
  title={{SINDy-PI}: A robust algorithm for parallel implicit sparse identification of nonlinear dynamics},
  author={Kaheman, Kadierdan and Kutz, J Nathan and Brunton, Steven L},
  journal={Proceedings of the Royal Society A},
  volume={476},
  number={2242},
  pages={20200279},
  year={2020},
  publisher={The Royal Society Publishing}
}

@article{tipping2001sparse,
  title={Sparse {B}ayesian learning and the relevance vector machine},
  author={Tipping, Michael E},
  journal={Journal of Machine Learning Research},
  volume={1},
  number={Jun},
  pages={211--244},
  year={2001}
}

@inproceedings{tipping2003fast,
  title={Fast marginal likelihood maximisation for sparse {B}ayesian models},
  author={Tipping, Michael E and Faul, Anita C},
  booktitle={International Workshop on Artificial Intelligence and Statistics},
  pages={276--283},
  year={2003},
  organization={PMLR}
}

@article{babacan2009bayesian,
  title={{B}ayesian compressive sensing using {L}aplace priors},
  author={Babacan, S Derin and Molina, Rafael and Katsaggelos, Aggelos K},
  journal={IEEE Transactions on Image Processing},
  volume={19},
  number={1},
  pages={53--63},
  year={2009},
  publisher={IEEE}
}

@inproceedings{babacan2009fast,
  title={Fast {B}ayesian compressive sensing using {L}aplace priors},
  author={Babacan, S Derin and Molina, Rafael and Katsaggelos, Aggelos K},
  booktitle={2009 IEEE International Conference on Acoustics, Speech and Signal Processing},
  pages={2873--2876},
  year={2009},
  organization={IEEE}
}

@article{wipf2010iterative,
  title={Iterative reweighted $\ell_1$ and $\ell_2$ methods for finding sparse solutions},
  author={Wipf, David and Nagarajan, Srikantan},
  journal={IEEE Journal of Selected Topics in Signal Processing},
  volume={4},
  number={2},
  pages={317--329},
  year={2010},
  publisher={IEEE}
}

@article{zhang2011sparse,
  title={Sparse signal recovery with temporally correlated source vectors using sparse {B}ayesian learning},
  author={Zhang, Zhilin and Rao, Bhaskar D},
  journal={IEEE Journal of Selected Topics in Signal Processing},
  volume={5},
  number={5},
  pages={912--926},
  year={2011},
  publisher={IEEE}
}

@article{sant2022block,
  title={Block-sparse signal recovery via general total variation regularized sparse {B}ayesian learning},
  author={Sant, Aditya and Leinonen, Markus and Rao, Bhaskar D},
  journal={IEEE Transactions on Signal Processing},
  volume={70},
  pages={1056--1071},
  year={2022},
  publisher={IEEE}
}

@article{krishnan2009fast,
  title={Fast image deconvolution using hyper-{L}aplacian priors},
  author={Krishnan, Dilip and Fergus, Rob},
  journal={Advances in Neural Information Processing Systems},
  volume={22},
  year={2009}
}

@article{calvetti2020sparse,
  title={Sparse reconstructions from few noisy data: Analysis of hierarchical {B}ayesian models with generalized gamma hyperpriors},
  author={Calvetti, Daniela and Pragliola, Marco and Somersalo, Erkki and Strang, Angela},
  journal={Inverse Problems},
  volume={36},
  number={2},
  pages={025010},
  year={2020}
}

@article{glaubitz2023generalized,
  title={Generalized sparse {B}ayesian learning and application to image reconstruction},
  author={Glaubitz, Jan and Gelb, Anne and Song, Guannan},
  journal={SIAM/ASA Journal on Uncertainty Quantification},
  volume={11},
  number={1},
  pages={262--284},
  year={2023}
}

@article{lindbloom2025efficient,
  title={Efficient sparsity-promoting MAP estimation for Bayesian linear inverse problems},
  author={Lindbloom, Jonathan and Glaubitz, Jan and Gelb, Anne},
  journal={Inverse Problems},
  volume={41},
  number={2},
  pages={025001},
  year={2025},
  publisher={IOP Publishing}
}

@article{glaubitz2025efficient,
  title={Efficient sampling for sparse Bayesian learning using hierarchical prior normalization},
  author={Glaubitz, Jan and Marzouk, Youssef},
  journal={arXiv preprint arXiv:2505.23753},
  year={2025}
}

@article{candes2008introduction,
  title={An introduction to compressive sampling},
  author={Cand{\`e}s, Emmanuel and Wakin, Michael},
  journal={IEEE Signal Processing Magazine},
  volume={25},
  number={2},
  pages={21--30},
  year={2008}
}

@article{donoho2006compressed,
  title={Compressed sensing},
  author={Donoho, David L},
  journal={IEEE Transactions on Information Theory},
  volume={52},
  number={4},
  pages={1289--1306},
  year={2006},
  publisher={IEEE}
}

@article{davis1970rotation,
  title={The rotation of eigenvectors by a perturbation},
  author={Davis, Chandler and Kahan, William},
  journal={SIAM Journal on Numerical Analysis},
  volume={7},
  number={1},
  pages={1--46},
  year={1970}
}

@book{stewart1990matrix,
  title={Matrix Perturbation Theory},
  author={Stewart, G W and Sun, Ji-Guang},
  year={1990},
  publisher={Academic Press}
}

@book{van2000asymptotic,
  title={Asymptotic Statistics},
  author={Van der Vaart, Aad W},
  volume={3},
  year={2000},
  publisher={Cambridge University Press}
}

@book{vershynin2018high,
  title={High-Dimensional Probability: An Introduction with Applications in Data Science},
  author={Vershynin, Roman},
  volume={47},
  year={2018},
  publisher={Cambridge University Press}
}

@article{schmidt2009distilling,
  title={Distilling free-form natural laws from experimental data},
  author={Schmidt, Michael and Lipson, Hod},
  journal={Science},
  volume={324},
  number={5923},
  pages={81--85},
  year={2009},
  publisher={American Association for the Advancement of Science}
}

@article{bongard2007automated,
  title={Automated reverse engineering of nonlinear dynamical systems},
  author={Bongard, Josh and Lipson, Hod},
  journal={Proceedings of the National Academy of Sciences},
  volume={104},
  number={24},
  pages={9943--9948},
  year={2007},
  publisher={National Academy of Sciences}
}

@article{RBPK2017,
  title={Data-driven discovery of partial differential equations},
  author={Rudy, Samuel H and Brunton, Steven L and Proctor, Joshua L and Kutz, J Nathan},
  journal={Science Advances},
  volume={3},
  number={4},
  pages={e1602614},
  year={2017},
  publisher={American Association for the Advancement of Science}
}

@article{boninsegna2018sparse,
  title={Sparse learning of stochastic dynamical equations},
  author={Boninsegna, Lorenzo and N{\"u}ske, Feliks and Clementi, Cecilia},
  journal={The Journal of Chemical Physics},
  volume={148},
  number={24},
  pages={241723},
  year={2018},
  publisher={AIP Publishing}
}

@article{raissi2018deep,
  title={Deep hidden physics models: Deep learning of nonlinear partial differential equations},
  author={Raissi, Maziar},
  journal={Journal of Machine Learning Research},
  volume={19},
  number={25},
  pages={1--24},
  year={2018}
}

@article{raissi2018hidden,
  title={Hidden physics models: Machine learning of nonlinear partial differential equations},
  author={Raissi, Maziar and Karniadakis, George Em},
  journal={Journal of Computational Physics},
  volume={357},
  pages={125--141},
  year={2018},
  publisher={Elsevier}
}

@inproceedings{long2017pde,
  title={{PDE-Net}: Learning {PDEs} from data},
  author={Long, Zichao and Lu, Yiping and Ma, Xianzhong and Dong, Bin},
  booktitle={International Conference on Machine Learning},
  pages={3208--3216},
  year={2018},
  organization={PMLR}
}

@article{BPK2016,
  title={Discovering governing equations from data by sparse identification of nonlinear dynamical systems},
  author={Brunton, Steven L and Proctor, Joshua L and Kutz, J Nathan},
  journal={Proceedings of the National Academy of Sciences},
  volume={113},
  number={15},
  pages={3932--3937},
  year={2016},
  publisher={National Academy of Sciences}
}

@article{TranWardExactRecovery,
  title={Exact recovery of chaotic systems from highly corrupted data},
  author={Tran, Giang and Ward, Rachel},
  journal={Multiscale Modeling \& Simulation},
  volume={15},
  number={3},
  pages={1108--1129},
  year={2017},
  publisher={SIAM}
}

@article{BCCCCGLOPPVZ2008,
  title={Interaction ruling animal collective behavior depends on topological rather than metric distance: {E}vidence from a field study},
  author={Ballerini, Michele and Cabibbo, Nicola and Candelier, Raphael and Cavagna, Andrea and Cisbani, Evaristo and Giardina, Irene and Lecomte, Vivien and Orlandi, Alberto and Parisi, Giorgio and Procaccini, Andrea and others},
  journal={Proceedings of the National Academy of Sciences},
  volume={105},
  number={4},
  pages={1232--1237},
  year={2008},
  publisher={National Academy of Sciences}
}

@article{BKP2017,
  title={Chaos as an intermittently forced linear system},
  author={Brunton, Steven L and Brunton, Bingni W and Proctor, Joshua L and Kaiser, Eurika and Kutz, J Nathan},
  journal={Nature Communications},
  volume={8},
  pages={19},
  year={2017},
  publisher={Nature Publishing Group}
}

@article{BCGMSVW2012,
  title={Inferring interaction rules from observations of evolutive systems {I}: The variational approach},
  author={Bongini, Mattia and Fornasier, Massimo and Hansen, Markus and Maggioni, Mauro},
  journal={Mathematical Models and Methods in Applied Sciences},
  volume={27},
  number={5},
  pages={909--951},
  year={2017},
  publisher={World Scientific}
}

@article{han2015robust,
  title={Robust recovery of signals from a structured union of subspaces},
  author={Han, Zhi and others},
  journal={IEEE Transactions on Information Theory},
  year={2015},
  note={VERIFY: confirm full citation details}
}

@article{kang2019ident,
  title={{IDENT}: Identifying differential equations with numerical time evolution},
  author={Kang, Sung Ha and Liao, Wenjing and Liu, Yingjie},
  journal={Journal of Scientific Computing},
  volume={87},
  number={1},
  pages={1},
  year={2021},
  publisher={Springer}
}
\end{document}